\documentclass{uiucthesis2021}

\usepackage{microtype}
\usepackage{graphicx}
\usepackage{subfigure}
\usepackage{subcaption}
\usepackage{booktabs} 

\usepackage{amsmath}
\usepackage{amssymb}
\usepackage{mathtools}
\usepackage{amsthm}

\usepackage[textsize=tiny]{todonotes}

\usepackage{xspace}
\usepackage[utf8]{inputenc} 
\usepackage[T1]{fontenc}    
\usepackage{url}            
\usepackage{booktabs}       
\usepackage{tcolorbox}
\usepackage{amsfonts}       
\usepackage{nicefrac}       
\usepackage{microtype}      
\usepackage{xcolor}         

\usepackage{amsmath,amsfonts,bm}

\def\eqref#1{equation~\ref{#1}}

\def\floor#1{\lfloor #1 \rfloor}
\def\1{\bm{1}}

\DeclareMathAlphabet{\mathsfit}{\encodingdefault}{\sfdefault}{m}{sl}
\SetMathAlphabet{\mathsfit}{bold}{\encodingdefault}{\sfdefault}{bx}{n}

\def\A{{\mathcal{A}}}

\newcommand{\E}{\mathbb{E}}

\newcommand{\R}{\mathbb{R}}

\newcommand{\btheta}{{\boldsymbol{\theta}}}
\newcommand{\thetafull}{{\btheta_{\mathrm{o}}}}
\newcommand{\thetaunl}{{\btheta_{\mathrm{u}}}}

\newcommand{\Fr}{\mathcal{F_{\text{R}}}}
\newcommand{\Dr}{\mathcal{D_{\text{R}}}}
\newcommand{\Df}{\mathcal{D_{\text{F}}}}
\newcommand{\Da}{\mathcal{D_{\text{A}}}}
\newcommand{\Dt}{\mathcal{D_{\text{T}}}}
\newcommand{\D}{\mathcal{D}}

\usepackage{graphicx}
\usepackage{multirow}

\usepackage{amsmath}
\usepackage{amssymb}
\usepackage{mathtools}
\usepackage{amsthm}
\usepackage{siunitx}
\usepackage{algorithm}
\usepackage{algorithmic}

\usepackage{caption}
\usepackage{subcaption}
\usepackage{float}

\usepackage{wrapfig}

\theoremstyle{plain}
\newtheorem{theorem}{Theorem}[section]
\newtheorem{proposition}[theorem]{Proposition}
\newtheorem{lemma}[theorem]{Lemma}
\newtheorem{corollary}[theorem]{Corollary}
\theoremstyle{definition}
\newtheorem{definition}[theorem]{Definition}

\theoremstyle{remark}
\newtheorem{remark}[theorem]{Remark}

\def\R{\mathbb{R}}
\def\E{\mathbb{E}}
\def\X{\mathcal{X}}
\def\Y{\mathcal{Y}}

\def\ddefloop#1{\ifx\ddefloop#1\else\ddef{#1}\expandafter\ddefloop\fi}
\def\ddef#1{\expandafter\def\csname bb#1\endcsname{\ensuremath{\mathbb{#1}}}}
\ddefloop ABCDEFGHIJKLMNOPQRSTUVWXYZ\ddefloop
\def\ddef#1{\expandafter\def\csname c#1\endcsname{\ensuremath{\mathcal{#1}}}}
\ddefloop ABCDEFGHIJKLMNOPQRSTUVWXYZ\ddefloop
\def\ddef#1{\expandafter\def\csname v#1\endcsname{\ensuremath{\boldsymbol{#1}}}}
\ddefloop ABCDEFGHIJKLMNOPQRSTUVWXYZabcdefghijklmnopqrstuvwxyz\ddefloop

\def\E{\mathbb{E}}
\def\R{\mathbb{R}}
\def\X{\mathcal{X}}
\def\Y{\mathcal{Y}}
\def\F{\mathcal{F}}
\def\G{\mathcal{G}}

\newcommand{\amun}{\textsc{AMUN}\xspace}

\def\1{\mathds{1}}

\newif\iffeedback
\feedbacktrue

\iffeedback
\newcommand{\varun}[1]{{\color{red}(Varun: #1)}}
\newcommand{\hari}[1]{{\color{olive}(Hari: #1)}}
\newcommand{\ali}[1]{{\color{green}(Ali: #1)}}
\newcommand{\aliq}[1]{{\color{blue}(Ali Q: #1)}}
\else
\newcommand{\varun}[1]{}
\newcommand{\hari}[1]{}
\newcommand{\ali}[1]{}
\newcommand{\aliq}[1]{}
\fi

{\end{itemize}}

\usepackage[utf8]{inputenc}
\usepackage[english]{babel}
\usepackage{csquotes}
\usepackage{microtype}
\usepackage[bookmarksdepth=3,linktoc=all,colorlinks=true,urlcolor=blue,linkcolor=blue,citecolor=blue]{hyperref}
\usepackage[capitalize]{cleveref}
\usepackage[style=ieee]{biblatex}
\usepackage{biblatex}

\newcommand{\PP}{\mathbb{P}}

\newcommand{\doop}{\operatorname{do}}
\newcommand{\PS}{\operatorname{PS}}
\newcommand{\PN}{\operatorname{PN}}

\newcounter{counterforappendices}

\begin{document}

\title{Toward Reliable Machine Unlearning: \\Theory, Algorithms, and Evaluation}
\author{Ali Ebrahimpour-Boroojeny}
\department{Computer Science}
\concentration{Artificial Intelligence}
\phdthesis
\degreeyear{2025}
\committee{
    Professor Hari Sundaram, Chair\\
    Assistant Professor Varun Chandrasekaran\\
    Professor Gang Wang\\
    Assistant Professor Florian Tramèr, ETH Zurich}
\maketitle

\frontmatter

\begin{abstract}

Removing the influence of specific data or concepts from trained models—\emph{machine unlearning}—is increasingly necessary for legal, ethical, and safety reasons. The canonical formulation is as follows: given a trained model \(F\) on dataset \(D\) and a forget set \(D_F \subset D\), the goal is to update the model parameters to obtain one that behaves as if it were never trained on \(D_F\), while preserving utility on \(D_R = D \setminus D_F\). \emph{Class unlearning} is a special case in which \(D_F\) contains all samples from one class. This thesis develops a unified framework that defines unlearning for classification models via the observed behavior of models retrained on \(D_R\). We propose new methodologies for both unlearning random set of samples and class unlearning and show that they outperform existing methods.
  
Evaluation methods for machine unlearning in classification models rely on the differences of prediction in the unlearned model with those of the prediction models. Therefore, the main driver of our unlearning methods is the \emph{similarity of predictions to a retrained model} on both the forget and remain samples. Exploiting this observation, we introduce \emph{Adversarial Machine UNlearning} (\amun{}), which surpasses prior state-of-the-art methods for image classification based on SOTA MIA scores. \amun{} lowers the model’s confidence on forget samples by fine-tuning on their corresponding adversarial examples. Because adversarial examples lie on the model-induced input distribution, fine-tuning on those nearest to each forget sample (a) localizes changes to the decision boundary and (b) avoids drastic shifts in global behavior, thereby preserving test accuracy. Using \amun{} to unlearn a random $10\%$ subset of CIFAR-10, we observe that even strong membership-inference attacks perform no better than random guessing.
  
Through theoretical analysis, we identify two factors that govern \amun{}'s performance: (1) the model’s \emph{Lipschitz constant} and (2) the \emph{transferability of adversarial examples} from the original to the retrained model. To facilitate training of smooth models with a controlled Lipschitz constant, we propose \emph{FastClip}, a scalable method that performs layer-wise spectral-norm clipping of affine layers. In a separate study, we show that increased smoothness naturally improves adversarial example transfer, thereby supporting the second factor above. We show the adverse effect of this increased transferability rate in the setting of ensembles and propose Layer-wise Orthogonalization for Training rObust enSembles (LOTOS), that alleviates this effect. 
  
Following the same principles for class unlearning, we show that existing methods fail in \emph{replicating a retrained model's behavior} by focusing only the probability assigned to the forgotten class and not leaving out the probabilities assigned to the other classes. We introduce a nearest-neighbor membership inference attack (MIA-NN) that uses the probabilities assigned to neighboring classes to detect unlearned samples and demonstrate the vulnerability of such methods. We then propose a fine-tuning objective that mitigates this leakage by approximating, for forget-class inputs, the distribution over remaining classes that a model retrained from scratch would produce. To construct this approximation, we estimate inter-class similarity and \emph{tilt} the target model’s distribution accordingly. The resulting \emph{Tilted ReWeighting} (TRW) distribution serves as the desired target during fine-tuning. Across multiple benchmarks, TRW matches or surpasses existing unlearning methods on prior metrics; on CIFAR-10, it reduces the gap to retrained models by $19\%$ (U-LiRA) and $46\%$ (MIA-NN) relative to the state-of-the-art baseline.

\end{abstract}

\begin{dedication}
To my parents, Fariba and Kohzad; my dear sister, Zahra; and my lovely fiancée, Faezeh.
\end{dedication}

\begin{acknowledgments}
This project would not have been possible without the support of many people. I am deeply grateful to my advisor, Hari Sundaram, for his valuable support and feedback throughout my research, and for helping me become a critical thinker and develop a concrete research plan. I extend special thanks to Varun Chandrasekaran, who worked closely with me on most of my projects and helped me improve my work and research skills. I thank Matus Telgarsky for inspiring me to strive for quality in research. I also thank my committee members, Gang Wang and Florian Tramèr, for their guidance and support. Thanks to the Siebel School of Computing and Data Science for awarding me the Gene Golub Fellowship and providing the resources that enabled me to complete this thesis. I would also like to acknowledge my colleagues at Crowd Dynamics Lab for their constructive research discussions, feedback, and collaborations.

Finally, my deepest thanks go to my beloved fiancée, Faezeh, for making this long journey easier at every step; to my parents, Kohzad and Fariba, and my sister, Zahra, for their constant, unconditional love; and to my dear friends, who made these years lighter and more joyful.
\end{acknowledgments}

{
    \hypersetup{linkcolor=black}  
    \tableofcontents
}

\chapter{List of Abbreviations}

\begin{abbrevlist}
\item[AMUN] Adversarial Machine UNlearning.
\item[MIA] Membership Inference Attack.
\item[MIA-NN] Membership Inference Attack via Nearest-Neighbor.
\item[LLM] Large Language Model.
\item[DPO] Direct Preference Optimization.
\item[LOTOS] (Method name as cited in the text).
\item[TRW] Tilted Reweighting.
\item[TRW-2R] Tilted Reweighting (two-round variant).
\item[FT] Fine-Tuning (baseline).
\item[RL] Reported baseline label in tables.
\item[GA] Gradient Ascent (baseline).
\item[$\ell_1$-ul] $\ell_1$-regularized unlearning (baseline).
\item[$\ell_2$-ul] $\ell_2$-regularized unlearning (baseline).
\item[BU] Reported baseline label in tables.
\item[SalUn] Saliency-based Unlearning.
\item[SVD] Singular Value Decomposition.
\item[SCRUB] Labeled unlearning baseline.
\item[SCAR] Labeled unlearning baseline.
\end{abbrevlist}

\chapter{List of Symbols}

\begin{symbollist}[0.7in]
\item[$P_X$] Input distribution over domain $X$.
\item[$Y=\{1,\dots,m\}$] Label set for $m$-class classification.
\item[$F:X\to Y$] (Multi-class) classifier mapping inputs to labels.
\item[$f(x)$] Prediction/probability function (e.g., softmax outputs).
\item[$\ell_F:X\times Y\to\mathbb{R}^+$] Loss function (e.g., cross-entropy).
\item[$D=\{(x_i,y_i)\}_{i=1}^N$] Labeled training set with $x_i\sim P_X$, $y_i\in Y$.
\item[$E_D(\ell_F(x,y))$] Empirical risk minimized during training.
\item[$D_T$] Test set sampled from $P_X$.
\item[$L$-Lipschitz] Function regularity: $\|g(x)-g(x')\|_2 \le L\|x-x'\|_2$.
\item[$A_F(x,\epsilon)$] Attack algorithm producing $x+\delta_x$ with $\|\delta_x\|_2\le\epsilon$ (untargeted, $y'\neq y$).
\item[$DF$] Forget set (subset of the training data $D$).
\item[$M_{D,DF}:\Theta\to\Theta$] Machine-unlearning map; given $DF\subset D$, outputs updated parameters $\theta_u$.
\item[$D_R=D-DF$] Remain set after removing $DF$ from $D$.
\item[$F_R$] Model retrained from scratch on $D_R$.
\item[$ACC_r$] Retained-class accuracy (higher is better, $\uparrow$).
\item[$ACC_f$] Forgotten-class accuracy (lower is better, $\downarrow$).
\item[$\uparrow$] “Higher is better” marker in tables.
\item[$\downarrow$] “Lower is better” marker in tables.
\item[$\tau$] Time taken to drink one cup of coffee.
\item[$\mu$g] Micrograms (of caffeine, generally).
\end{symbollist}

\mainmatter

\chapter{Introduction}
\label{sec:intro}

In this project we elaborate on our recent findings on important factors influencing the robustness and safety of machine unlearning methods. The goal of \textit{machine unlearning} is to remove the influence of a subset of the training dataset for a model that has been trained on that dataset~\cite{vatter2023evolution}. Since its proposal~\cite{cao2015towards} it has become increasingly important and the necessity for these methods has been determined by privacy regulations such as the European Union’s General Data Protection Act and the California Consumer Privacy Act. Machine unlearning originated in the classification setting. Early work by \cite{ginart2019making} formalized the notion of quantifiable unlearning for K-means, providing the first provable guarantee that a model’s behaviour converges to that of an oracle retrained from scratch after each deletion. \cite{bourtoule2021machine} then introduced SISA training—Sharding, Isolation, Slicing, and Aggregation—which strategically partitions the training data so that only a small subset of shards must be retrained when a datapoint is withdrawn; empirical results on deep networks showed an order of magnitude in speed-up over full retraining while achieving exact deletion. \cite{guo2019certified} proposed the notion of certified data removal, designing algorithms that supply verifiable statistical certificates that the post-unlearning classifier is indistinguishable from one trained without the erased data; their method, however, was only applicable to linear and logistic regression. Collectively, these seminal papers established the core trade-off space—exactness, computational cost, and verifiability—that still frames contemporary work on unlearning for conventional supervised models.
Despite early efforts on proposing ``exact'' solutions to this problem, the community has favored ``approximate'' solutions due to their ability to preserve the original model's accuracy while being more computationally efficient~\cite{chen2023boundary,liu2024model,fan2023salun}. 

We have designed an approximate unlearning method that utilizes the adversarial examples for unlearning the forget set (section~\ref{sec:amun-chapter}). We propose Adversarial Machine UNlearning (\amun{})~\cite{ebrahimpournot}. \amun{} is a method that, when applied to a trained model, replicates that (desired) behavior of the models retrained on the remaining samples after a few iterations. The success of \amun{} relies on an intriguing observation: fine-tuning a trained model on the adversarial examples of the training data does not lead to a catastrophic forgetting and instead has limited effect on the deterioration of model's test accuracy. One of the clear advantages of our method is its ability to perform effective unlearning in the absence of retained data. Through various experiments, we have shown that it outperforms all prior methods.

Through theoretical analysis on the effectiveness of our proposed unlearning method, we characterize influencing factors, including (a) Lipschitz constant of the model and (b) transferability of adversarial examples from the original model to the retrained one. In this thesis we also introduce a set of efficient methods to control the spectrum of the affine layers in a model, which leads to an upper-bound for the Lipschitz constant of the model. This helps to train models that satisfy the influencing factor (a). We also study their effectiveness in increasing the robustness of individual models against adversarial examples~\cite{boroojeny2024spectrum}. 

In another study we show that lower Lipschitz constant of the models leads to an increased rate of transferability of adversarial example, which shows that the influencing factor (b) is automatically satisfied on a model that has a controlled Lipschitz constant using our developed methods. We then focus on how this Lipschitz continuity can act as a double-edged sword in the setting of ensembles by increasing the transferability rate of adversarial examples among the models within the ensemble. We show that although a lower Lipschitz constant increases the robustness of a single model, it is not as beneficial in training robust ensembles as it increases the transferability rate of adversarial examples across models in the ensemble. Therefore, we introduce LOTOS, a new training paradigm for ensembles which counteracts this adverse effect. It does so by promoting orthogonality among the top-sub-spaces of the transformations of the corresponding affine layers of any pair of models in the ensemble~\cite{ebrahimpour2024lotos}. A byproduct of these prior works is the precise tools they provide for controlling the spectrum of the transformations of each layer of neural network and how they can be orthogonalized with respect to arbitrary subspaces. 

We then turn our attention to a variant of machine unlearning in which we want to forget all the samples that belong to a certain class and derive a model that acts as if it has never been trained on the forgotten class.
The setting of class unlearning is different from unlearning random training samples because the unlearned model has to be depleted of any sign of the unlearned class. Failing to consider this difference, we show that all the existing methods are susceptible to our newly designed variant of membership inference attacks (MIAs), which is based on the \emph{nearest neighbor} of the target class (MIA-NN)~\cite{ebrahimpour2025necessity}. This vulnerability has remained unnoticed because previously used evaluations and metrics are mainly designed or inspired by prior work on unlearning random subsets of data, and without paying attention to the properties of a model that is retrained from scratch on all the data other than the target class.

We then propose a new training objective that enhances the robustness to MIA-NN, without computational overhead over finetuning-based unlearning methods. When a class is marked for removal, this method first reassigns its predicted probability mass proportionally to the remaining classes, and then tilts this distribution according to inter-class similarities to better approximate the distribution of the models retrained from scratch on the remaining classes. Our proposed method not only exhibits strong robustness against MIA-NN, it also does not fail the evaluations using prior evaluation metrics.

\section{Outline of the Presented Works}

This thesis can be split into three thrusts. The first one establishes efficient methods for controlling the Lipschitz constant of affine layers and orthogonalizing the spectrum of the learned parameters. The second presents an efficient machine unlearning method for unlearning a random set of training samples from classification models. The third thrust proposes an efficent method for unlearning a class in classification models, as well as introducing a new membership inference attack based on neighboring classes.

\begin{itemize}
    \item \textbf{Shaping the Spectrum of Affine Layers.} In this thrust of work, we designed state-of-the-art method for clipping the spectral norm of affine layers in an efficient and scalable way~\cite{boroojeny2024spectrum}. This also leads to strict upper-bounds on the Lipschitz constant of neural network architectures that consists of these layers. We also designed a method that effectively orthogonalizes the transformation of different affine layers with respect to each-other~\cite{ebrahimpourtraining}. Our method, which is highly efficient for convolutional layers, was designed to prevent an increase in the transferability rate of adversarial examples among multiple models with bounded Lipschitz constant that co-exist in an ensemble of models. Although this approach was mainly used to increase the robustness of ensemble of models, the same techniques could be used to control the spectrum of affine layers and orthogonalize their spectrum with respect to arbitrary transformations.

    \item \textbf{Machine Unlearning for Classification Models.} Approximate unlearning methods have fallen short of reaching the effectiveness of exact unlearning: models produced fail to obtain comparable accuracy and prediction confidence on both the forget and test (i.e., unseen) dataset. Exploiting this observation, we propose a new unlearning method, Adversarial Machine UNlearning (\amun{}), that outperforms prior state-of-the-art (SOTA) methods for image classification~\cite{ebrahimpour2025amun}. \amun{} lowers the confidence of the model on the forget samples by fine-tuning the model on their corresponding adversarial examples. Adversarial examples naturally belong to the distribution imposed by the model on the input space; fine-tuning the model on the adversarial examples closest to the corresponding forget samples (a) localizes the changes to the decision boundary of the model around each forget sample and (b) avoids drastic changes to the global behavior of the model, thereby preserving the model's accuracy on test samples. Using \amun{} for unlearning a random $10\%$ of CIFAR-10 samples, we observe that even SOTA membership inference attacks cannot do better than random guessing. Although our approach was designed for unlearning in classification models, the same ideas and approaches can be extended to unlearning in image generation models.

    \item \textbf{Machine Unlearning for LLMs.} In this thrust of work we reveal a significant shortcoming in class unlearning evaluations: overlooking the underlying class geometry can cause privacy leakage. We further propose a simple yet effective solution to mitigate this issue. We introduce a membership-inference attack via nearest neighbors (MIA-NN) that uses the probabilities the model assigns to neighboring classes to detect unlearned samples. Our experiments show that existing unlearning methods are vulnerable to MIA-NN across multiple datasets. We then propose a new fine-tuning objective that mitigates this privacy leakage by approximating, for forget-class inputs, the distribution over the remaining classes that a retrained-from-scratch model would produce. To construct this approximation, we estimate inter-class similarity and tilt the target model’s distribution accordingly. The resulting Tilted ReWeighting (TRW) distribution serves as the desired distribution during fine-tuning. We also show that across multiple benchmarks, TRW matches or surpasses existing unlearning methods on prior unlearning metrics. More specifically, on CIFAR-10, it reduces the gap with retrained models by $19\%$ and $46\%$ for U-LiRA and MIA-NN scores, accordingly, compared to the SOTA method for each category.
    
\end{itemize}

\newpage
\chapter{Related Work}

We first elaborate on some of the prior work in computing and controlling the spectral norm. For our approach on orthogonalizing the spectrum of affine layers with respect to arbitrary transformations, we will provide the related works on ensemble robustness because that was the area in which our method was show-cased. 

\section{Controlling the Spectral Norm}

\cite{miyato2018spectral} approximate the original transformation by reshaping the kernel to a $n_{in}k^2 \times n_{out}$ matrix and perform the power-iteration they designed for the dense layers to compute the spectral norm of the convolutional layers. To clip the spectral norm, they simply scale all the parameters to get the desired value for the largest singular value.~\cite{farnia2018generalizable} and~\cite{gouk2021regularisation} use the transposed convolution layer to perform a similar power iteration method to the one introduced by~\cite{miyato2018spectral} for dense layers, and then perform the same scaling of the whole spectrum.~\cite{virmaux2018lipschitz} perform the power method implicitly using the automatic differentiation that is similar to our implicit approach; however, they do not provide any algorithm for clipping the spectral norms and leave that for future work. Also, for extracting multiple singular values, they use successive runs of their algorithm followed by deflation, which is much less efficient than our extraction method.
Some other works consider additional constraints.~\cite{cisse2017parseval} constrained the weight matrices to be orthonormal and performed the optimization on the manifold of orthogonal matrices.~\cite{sedghi2018singular} gave a solution that works only for convolutional layers with circular padding and no stride. For other types of convolutional layers, they approximate the spectral norm by considering the circular variants; however, the approximation error can be large for these methods.~\cite{senderovich2022towards} extends the method of~\cite{sedghi2018singular} to support convolutional layers with strides. It also provides an optional increase in speed and memory efficiency, albeit with the cost of expressiveness of convolutions, through specific compressions.~\cite{delattre2023efficient} use Gram iteration to derive an upper-bound on the spectral norm of the circular approximation in a more efficient way. However, they do not provide a solution for strides other than $1$. Therefore, in the best case, they will have the same approximation error as~\cite{sedghi2018singular}. A parallel line of research on controlling the spectrum of linear layers seeks to satisfy an additional property, gradient norm preserving, in addition to making each of the dense and convolutional layers $1$-Lipschitz~\cite{singla2021skew,yu2021constructing,trockman2021orthogonalizing,singla2021improved,xu2022lot,achour2022existence}. Some state-of-the-art methods in this line of work are also based on the clipping method introduced by~\cite{sedghi2018singular}~\cite{yu2021constructing,xu2022lot}. Therefore, those approaches are even slower and less efficient.

\subsection{Ensemble Robustness}
To increase the diversity of the models in the ensemble,~\cite{kariyappa2019improving} proposed misalignment of the gradient vectors with respect to the inputs. The intuition behind this idea is how several earlier works generate adversarial examples~\cite{goodfellow2014explaining,kurakin2018adversarial,papernot2016limitations}: these methods use the gradient direction with respect to a given input to find the direction that increases the loss function the most. By moving a small amount toward that direction, examples that are similar to the original example but are misclassified by the model are found. In their paper,~\cite{kariyappa2019improving} hypothesize that when this gradient direction is the same for various models of the ensemble, the common subspace of their adversarial examples will be larger because the loss function behaves similarly around the original data. They use cosine similarity to capture this similarity in the direction; by incorporating that for pairs of models in the loss function, they diversify the models in the ensemble to make it more robust.

\cite{pang2019improving} propose a regularizer to increase the diversity in an ensemble by increasing the entropy in non-maximal predictions.~\cite{yang2020dverge} use an adversarial training objective to increase the diversity of the ensemble by making the non-robust features more diverse.~\cite{yang2021trs} suggest that not only does the misalignment of the gradient vectors of the loss with respect to the inputs matters, but also the Lipschitz constant of the gradients (not parameters) of the loss with respect to the inputs has to decrease and propose a heuristic method to achieve this, which outperforms the prior methods.~\cite{zhang2022building} propose a method based on margin-boosting to diversify the models of an ensemble which did not achieve better results than the prior state-of-the-art methods~\cite{yang2020dverge,yang2021trs}. More recent works try to enhance the robustness of the existing methods by incorporating adversarial training against a variety of publicly available models~\cite{sitawarin2023defending} or enhance their time complexity by using faster (but weaker) attacks for data augmentation and enforcing diversity in the latent space~\cite{huang2023fasten}.  



There is a different line of work on the orthogonality of the affine layers of the deep learning models which focuses on making the transformation of a single layer orthogonal (i.e., its rows and columns become orthonormal vectors). This helps to preserve the gradient norm of the layer and has been shown to improve the stability and robustness of the models~\cite{trockman2021orthogonalizing,singla2021skew,xu2022lot,singla2021improved,prach2022almost,hu2023recipe}. This notion of orthogonality is different from what we consider in this work; we wish to make the transformation of ``corresponding layers'' from different models {\em orthogonal with respect to each other}.

\subsection{Machine Unlearning}

Early works in machine unlearning focused on exact solutions~\cite{cao2015towards,bourtoule2021machine}; those ideas were adapted to unlearning in other domains such as graph neural networks~\cite{chen2022graph} and recommendation systems~\cite{chen2022recommendation}. The extensive computational cost and utility loss resulted in the design of approximate methods. An example is the work of~\cite{ginart2019making}, who provide a definition of unlearning based on differential privacy. Works that followed sought solutions to satisfy those probabilistic guarantees~\cite{ginart2019making,gupta2021adaptive,neel2021descent,ullah2021machine,sekhari2021remember}. However, the methods that satisfy these guarantees were only applied to simple models, such as $k$-means~\cite{ginart2019making} , linear and logistic regression~\cite{guo2019certified,izzo2021approximate}, convex-optimization problems~\cite{neel2021descent}, or graph neural networks with no non-linearities~\cite{chien2022efficient}. 
Additional research was carried out to design more scalable approximate methods, those 
that can be applied to the models that are used in practice, including large neural networks~\cite{golatkar2020eternal,warnecke2021machine,izzo2021approximate,thudi2022unrolling,chen2023boundary,liu2024model,fan2023salun}. However, these approximate methods do not come with theoretical guarantees; their effectiveness are evaluated using membership inference attacks (MIAs). 
MIAs aim to determine whether a specific data sample was used in the training set of a trained model~\cite{shokri2017membership,yeom2018privacy,song2019privacy,hu2022membership,carlini2022membership,zarifzadeh2024low}, and is a common evaluation metric~\cite{liu2024model,fan2023salun}.

To the best of our knowledge \amun{} is the first work that considers fine-tuning of a model on the adversarial examples with their wrong labels as a method for unlearning a subset of the samples. However, upon reviewing the prior works in unlearning literature, there are several works that their titles might suggest otherwise. Therefore, here we mention a few of these methods and how they differ from our work.

To improve upon fine-tuning on samples in $\Df$ with randomly chosen wrong labels,~\cite{chen2023boundary} use the labels derived from one step of the FGSM attack to choose the new labels for the samples in $\Df$. This method which was presented as \texttt{BS} in our experiments (\S~\ref{sec:amun_results}), does not use the adversarial examples and only uses their labels as the new labels for samples of $\Df$. This corresponds to the dataset $\D_{AdvL}$ in \S~\ref{sec:ablation-finetune}. As our results in Figures~\ref{fig:fine_tune_10} and~\ref{fig:fine_tune_50} show, fine-tuning the trained model on this dataset leads to catastrophic forgetting even when $\Dr$ is available. This is simply due to the fact that the samples in $\D_{AdvL}$ contradict the distribution that the trained models have already learned.

The work by~\cite{setlur2022adversarial} is not an unlearning method, despite what the name suggest. They propose a regularization method that tries to maximize the loss on the adversarial examples of the training samples that are relatively at a higher distance to lower the confidence of the model on those examples. The work by~\cite{zhang2024defensive} proposed a defense method similar to adversarial training for making to unlearned LLMs more robust to jailbreak attacks on the topics that they have unlearned. \cite{lucki2024adversarial} also study the careful application of jailbreak attacks against unlearned models. The work by~\cite{jung2024attack} investigate computing adversarial noise to mask the model parameters. Many of the works with similar titles, use ``adversarial" to refer to minimax optimization~\cite{zeng2021adversarial} or considering a Stackelberg game setting between the source model and the adversary that is trying to extract information~\cite{di2024adversarial}.

\subsection{Class Unlearning}
In this section, we review prior work on machine and class unlearning, highlighting key approaches such as retraining, fine-tuning, pruning, and representation manipulation.

\paragraph{Machine Unlearning.} Early work on machine unlearning focused on completely retraining models from scratch on the retained data. While accurate, this approach is often impractical for deep networks~\cite{bourtoule2021machine}. Researchers have thus developed approximate methods without full retraining. 
One line of work uses influence functions to estimate parameter changes from removing specific data points~\cite{warnecke2021machine}, enabling efficient unlearning on entire features or class labels with theoretical guarantees.
Another direction is to approximate data deletion with minimal cost. \cite{izzo2021approximate} propose a projection-based update for linear and logistic regression models that achieves data deletion in $O(d)$ time. \cite{thudi2022unrolling} unroll the SGD training trajectory to quantify unlearning. 
And \cite{jia2023model} propose to prune a model before unlearning, which significantly closes the gap between approximate and exact unlearning. 
Another category of techniques scrubs specific knowledge from model weights via targeted fine-tuning. \cite{golatkar2020eternal} propose to inject noise guided by the fisher information matrix to remove information about the forget set. 
More recently, \cite{fan2023salun} proposed Saliency Unlearning (SalUn), which computes a weight saliency map to identify parameters influential in the forgetting set and adjusts only those weights. 
~\cite{foster2024fast} propose selective synaptic damping, which uses fish information to identify and dampen parameters important to the target data.
~\cite{cha2024learning} performs instance-wise unlearing through intentionally misclassifying forget samples while preserving utility via adversarial regularization and weight-importance constraints.
~\cite{bonato2024retain} proposes a  model-agnostic, retain-set–free method that moves forget features toward a nearest wrong class and preserves utility via OOD distillation.
~\cite{kurmanji2023towards} frames unlearning as selective teacher–student training—aligning on retain data and repelling on forget data that improves scalability and MIA robustness.

Recent works tackle the problem of removing an entire class's influence while retaining accuracy on other classes. One strategy is to directly adjust the model's decision boundaries. \cite{chen2023boundary} propose to shift decision boundaries via fine-tuning on relabeled or pseudo-class data. 
\cite{chang2024zero} propose a Neuron Attribution method using layer-wise relevance propagation to find and perturb class-specific activation paths. 
And \cite{shen2024camu} construct counterfactual samples to unlearn the class by aligning it with random noise representations.
Other strategies involve teacher-student learning and representation subtraction. 
\cite{chundawat2023can} train a student model through a competent teacher (full knowledge) and an incompetent one (forgetting class). 
\cite{tarun2023fast} apply a two-step impair-repair fine-tuning method to rapidly forget the class.
And \cite{kodge2024deep} identify forget and retain spaces via singular value decomposition and subtract shared components to remove class-specific features. 


\paragraph{Membership Inference Attacks.}
A membership inference attack (MIA) tests whether a specific example was part of a model’s training set by exploiting the fact that overfitted models tend to behave differently on seen (members) vs.\ unseen points (non-members). The original MIA has a shadow-model attack that 
queries a target to collect confidence vectors and trains an attack classifier to decide membership~\cite{shokri2017membership}. 
Reframing evaluation toward the low-FPR regime,~\cite{carlini2022membership} builds a per-example likelihood ratio that improves TPR at small FPRs. And~\cite{kodge2024deep} proposes a low-cost, high-power MIA through a Support Vector Machine. More recently,~\cite{zarifzadeh2024low} designs a robust, low-cost statistical test by composing pairwise likelihood ratios against population draws, outperforming prior methods even with very few reference models. Also by adapting LiRA,~\cite{hayes2024inexact} introduces per-example unlearning MIAs, showing that stronger, tailored attacks reveal overestimated privacy in prior evaluations and can even degrade retain-set privacy, urging more rigorous U-MIA testing

\newpage
\chapter{Shaping the Spectrum of Affine Layers}
\label{sec:shaping-spectrum}

Deep learning models are very brittle to input changes:~\cite{szegedy2013intriguing} were the first to carefully craft ``adversarial'' perturbations that result in incorrect classification outcomes. Subsequent research showed that these adversarial examples ``transfer'' to models with different hyper-parameters, and even different hypothesis classes~\cite{papernot2016transferability, liu2016delving}. This transferability property was used to design black-box attacks against models for which only query access is available~\cite{papernot2017practical}. In these attacks, the adversary trains a local model that is ``similar'' to the victim (black-box) model and uses that to find transferable adversarial examples~\cite{xiao2018generating}. To increase the rate of transfer success, a common strategy is to choose those inputs that are adversarial to an ``ensemble'' of models~\cite{liu2016delving,chen2023rethinking}: fooling an ensemble may result in fooling a potentially unseen model. 

Research has also been carried out on understanding and improving model resilience to such attacks~\cite{dong2019evading}. Adversarial robustness is the innate ability of the model to correctly classify adversarial examples~\cite{madry2017towards}. One way of increasing this robustness is utilizing a diverse\footnote{In terms of the parameters and decision boundaries.} ensemble of models~\cite{yang2020dverge,yang2021trs}. This has also been considered as a mitigation to the transferability problem~\cite{pang2019improving,kariyappa2019improving,yang2020dverge,yang2021trs,sitawarin2023defending}; by increasing the diversity among models of the ensemble, the subspace of the adversarial examples that are effective against most of the models within the ensemble shrinks. While empirical evidence is promising, most of these ensemble robustness methods are either computationally expensive or come at a considerable cost to the accuracy of the models and the overall ensemble. In another vein of research, Lipschitz continuity was also shown to be important for robustness~\cite{szegedy2013intriguing,farnia2018generalizable,boroojeny2024spectrum,ebrahimpour2025amun}. Since the model's Lipschitz continuity controls how the predictions change for small changes in the input, it is intuitive that bounding it will improve robustness. It is important to note that neural networks are a composition of multiple layers and therefore the existing works in this area obtain an upper-bound on the overall Lipschitz constant of the network by bounding each layer independently~\cite{szegedy2013intriguing,sedghi2018singular,senderovich2022towards,delattre2023efficient,boroojeny2024spectrum}.



In our work, {\em we investigate the effect of Lipschitz continuity on the transferability of adversarial examples}. We observe that while decreasing the Lipschitz constant makes each model of the ensemble {\em individually more robust}, it makes them less diverse and consequently {\em increases the transferability rate} among them which in turn {\em hinders the overall ensemble robustness} (\S~\ref{sec:lotos_motivation},~\Cref{fig:vanilla_compare}). To resolve this adverse effect, we introduce our novel training paradigm, \textbf{Layer-wise Orthogonalization for Training rObust enSembles (\texttt{LOTOS})}~\cite{ebrahimpourtraining}, which orthogonalizes the corresponding affine layers of the models with respect to one another. This increases the diversity of the trained models. \texttt{LOTOS} can be combined with any prior method of training diverse ensembles to further boost their robustness.

Through extensive experiments and ablation studies, we show that \texttt{LOTOS} effectively decreases the transferability rate among the models of an ensemble, which leads to a higher robustness against black-box attacks and non-adversarial noise. As we will show, \texttt{LOTOS} is highly efficient for convolutional layers and is very fast compared to prior methods. We also show that it is an effective method for training robust ensembles of heterogeneous architectures, where other state-of-the-art (SOTA) methods are not applicable. Finally, we investigate how \texttt{LOTOS} is able to improve the results of the prior SOTA methods when they are combined. 

\section{Preliminaries}

After introducing the notations, we define the notion of transferability rate (\S~\ref{sec:defs}) and continue with introducing our conjecture about the effect of Lipschitz continuity on transferability rate (\S~\ref{sec:conjecture}).


\subsection{Notation}

An implicitly linear layer can be written as an affine function $f(x) = M_W x + b$, where $x \in \R^n$, $M_W \in \R^{m\times n}$, and $b \in \R^m$. $M_W$ is not necessarily the explicit form of the parameters of the layer, in which case the explicit form is shown with $W$ (e.g., the kernel of convolutional layers); otherwise, $M_W$ is the same as $W$. The spectral norm of $M_W$, which is the largest singular value of $M_W$, is shown with $\|M_W\|_2$. $\sigma_i(W)$ is used to represent the $i$-th largest singular value of matrix $M_W$. The largest singular value might be referred to as either $\|W\|_2$ or $\|M_W\|_2$. We use $\omega=\mathrm{exp}(2\pi i/n)$ (the basic $n$-th roots of unity), where $n$ is the rank of the linear transformation. $\Re(.)$ returns the real part of its input, and we also define the symmetric matrix $A_W$ as $M_W^\top M_W$. Hereinafter $[k]$ will be used to show the set $\{0,1,2,\dots, k\}$. When studying the composition of affine functions, we use the word “concatenation” to refer to the architecture of the network (i.e., succession of the layers), and “composition” to refer to the mathematical form of this concatenation. Given the domain of inputs $\mathcal{X}$ and $m$ classes $\mathcal{Y} = \{1,2,\dots,m\}$, we consider a multi-class classifier $\mathcal{F}:\mathcal{X} \rightarrow \mathcal{Y}$ and its corresponding prediction function $f(x)$ which outputs the probabilities corresponding to each class (e.g., the outputs of the softmax layer in a neural network). The loss function for model $\mathcal{F}$ is denoted $\ell_\mathcal{F}: \mathcal{X} \times \mathcal{Y} \rightarrow \mathbb{R_+}$; it uses the predicted scores from $f(x)$ to compute the loss given the true label $y$ (e.g., cross-entropy loss). The population loss for model $\mathcal{F}$, which we may also refer to as risk, is defined as $R_\F(x,y) = \E_x [\ell_\mathcal{F}(x,y)$]. When the models are deep neural networks, we refer to a specific layer using superscripts (e.g., $f^{(i)}$ for the $i$-th layer of deep network). A funtion $f(x)$ is $L$-Lipschitz if $\| f(x) - f(x^\prime)\|_2 \leq L \| x-x^\prime \|_2, \forall x,x^\prime \in \mathcal{X}$. For a matrix $A$, the spectral norm is can be equivalently defined as $\|A\|_2 = \mathrm{sup}_{x\neq 0} \frac{\|Ax\|_2}{\|x\|_2}$, meaning that the transformation matrix $A$ is $\|A\|_2$-Lipschitz. When we say layer $i$ (with transformation matrix $A$) has been clipped to a value $C$, this means we ensure $\|A\|_2 \simeq C$.


\subsection{Definitions}
\label{sec:defs}

We begin with the definition of an adversarial attack and then formally define the \textit{transferability rate ($T_{rate}$)} of adversarial examples.


\begin{definition}[Attack Algorithm]
    For a given input/output pair $(x,y) \in \mathcal{X} \times \mathcal{Y}$, a model $\mathcal{F}$, and a positive value $\epsilon$, a targetted attack algorithm $\A_\mathcal{F}^{(t)}(x) = x+\delta_x$ minimizes $\ell_\F(x+\delta_x,y_t)$ such that $\|\delta_x\|_2 \leq \epsilon$. An untargeted attack $\A(x)$ maximizes $\ell_\F(x+\delta_x,y)$. 
\end{definition}

\begin{definition}[Transferability Rate]
\label{def:trans}
For an untargeted adversarial algorithm $\A_\mathcal{F}$ and input space $\mathcal{X}$, we define the transferability rate ($T_{rate}$) of $\A_\mathcal{F}(x)$ from $\mathcal{F}$ to another classifier $\mathcal{G}$, as the following conditional probability:
\begin{equation}
T_{rate}(\A_\mathcal{F}, \mathcal{F}, \mathcal{G}) = \mathbb{P}_{(x,y)\in \mathcal{X} \times \mathcal{Y}} \left[\mathcal{G}(\A_\mathcal{F}(x)) \neq y \mid \; \mathcal{F}(x) = \mathcal{G}(x) = y \land \mathcal{F}(\A_\mathcal{F}(x)) \neq y \right].
\end{equation}

For the transferability of a targeted attack algorithm $\A_\mathcal{F}^{(t)}$, and target class $y_t$ this definition is:

\begin{equation}
T_{rate}(\A_\mathcal{F}^{(t)}, \mathcal{F}, \mathcal{G}) = \mathbb{P}_{(x,y)\in \mathcal{X} \times \mathcal{Y}} \left[\mathcal{G}(\A_\mathcal{F}^{(t)}(x)) = y_t \mid \; \mathcal{F}(x) = \mathcal{G}(x) = y \land \mathcal{F}(\A_\mathcal{F}^{(t)}(x)) = y_t \right].
\end{equation}
\end{definition}

Note that~\cite{yang2021trs} have similar definitions of transferability except they use joint probabilities instead of conditional probabilities. Using the conditional probability, we can better isolate the ``transferability'' property we are interested in. By using the joint probability, the $T_{rate}$ will depend on the accuracy of the two models and also the performance of the attack algorithm on the source model $\mathcal{F}$. This will not allow us to have an accurate comparison of the $T_{rate}$ between different settings.

\section{Motivation}
\label{sec:lotos_motivation}







Prior work has shown that there is a trade-off between the accuracy of two models and the $T_{rate}$ of adversarial examples between them~\cite{yang2021trs}; as the two models become more accurate, their decision boundaries become more similar; this increases the probability that an adversarial example generated for one of the models transfers to the other model because of the similarity of their margins around the source sample. In the following section, by making a parallel to this trade-off, we introduce our conjecture on the trade-off between the Lipschitzness of the models and their $T_{rate}$.

\subsection{Our Conjecture: Lipschitz Continuity Influences Transferability}
\label{sec:conjecture}

Prior works~\cite{szegedy2013intriguing,farnia2018generalizable,boroojeny2024spectrum} highlight the importance of Lipschitzness on the robustness of a model against adversarial examples. To enforce the Lipschitzness of the deep neural networks, these works enforce the Lipschitzness on each individual layer: since the Lipschitz constant of the composition of all the components of the model (layers, activation functions, etc.) is upper bounded by the multiplication of their Lipschitz constants, this leads to a bound on the Lipschitzness of the whole model. In this work, we follow the same procedure to control the upper bound on the Lipschitz constant of a model. For this, we use FastClip~\cite{boroojeny2024spectrum} which is the current SOTA method for controlling the spectral norm of dense layers and convolutional layers. We use the chosen spectral norm for each individual layer to represent these models in the figures and tables. For example, $C=1$ shows that the spectral norm of all the dense layers and convolutional layers have been clipped to $1$, which makes each of them $1$-Lipschitz. For many architectures such as ResNet-18~\cite{he2016deep}, and DLA~\cite{yu2018deep} this effectively controls the Lipschitz constant of the model because the other constituent components (e.g., ReLU activation, max-pooling, softmax) are $1$-Lipschitz~\cite{goodfellow2014explaining}. 

The prior works have shown that by decreasing the value of the Lipschitz constant $L$ for a model, adversarial attacks achieve a lower success rate. This confirms that decreasing the Lipschitz constant for a model makes it more robust to input perturbations. However, the situation is complicated in an ensemble. While a lower Lipschitz constant promotes robustness for each of the models (of an ensemble), we conjecture that it will increase the $T_{rate}$ by making the classification boundaries of the models more similar. We formalize this intuition below:


\begin{figure}[t!]
\centering
    \includegraphics[width=.99\linewidth]{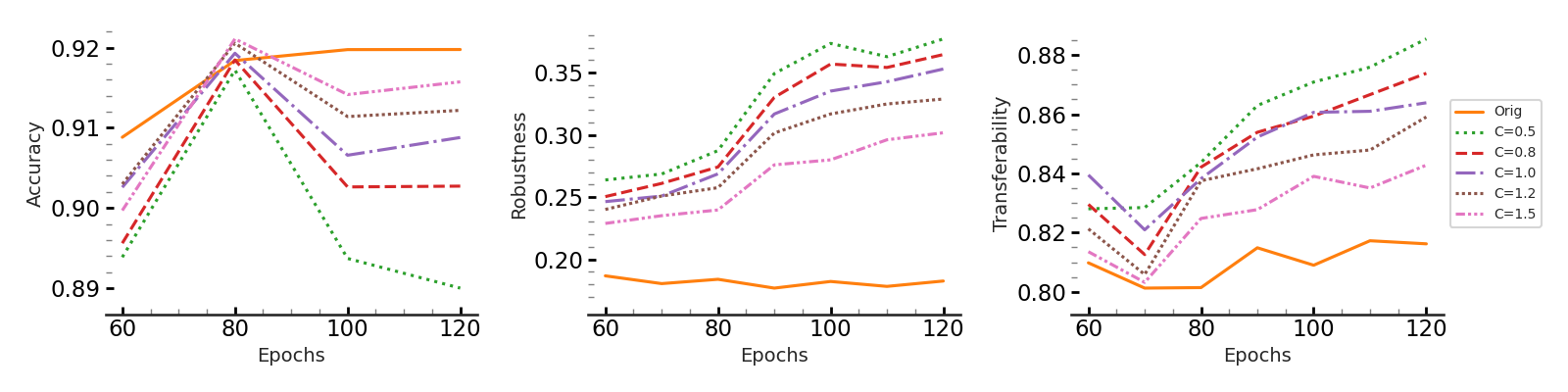} 
\caption{\footnotesize {\bf Accuracy vs. Robust Accuracy vs. Transferability:} Changes in the average accuracy and robust accuracy of \textit{individual} ResNet-18 models, along with the {\em average transferability rate between any pair of the models in each ensemble} as the layer-wise clipping value (spectral norm) changes. As the plots show, although the robustness of \textit{individual} models increases with decreasing the clipping value, the transferability rate among the models increases, which might forfeit the benefits of the clipping in the robustness of the whole ensemble.}
\vspace{-4mm}
\label{fig:vanilla_compare}
\end{figure}

\begin{proposition} 
\label{prop:trans}
    
    Assume $\X = [0,1]^d$ and $\|\delta_x\| \leq \epsilon$. For two models $\F$ and $\G$, if the loss function on both for any $y \in \Y$ is $L$-Lipschitz with respect to the inputs, we have the following inequality:

\begin{align*}
     | R_\F(\A_\mathcal{F}(x),y) - R_\G(\A_\mathcal{F}(x),y) | \leq 2L\epsilon + |R_\F(x,y) - R_\G(x,y)|.
\end{align*}

\end{proposition}

The proof for the proposition can be found in our paper. In Proposition~\ref{prop:trans}, we study the $T_{rate}$ of an adversarial example generated by $A_\mathcal{F}$ to model $\G$ by using the difference in the population loss of the two models on these adversarial examples as a proxy; the lower this difference is, the more likely that the two models will perform similarly on these adversarial examples. To verify this conjecture empirically, in Figure~\ref{fig:vanilla_compare} we show how $T_{rate}$ changes among three ResNet-18~\cite{he2016deep} models without batch norm layers as the layer-wise Lipschitz constant (which governs the upper-bound on the Lipschitz constant of the model as mentioned earlier) changes. For more details see \S~\ref{sec:tradeoff}. 

\begin{tcolorbox}
\noindent{\bf Main Takeaway:} According to Proposition~\ref{prop:trans}, decreasing $L$ (Lipschitz constant), might be an indicator of higher similarity in the risk of the two models on the adversarial examples, which might imply a higher $T_{rate}$. Therefore, although a lower Lipschitz constant could contribute to the robustness of a single model, it might increase the $T_{rate}$ among the models of an ensemble which might hinder the expected benefits of the Lipschitzness. 
\end{tcolorbox}

\section{Methods}

In this section, we start with explaining our efficient algorithms for extracting and controlling the spectral norm of affine layers. Then, building on these results present our efficient approach for orthogonalizing the spectrum of these layers with respect to arbitrary transformations.

\subsection{Spectrum Extraction}
\label{sec-extraction}



We introduce our \textit{PowerQR} algorithm, which performs an implicit version of the shifted subspace iteration algorithm on any implicitly linear layer. Shifted subspace iteration is a common method for extracting the spectrum of linear transformations; the shift parameter makes the algorithm more stable and faster than the regular power method. Also, it is much more efficient when multiple singular values and vectors are of interest, compared to iterative calls to the regular power method followed by deflation (see chapter 5 of~\cite{saad2011numerical} for more details of this algorithm). However, the direct application of this algorithm requires the explicit matrix form of the transformation. In~\Cref{alg:powerqr}, we show how auto-differentiation can be used to perform this algorithm in an implicit way, without requiring the explicit matrix form. Regarding the complexity of~\Cref{alg:powerqr}, other than the $O(n^2 k)$ complexity of QR decomposition in line $6$, it has an additional cost of computing the gradients of the layer at line $4$ for a batch of $k$ data points, which is dominated by the former cost.



\begin{algorithm}[t]
\caption{PowerQR ($f, X, N, \mu=1$)}\label{alg:powerqr}
\begin{algorithmic}[1]
\STATE {\bfseries Input:} Affine function $f$, Initial matrix $X \in \R^{n\times k}$, Number of iterations $N$, Shift value $\mu$ 
\STATE {\bfseries Output:}  Top $k$ singular values and corresponding right singular vectors
    \FOR{$i=1$ {\bfseries to} $N$}
        \STATE $X^\prime \gets \nabla_X \frac{1}{2} \|f(X) - f(0)\|^2$ \\ \COMMENT{$\nabla_x \frac{1}{2} \|f(x) - f(0)\|^2 = M^\top Mx$}
        \STATE $X \gets \mu X + X^\prime$
        \STATE $(Q, R) \gets \mathrm{QR}(X)$ \COMMENT{$\mathrm{QR}$ decomposition of $X$}
        \STATE $X \gets Q$
    \ENDFOR

\STATE $S \gets \sqrt{\mathrm{diag}(R - \mu I)}$ \COMMENT{Singular values}
\STATE $V \gets X$  \COMMENT{Right singular vectors}
\STATE {\bfseries Return} $S,V$

\end{algorithmic}
\end{algorithm}



\begin{proposition}
\label{lemma-shifted-subspace}
Let $f(x) = Mx + b$. Then~\Cref{alg:powerqr} correctly performs the shifted subspace iteration algorithm on $M$, with $\mu$ as the shift value.   
\end{proposition}

 Note that by subtracting the value of $f$ at $0$ in line $4$, we remove the bias term from the affine function, which does not affect the singular values and vectors. This technique works even for the composition of multiple affine functions (e.g., concatenation of implicitly linear layers).
 The convergence rate in the subspace iteration algorithm depends on the initial matrix $X$ (line $1$), and if there are vectors $s_i$ in the column space of $X$ for which $\|s_i - v_i\|_2$ is small, then it takes fewer steps for the $i$-th eigenvector to converge. Meanwhile, we know that SGD makes small updates to $W$ in each training step. Therefore, by reusing the matrix $V$ in line $10$ (top-$k$ right singular vectors) computed for $W_{i-1}$ as initial $X$ for computing the spectrum of $W_i$, we can get much faster convergence. Our experiments show that by using this technique, it is enough to perform only one iteration of PowerQR per SGD step to converge to the spectrum of $W$ after a few steps. Using a warm start (performing more iterations of PowerQR for the initial $W$) helps to correctly follow the top-$k$ singular values during the training. The prior work on estimating and clipping the singular values has exploited these small SGD updates for the parameters in a similar manner~\cite{farnia2018generalizable,gouk2021regularisation,senderovich2022towards}.

\subsection{Clipping the Spectral Norm}
\label{sec-clipping}

In this section, we introduce our algorithm for clipping the spectral norm of an implicitly linear layer to a specific target value. To project a linear operator $M = USV^\top$ to the space of operators with a bounded spectral norm of $c$, it is enough to construct $A^\prime = US^\prime V^\top$, where $S^\prime_{i,i} = \min (S_{i,i}, c)$~\cite{lefkimmiatis2013hessian}. Therefore, for this projection, we will not need to extract and modify the whole spectrum of $A$ (as in~\cite{sedghi2018singular,senderovich2022towards}), and only clipping the singular values that are larger than $c$ to be exactly $c$ would be enough. Note that after computing the largest spectral norm (e.g., using the PowerQR method) by simply dividing the parameters of the affine model by the largest singular value, the desired bound on the spectral norm would be achieved, and that is the basis of some prior work~\cite{miyato2018spectral,farnia2018generalizable,gouk2021regularisation}; however, this procedure scales the whole spectrum rather than projecting it to a norm ball. Therefore, it might result in suboptimal optimization of the network due to replacing the layer with one in the norm ball that behaves very differently.
This adverse effect can be seen in the results of our paper~\cite{boroojeny2024spectrum} as well. Thus, to resolve both aforementioned issues, our algorithm iteratively shrinks the largest singular value by substituting the corresponding rank $1$ subspace in an implicit manner. 


\begin{algorithm}[t]
\caption{Clip ($f_W, c,  N, P$)}\label{alg:clip}
\begin{algorithmic}[1]
\STATE {\bfseries Input:} Affine function $f_W= M_W x + b$, Clip value $c$, Number of iterations $N$, Learning rate $\lambda$, Number of iterations of PowerQR $P$
\STATE {\bfseries Output:}  Affine function $f_{W^\prime}$ with singular values clipped to $c$
\STATE $\sigma_1, v_1 \gets \mathrm{PowerQR}(f_{W}, v, P)$ \small{($v$: Random vector)}
    \WHILE{$\sigma_1 > c$}
    \STATE $W^\prime \gets W$ \COMMENT{$W^\prime = \sum_{i=1}^{n} u_i \sigma_i v_i^\top$}
        \FOR{$i=1$ {\bfseries to} $N$}
            \STATE $W^\prime_\delta \gets \nabla_{W^\prime} \frac{1}{2} \|f_{W^\prime}(v_1) - f_{W}(c \sigma_1^{-1} v_1)\|^2$ \\ \COMMENT{ $W^\prime_\delta = u_1 (\sigma_1 - c)v_1^\top$}
            \STATE $W^\prime \gets W^\prime - \lambda W^\prime_\delta$ 
        \ENDFOR
        \STATE $W \gets W^\prime$
        \STATE $\sigma_1, v_1 \gets \mathrm{PowerQR}(f_{W}, v, P)$ \small{($v$: Random vector)} \COMMENT{Update $\sigma_1$ and $v_1$}
    \ENDWHILE

\STATE {\bfseries Return} $f_{W^\prime}, \sigma_1, v_1$

\end{algorithmic}
\end{algorithm}

\Cref{alg:clip} shows our stand-alone clipping method. The outer \texttt{while} loop clips the singular values of the linear operator one by one. After clipping each singular value, the PowerQR method in line \texttt{11} computes the new largest singular value and vector, and the next iteration of the loop performs the clipping on them. The clipping of a singular value is done by the \texttt{for} loop. Considering that in the \texttt{for} loop $M:= M_{W} = M_{W^\prime} = \sum_{i=1}^{n} u_i \sigma_i v_i^\top$, since $v_i$s are orthogonal and $v_i^\top v_i = 1$, for line \texttt{7} we have:
\vspace{-1 pt}
\begin{align*}
    W^\prime_\delta &= \nabla_{W^\prime} \frac{1}{2} \|f_{W^\prime}(v_1) - f_{W}(c \sigma_1^{-1} v_1) \|^2 \\
    &=  \left(f_{W^\prime}(v_1) - f_{W}(c \sigma_1^{-1} v_1)\right) \nabla_{W^\prime} f_{W^\prime}(v_1) \\
    &= \left(Mv_1 + b - c\sigma_1^{-1}Mv_1 - b \right) v_1^\top \\
    &= \left(u_1 \sigma_1 - u_1 c \right) v_1^\top = u_1 \sigma_1 v_1^\top - u_1 c v_1^\top,
\end{align*}
\noindent where $\nabla_{W^\prime} f_{W^\prime}(v_1) = v_1^\top$ because it is as if we are computing the gradient of the linear operator with respect to its transformation matrix. Therefore, in line \texttt{7} if $\lambda = 1$ we compute the new transformation matrix as $\sum_{i=1}^{n} u_i \sigma_i v_i^\top - u_1 \sigma_1 v_1^\top + u_1 c v_1^\top = u_1 c v_1^\top + \sum_{i=2}^{n} u_i \sigma_i v_i^\top$. Notice that $W^\prime_\delta$ is in the same format as the parameters of the linear operator (e.g., convolutional filter in convolutional layers). 

If we set $\lambda=1$, the largest singular value will be clipped to the desired value $c$, without requiring the \texttt{for} loop. That is indeed the case for dense layers. The reason that we let $\lambda$ be a parameter and use the \texttt{for} loop is that for convolutional layers, we noticed that a slightly lower value of $\lambda$ is required for the algorithm to work stably. The \texttt{for} loop allows this convergence to singular value $c$.




To derive our fast and precise clipping method, \textit{FastClip}, we intertwine~\Cref{alg:powerqr} and~\Cref{alg:clip} with the outer SGD iterations used for training the model, as shown in~\Cref{alg:fastclip}.
As we mentioned in~\Cref{sec-extraction}, the PowerQR method with warm start is able to track the largest singular values and corresponding vectors by running as few as one iteration per SGD step (lines $4$ and $8$ in~\Cref{alg:fastclip}). Whenever the clipping method is called, we use this value and its corresponding vector as additional inputs to~\Cref{alg:clip} (rather than line $3$ in~\Cref{alg:clip}). By performing this clipping every few iterations of SGD, since the number of calls to this method becomes large and the changes to the weight matrix are slow, we do not need to run its \texttt{while} loop many times. Our experiments showed performing the clipping method every $100$ steps and using only $1$ iteration of \texttt{while} and \texttt{for} loops is enough for clipping the trained models (lines $9$ and $10$). Because after the clipping, the corresponding singular value of $v$ has shrunk, we need to perform a few iterations of PowerQR on a new randomly chosen vector (since $v$ is orthogonal to the other right singular vectors) to find the new largest singular value and corresponding right singular vector, and line $11$ of~\Cref{alg:clip} takes care of this task. Obviously, the constants used in our algorithms are hyperparameters, but we did not tune them to find the best ones and the ones shown in this algorithm are what we used in all of our experiments. Also, any optimizer can be used for the SGD updates in line $7$ (e.g., Adam~\cite{kingma2014adam}).

\begin{algorithm}[tb]
\caption{FastClip ($f_W, X, c, N, \eta$)}\label{alg:fastclip}
\begin{algorithmic}[1]
\STATE {\bfseries Input:} Affine function $f_W$, Dataset $X$, Clip value $c$, Number of SGD iterations $N$, Learning rate $\eta$
\STATE {\bfseries Output:}  Trained affine function $f_{W^\prime}$ with 
singular values clipped to $c$
\STATE $v \gets \mathrm{Random\,input\,vector} $
\STATE $\sigma_1, v_1 \gets \mathrm{PowerQR} (f_W, v_1, 10)$
    \FOR{$i=1$ {\bfseries to} $N$}
        \STATE $X_b \gets \mathrm{SampleFrom} (X) $ \COMMENT{Sample a batch}
        \STATE $W = W - \eta \nabla_W \ell(f_W(X_b))$ \COMMENT{SGD step}
        \STATE $\sigma_1, v_1 \gets \mathrm{PowerQR} (f_W, v_1, 1)$
        \IF{$ i \,\,\mathrm{isDivisibleBy} \,\,100$}
            \STATE $f_W, \sigma_1, v_1 \gets \mathrm{Clip} (f_W, \sigma_1, v_1, c, 1, 10)$
        \ENDIF
    \ENDFOR

\STATE {\bfseries Return} $f_W$

\end{algorithmic}
\end{algorithm}

\subsection{Batch Normalization Layers}
\label{sec-bn}

Batch Normalization~\cite{ioffe2015batch} has proved to successfully stabilize and accelerate the training of deep neural networks and is thus by now standard in many architectures that contain convolutional layers. However, the adverse effect of these layers on the adversarial robustness of models has been noted in previous research~\cite{xie2019intriguing,benz2021revisiting,galloway2019batch}. As we will show in our experiments, not controlling the spectral norm of the batch normalization layers might forfeit the benefits of merely controlling the spectral norm of convolutional layers. We also point out an interesting compensation behavior that the batch normalization layer exhibits when the spectral norm of the convolutional layer is clipped; As the clipping value gets smaller, the spectral norm of the batch normalization layers increases (see Figure~\ref{fig:resnet18-spectral-norm}a).

The clipping method introduced by~\cite{gouk2021regularisation} (explained in~\Cref{apx-alg-bn}), which is also used by the follow-up works~\cite{senderovich2022towards,delattre2023efficient}, performs the clipping of the batch norm layer separately from the preceding convolutional layer, and therefore upper-bounds the Lipschitz constant of their concatenation by the multiplication of their individual spectral norms. This upper-bound, although correct, is not tight, and the actual Lipschitz constant of the concatenation might be much smaller, which might lead to unwanted constraining of the concatenation such that it hurts the optimization of the model. Also, the purpose of the batch normalization layer is to control the behavior of its preceding convolutional layer, and therefore, clipping it separately does not seem to be the best option. As explained in~\Cref{sec-clipping}, our clipping method, can be applied directly to the concatenation of the convolutional layer and the batch normalization layer. This way, we can control the spectral norm of the concatenation without tweaking the batch normalization layer separately. We will show in our experiments that this method can be effective in increasing the robustness of the model, without compromising its accuracy.

\subsection{Layer-wise Orthogonalization for Training Robust Ensembles}
\label{sec:lotos}






When clipping the spectral norm of the layers, we are reducing the capacity of the parameters that can be used during the optimization~\cite{neyshabur2017exploring,bartlett2017spectrally}; it is more likely that the parameters of the two models become ``more similar'' when optimized over these constrained spaces. Therefore, to control the Lipschitz constant of each model to make them more ``individually'' robust, {\em and} avoid sacrificing the ``ensemble robustness'', we need to utilize other modifications to enforce the diversity of their decision boundaries on the adversarially perturbed samples.

\begin{tcolorbox}
\noindent{\bf Our Intuition:} The method that we introduce to enforce the diversity 
is based on {\em promoting the orthogonality of the sub-spaces of the corresponding layers of the models that correspond to their top singular vectors}. Since the top singular vectors govern the major part of the transformation by each layer, this orthogonality promotes the difference in the outputs of the corresponding layers from different models. 

\end{tcolorbox}
Affine layers transform the input space such that the sub-space spanned by the top singular vectors will have the most amount of change in the output space for a perturbed input. When the adversary is choosing a perturbation to add to the input, a natural choice would be to choose the direction along the top singular vectors: with the same amount of perturbation, the adversary will get the highest amount of change in the output space for a perturbed input along this direction. Based on this analogy, we consider any two corresponding affine layers $f^{(j)},g^{(j)}$ from a pair of the models $\mathcal{F}$ and $\mathcal{G}$ in the ensemble, whose linear transformations are represented by matrices $A$ and $B$, with the singular value decompositions $A = \sum_{i=1}^d \sigma_i u_i v_i^T$ and $B = \sum_{i=1}^d \sigma_i^\prime u_i^\prime v_i^{\prime T}$, respectively. We define a notion of similarity based on the top-$k$ sub-spaces:
\begin{align}
    \label{equ:sk}
    S_k^{(j)}(f^{(j)},g^{(j)}, \texttt{mal}) := \sum_{i=1}^k w_i ( \mathrm{ReLU} (\|f^{(j)}(v_i^\prime)\|_2 - \texttt{mal}) + \mathrm{ReLU} (\|g^{(j)}(v_i)\|_2 - \texttt{mal})),
\end{align}
\noindent where (a) $w_i$'s are arbitrary weights which are non-increasing with $i$ to emphasize ``more importance'' for the singular vectors corresponding to top singular values, and (b) \texttt{mal} refers to the \textit{maximum allowed length} of the output of each layer when it is given the singular vectors of the other layer as the input. Observe that when \texttt{mal} is set to $0$, the value of $S_k(f,g)$ is $0$ if the transformations of $f$ and $g$ are orthogonal in their top-$k$ sub-space (i.e., $\|f^{(j)}(v_i^\prime)\|_2 = \|A v_i^\prime\|_2 = \|\sum_{l=1}^d \sigma_{l} u_{l} v_{l}^T v_i^\prime\|_2 = 0$).

Utilizing this insight, we introduce our technique, \textbf{Layer-wise Orthogonalization for Training Robust Ensembles (\texttt{LOTOS})}. \texttt{LOTOS} promotes the orthogonality among these sub-spaces which leads to different behaviors when perturbing the clean samples along a specific direction. We add this similarity for each pair of corresponding affine layers (dense and convolutional layers) in each pair of models within the ensemble and add them to the cross-entropy loss. More specifically, given an ensemble of $N$ models $\cF_i, \,\, i=\{1,\dots,N\}$ with $M$ layers that would be incorporated in the orthogonalization process, the new loss becomes:
\begin{align}
\mathcal{L}_{\text{train}} = \frac{1}{N}\sum_{i=1}^{N}\mathcal{L}_{\text{CE}}(\mathcal{F}_i(x), y) + \frac{\lambda}{M\,N\,(N-1)}\sum_{z=1}^{N-1}\sum_{j=z+1}^{N} \sum_{l=1}^{M} S_k^{(l)}(f_z^{(l)},f_j^{(l)}, \texttt{mal})
\label{equ:lotos-loss}
\end{align}

where $\mathcal{L}_{\text{CE}}(\mathcal{F}_i(x), y)$ is the cross-entropy loss of $\mathcal{F}_i(x)$ given its output on $x$ and the ground-truth label $y$. $\lambda$ controls the effect of the orthogonalization loss and could be adjusted. 

\section{Theoretical Results}
\label{sec:efficiency_lotos}

In this section we first present a theorem that shows the specific structure of the singular values in convolutional layers. Building on this result, we show why LOTOS is highly efficient for convolutional layers.

\subsection{Limitations of Convolutional Layers}
\label{sec-limitations}

In this section, we shed light on the formerly overlooked limitation of convolutional layers to represent any arbitrary spectrum. In some prior work, the proposed method for clipping computes the whole spectrum, clips the spectral norm, and then tries to form a new convolutional layer with the new spectrum~\cite{sedghi2018singular,senderovich2022towards}. We also introduce a new simple optimization method that uses our PowerQR algorithm and adheres to this procedure in~\Cref{sec-modification}. In the following, we start with a simple example that shows an issue with this procedure, and then, present our theoretical results that show a more general and fundamental limitation for a family of convolutional layers.

Consider a 2d convolutional layer whose kernel is $1\times 1$ with a value of $c$. Applying this kernel to any input of size $n\times n$ scales the values of the input by $c$. The equivalent matrix form of this linear transformation is an $n\times n$ identity matrix scaled by $c$. We know that this matrix has a rank of $n$, and all the singular values are equal to $1$. Therefore, if the new spectrum, $S^\prime$, does not represent a full-rank transformation, or if its singular values are not all equal, we cannot find a convolutional layer with a $1\times 1$ kernel with $S^\prime$ as its spectrum.


In the following theorem, we compute the closed form of the singular values of convolutional layers with circular padding, and in Remark~\ref{remark-duplicate}, we mention one of the general limitations that it entails.



\begin{theorem}
\label{theorem-limitations}
For a convolutional layer with $1$ input channel and $m$ output channels (the same result holds for convolutions with $1$ output channel and $m$ input channels) and circular padding applied to an input which in its vectorized form has a length of $n$, if the vectorized form of the $l$-th channel of the filter is given by $\textbf{f}^\textbf{(l)} = [f_0^{(l)}, f_1^{(l)}, \dots, f_{k-1}^{(l)}]$, the singular values are:
\vspace{-3 pt}
\begin{align}
    \mathcal{S}(\omega) = \left\{ \sqrt{\sum_{l=1}^m \mathcal{S}_j^{(l)^2}(\omega)}, \, j\in[n-1] \right\},
    \label{equ-eigval_2}
\end{align}
\vspace{-3pt}
\noindent where for $l \in {1,\dots, m}$:
\vspace{-3 pt}
\begin{align*}
    \mathcal{S}^{(l)}(\omega) = \left[\sqrt{c_0^{(l)} + 2\sum_i^{k-1} c_i^{(l)} \Re(\omega^{j\times i})},\, j\in [n-1] \right]^\top
    \label{equ-eigval}
\end{align*}
\vspace{-3pt}
\noindent in which $c_i^{(l)}$'s are defined as:
\begin{align*}
    c_0^{(l)} &:= f_0^{(l)^2} + f_1^{(l)^2} + \dots + f_{k-1}^{(l)^2},\\
    c_1^{(l)} &:= f_0^{(l)} f_1^{(l)} + f_1^{(l)} f_2^{(l)} + \dots + f_{k-2}^{(l)} f_{k-1}^{(l)},\\
    &\vdots\\
    c_{k-1}^{(l)} &:= f_0^{(l)} f_{k-1}^{(l)}.
\end{align*}

\label{pro-eigval_2}
\end{theorem}

\begin{remark}
\label{remark-duplicate}
 Note that in the closed form, we only have the real parts of the roots of unity. So, for any non-real root of unity, we get a duplicate singular value because mirroring that root around the real axis (i.e., flipping the sign of the imaginary part) gives another root of unity with the same real component. Only the roots that are real might derive singular values that do not have duplicates. The only such roots with real parts are $1$ and $-1$ ($-1$ is a root only for even $n$). Therefore, except for at most two singular values, other ones always have duplicates. This shows a more general limitation in the representation power of convolutional layers in representing arbitrary spectrums.
\end{remark}

\subsection{Computational Efficiency of LOTOS}

In this section we use the results of Theorem~\ref{theorem-limitations} to prove that LOTOS is highly efficient for convolutional layers.

\noindent{\bf Time Complexity:} Note that the number of summands in the orthogonalization loss is $O(N^2 M)$. The computation of $S_{k}^{(t)}$ uses the computed singular vectors of each layer by FastClip~\cite{boroojeny2024spectrum}, which is fast and accurate in practice, and feeds them to the corresponding layer of the other models (see~\Cref{equ:sk}): therefore, it is as if each model has an extra batch of size $N-1$ to process at each iteration, which is relatively small when $N$ is small. our experiments (\S~\ref{sec:amun_results}) show that performing the orthogonalization for only the first layer would be effective for training robust ensembles (i.e., $M=1$). Therefore, the increase in the training time becomes negligible compared to when the clipping model is used without \texttt{LOTOS}.

\noindent{\bf Highly Efficient for Convolutional Layers:} For the orthogonalization to be effective in~\Cref{equ:lotos-loss}, it is necessary to increase the value of $k$ (dimension of orthogonal sub-spaces) because the layers of the DNNs are transformations between high dimensional representations. Therefore, only orthogonalizing the sub-space corresponding to the few top singular values does not guarantee that there is no strong correlation among the remaining top singular vectors (which might correspond to high singular values in each of the models). However, increasing $k$ decreases the computational efficiency, in both compution of the singular vectors~\cite{boroojeny2024spectrum} and computation of the orthogonalization loss in~\Cref{equ:lotos-loss}. 

Fortunately, specific properties of the convolutional layers, which are the most common affine layers in DNNs, allow an effective orthogonalization even with very small values of $k$. In~\Cref{theorem-ortho}, we prove even $k=1$ can be effective in orthogonalization with respect to the remaining singular vectors for convolutional layers.



\begin{theorem}
\label{theorem-ortho}
    Given two convolutional layers, $M_1$ and $M_2$ with a single input and output channel and circular padding for which $\textbf{f}$ is the vectorized form of the filter with a length of $T$, and considering $n$ to be the length of the vectorized input, if $\|A v_1^\prime\|_2 \leq \epsilon $, then: 
\begin{align}
    \|A v_p^\prime\|_2 \leq \sqrt{\epsilon^2 + \pi \|\textbf{f}\|_2^2 \, T^2 \frac{p}{n}},
\end{align}

\noindent where $A$ is the corresponding linear transformation of $M_1$ and $v_p^\prime$ is singular vector of $M_2$ corresponding to its $p$-th largest singular value. 
    
\end{theorem}

The proof can be found in our paper. As~\Cref{theorem-ortho} shows, by orthogonalization of the linear transformation of the convolutional layer $M_1$ (i.e., $A$) and only the first singular vector of $M_2$ (i.e., $v_1^\prime$), so that $\|A v_1^\prime\|_2 \leq \epsilon$, the size of the output of $M_1$ when applied to the remaining singular vectors of $M_2$ (i.e., $\|A v_p^\prime\|_2$) will be upper bounded. This upper-bound depends on the ratio of the ranking of the corresponding singular value to the input size (i.e., $\frac{p}{n}$), which gets smaller for the top singular vectors that have a higher contribution to the transformations. It also depends on the size of the kernel ($T$) which is usually small in models used in practice (e.g., $3^2$ in 2D convolutional layers). Finally, it also depends on the $\ell_2$ norm of the filter values, which can be controlled simply by using weight decay when optimizing the parameters during training. We verify this efficiency of \texttt{LOTOS} when applied to convolutional layers in our experiments.

\section{Empirical Results}
\label{sec:lotos_results}


After showcasing the effectiveness of our clipping method, we first report the results on comparing it to existing methods. Then, we wish to answer the following questions: (1) Does decreasing the Lipschitz constant of the models of an ensemble increase the $T_{rate}$ between them?; (2) does \texttt{LOTOS} decrease the $T_{rate}$ among the models of an ensemble, and does this decrease in the $T_{rate}$ among the models of the ensemble lead to a lower success rate in black-box attacks from other source models?; (3) what are the effects of varying the ensemble size and the number of orthogonalized singular vectors ($k$) on the performance of \texttt{LOTOS}?; (4) is \texttt{LOTOS} still effective when the models of the ensemble are different?; (5) can we combine \texttt{LOTOS} with the prior work on training robust ensembles to provide additional enhancements to robustness?; (6) can \texttt{LOTOS} be combined with common methods used for increasing the robustness of the models, such as adversarial training?; and (7) is \texttt{LOTOS} effective for non-adversarial noise?

As a quick summary, our results show that: (1) decreasing the Lipschitz constant of the models of an ensemble, although make them \textit{individually} more robust, increases the $T_{rate}$ among them (\S~\ref{sec:tradeoff}); (2) \texttt{LOTOS} is indeed effective at reducing the $T_{rate}$ between the models of an ensemble which leads to more robust accuracy against black-box attacks (\S~\ref{exp:layer-wise_ortho}); (3) when using \texttt{LOTOS}, increasing the ensemble size leads to much higher improvement in the robust accuracy, and changing the number singular values has negligible impact on the transferability; (4) \texttt{LOTOS} is effective even when the ensemble is heterogeneous; (5) \texttt{LOTOS} in conjunction with \texttt{TRS}~\cite{yang2021trs} or \texttt{DVERGE}~\cite{yang2020dverge}, two of the SOTA methods in training robust ensembles, yields better performance than either in isolation; (6) \texttt{LOTOS} can be used together with adversarial training to boost the robustness of the ensemble; and (7) by effective diversification of the clipped models, \texttt{LOTOS} enhances robustness against non-adversarial noise. The details on the experiments for questions 3 to 7 can be found in our paper.



\noindent{\bf Attacks:} We use both black-box attacks and white-box attacks in our experiments. The \textbf{white-box} attack is used to evaluate the $T_{rate}$ of adversarial examples between the models in the ensemble; for each ordered pair of the models in the ensemble, the former is used as the source model to generate the adversarial examples and then the $T_{rate}$ of the generated adversarial examples is evaluated on the latter (target model) using Definition~\ref{def:trans}. The average of this value for all the ordered pairs of the models is considered the $T_{rate}$ of the ensemble. A low $T_{rate}$ between the models of the ensemble does not necessarily imply a more robust ensemble. So to evaluate the robustness of ensembles against adversarial attacks, we also use black-box attacks. In the \textbf{black-box} attacks, an independently trained source (surrogate) model (of the same type as the models in the ensemble) is used to generate the adversarial examples; we then measure the robust accuracy of the ensembles against these adversarial examples. For further details on the setup of experiments, please refer to our paper. 

\subsection{Effective Clipping of Spectral Norm}
\label{sec-experiment-clipping}

As pointed out previously, the methods introduced by~\cite{senderovich2022towards} and~\cite{delattre2023efficient}, compute the exact spectral norm only when circular padding is used for the convolutional layers. For the other types of common paddings for these layers, these methods can result in large errors. To show this, we computed the average absolute value of the difference in their computed values and the correct value for a $100$ randomly generated convolutional filters for different types of padding (see~\Cref{apx-exp-powerqr}). As the results show, the error dramatically increases as the number of channels increases. Also, the method introduced by~\cite{delattre2023efficient} assumes a stride of $1$ and therefore leads to even larger errors when a stride of $2$ is used. 
We also show that PowerQR is much more efficient than $k$ successive runs of the power method followed by deflation (as suggested in~\cite{virmaux2018lipschitz}) for extracting the top-$k$ singular values in~\Cref{apx-multi}. We also have utilized the capability of our method to extract the spectral norm of the concatenation of multiple implicitly linear layers for analyzing the spectral norm of the concatenation of convolutional and batch norm layers in~\Cref{fig:resnet18-spectral-norm} (see~\Cref{fig:resnet18-catclip-study}a for more examples).

\begin{figure*}[t]
\centering
  \includegraphics[width=.65\linewidth]{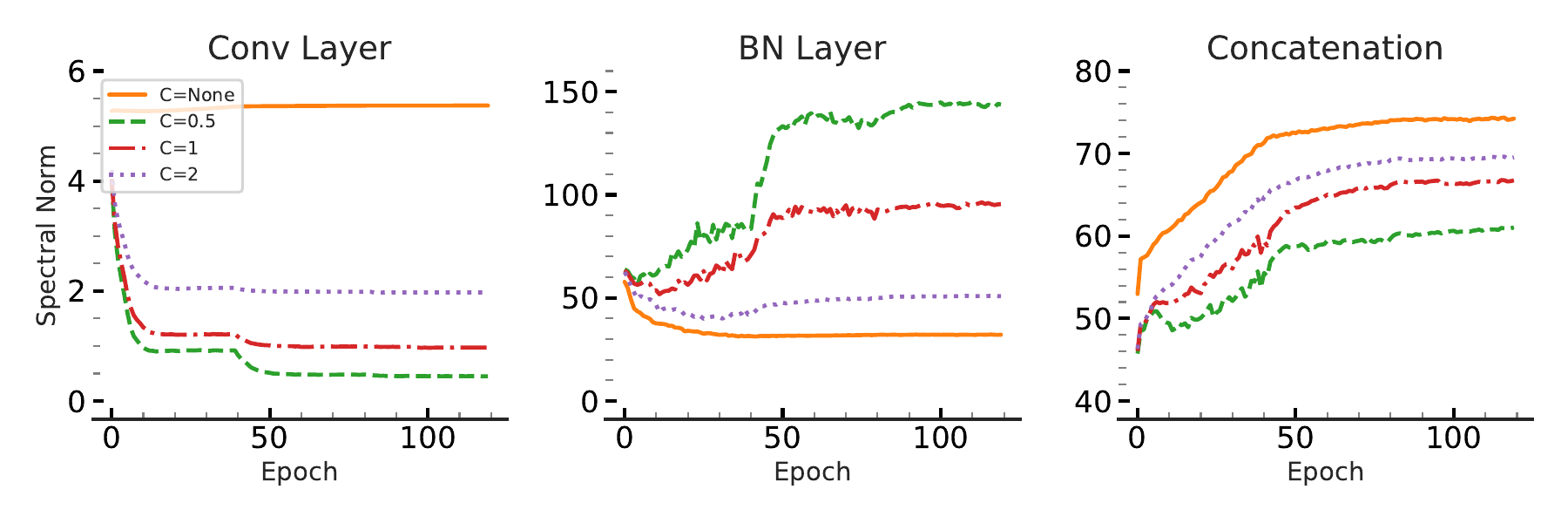}
  \hfill
  \includegraphics[width=.28\linewidth]{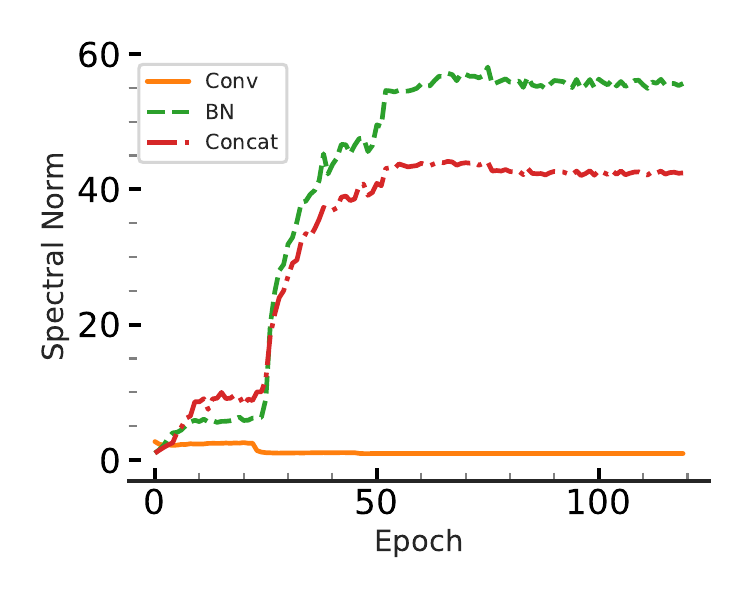}
\caption{\textbf{(a)} The first three plots show the clipping of the convolutional layer in a simple two-layer network to various values on MNIST. As the clipping target gets smaller, the spectral norm of the batch norm layer compensates and becomes larger. Meanwhile, the spectral norm of their concatenation slightly decreases. \textbf{(b)} The right-most plot shows the spectral norm of a convolutional layer, its succeeding batch norm layer, and their concatenation from the clipped ResNet-18 model trained on CIFAR-10. Although the convolutional layer is clipped to $1$, the spectral norm of the concatenation is much larger due to the presence of the batch norm layer.}
\label{fig:resnet18-spectral-norm}
\end{figure*}


For comparing the correctness of the clipping models for different types of convolutional layers. Then, we show the effectiveness of each of these methods on the generalization and robustness of ResNet-18~\cite{he2016deep}, which is the same model used in the experiments of prior work~\cite{gouk2021regularisation,senderovich2022towards,delattre2023efficient} and Deep Layer Aggregation (DLA)~\cite{yu2018deep} model, which is more complex with more layers and parameters. We train the models on the same datasets used in prior works, MNIST~\cite{lecun1998mnist} and CIFAR-10~\cite{krizhevsky2009learning}.




\begin{figure}[t]
\centering
\includegraphics[width=1.\linewidth]{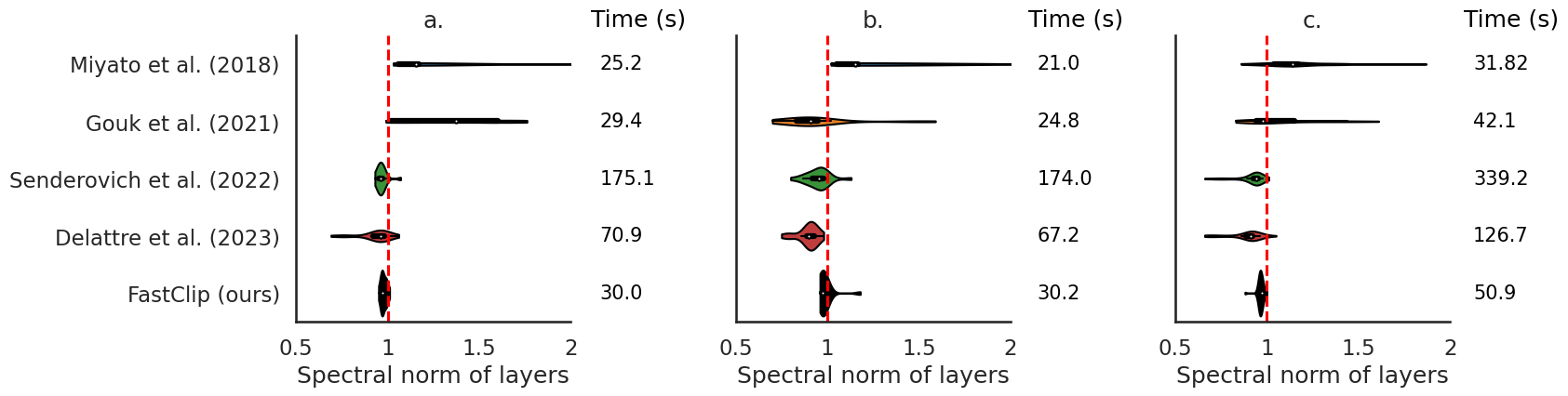}
\caption{The layer-wise spectral norm of a ResNet-18 model trained on CIFAR-10 (\textbf{a}) and MNIST (\textbf{b}) using each of the clipping methods. The time columns shows the training time per epoch for these methods. \textbf{c.} The layer-wise spectral norm of a DLA model trained on CIFAR-10 using each of the clipping methods. The time column shows the training time per epoch for these methods. As all of the plots show, by using our method, all the layers have a spectral norm very close to the \textbf{target value $\pmb{1}$}. Our method is also much faster than the relatively accurate alternatives and shows a slower increase in running time as the model gets larger.} 
\label{fig:resnet-clip-compare}
\end{figure}

\begin{table*}[t!]
\caption{The accuracy on the test set and adversarial examples generated with PGD-$(50)$ and CW for ResNet-18 trained on MNIST and CIFAR-10, for the models with their convolutional and dense layers clipped to $1$ (\texttt{With BN}) and clipped models with their batch norm layers removed (\texttt{No BN}). The accuracy of the original model on the test set, PGD-generated examples, and CW-generated examples for MNIST are $\pmb{99.37 \pm 0.02}$, $\pmb{15.30 \pm 9.11}$, and $\pmb{77.46 \pm 5.89}$, respectively. For CIFAR-10, these values are $\pmb{94.77 \pm 0.19}$, $\pmb{22.15 \pm 0.54}$, and $\pmb{13.85 \pm 0.74}$.}
\label{tab:bn-cifar}
\vskip 0.15in
\begin{center}
\begin{small}
\begin{sc}
\begin{tabular}{@{} l  c c  c  c  c c  c @{}}
 \toprule
 
 & \multicolumn{2}{@{}c}{\textbf{MNIST}} & \multicolumn{2}{@{}c}{\textbf{CIFAR-10}} \\\addlinespace[0.3em]

  \textbf{Method} &  With BN & No BN
 & With BN & No BN  \\\addlinespace[0.3em]

    \cmidrule(r){2-3}
    \cmidrule(r){4-5}

 & \multicolumn{4}{@{}c}{Accuracy on the test set}   \\ \addlinespace[0.4em]
 
 \cite{miyato2018spectral} &  $99.40 \pm 0.06$ & $98.00 \pm 0.22$ 
 & $94.82 \pm 0.11$ & $88.83 \pm 1.41$  \\\addlinespace[0.3em]

 \cite{gouk2021regularisation} & $99.25 \pm 0.04$ &  $21.94 \pm 6.01$ 
 & $89.98 \pm 0.38$ & $19.80 \pm 5.55$ \\\addlinespace[0.3em]

 \cite{senderovich2022towards} &  $99.40 \pm 0.03$ & $62.63 \pm 24.01$ 
 & $94.19 \pm 0.13$  &  $68.29 \pm 10.63$ \\\addlinespace[0.3em]

\cite{delattre2023efficient} &  $99.29 \pm 0.05$ & $97.27 \pm 0.03$ 
& $93.17 \pm 0.13$ &  $39.35 \pm 9.84$   \\\addlinespace[0.3em]


FastClip~(\Cref{alg:fastclip}) & $\pmb{99.41 \pm 0.04}$ & $\pmb{99.31 \pm 0.02}$      
& $\pmb{95.28 \pm 0.07}$ & $\pmb{92.08 \pm 0.28}$ \\\addlinespace[0.4em]


    & \multicolumn{4}{@{}c}{Accuracy on samples from PGD attack} 
        & \\\addlinespace[0.4em]
 
  \cite{miyato2018spectral} &  $21.77 \pm 12.98$ &  $32.67 \pm 14.08$  
  & $23.48 \pm 0.11$  & $35.18 \pm 7.72$  \\\addlinespace[0.3em]

  \cite{gouk2021regularisation} & $2.40 \pm 2.94$  & $8.41 \pm 3.03$  
  &  $16.13 \pm 1.28$  & $14.66 \pm 3.99$ \\\addlinespace[0.3em]

   \cite{senderovich2022towards} & $30.99 \pm 9.28$ & $15.97 \pm 4.84$ 
   &  $21.74 \pm 0.72$ & $39.84 \pm 7.87$  \\\addlinespace[0.3em]

\cite{delattre2023efficient} &  $30.87 \pm 4.77$  &  $71.75 \pm 1.49$ 
&  $21.08 \pm 0.84$ & $16.22 \pm 3.17$  \\\addlinespace[0.3em]


   FastClip~(\Cref{alg:fastclip}) &  $\pmb{47.90 \pm 5.49}$ & $\pmb{78.50 \pm 2.85}$  
    & $\pmb{24.48 \pm 0.32}$ & $\pmb{41.37 \pm 0.95}$ \\\addlinespace[0.4em]

     & \multicolumn{4}{@{}c}{Accuracy on samples from CW attack} 
         & \\\addlinespace[0.4em]
  
   \cite{miyato2018spectral} &  $86.25 \pm 2.18$ &  $73.56 \pm 10.38$  
   & $16.68 \pm 0.95$  & $48.48 \pm 6.40$  \\\addlinespace[0.3em]
 
   \cite{gouk2021regularisation} & $66.59 \pm 21.91$  & $21.94 \pm 6.01$  
   &  $18.79 \pm 2.99$  & $12.63 \pm 4.33$ \\\addlinespace[0.3em]
 
    \cite{senderovich2022towards} & $87.72 \pm 2.75$ & $58.71 \pm 20.67$ 
    &  $20.53 \pm 0.77$ & $43.82 \pm 9.57$  \\\addlinespace[0.3em]
 
 \cite{delattre2023efficient} &  $83.97 \pm 1.79$  &  $\pmb{96.93 \pm 0.06}$ 
 &  $24.05 \pm 1.72$ & $11.92 \pm 5.34$  \\\addlinespace[0.3em]
 
 
    FastClip~(\Cref{alg:fastclip}) &  $\pmb{90.21 \pm 1.80}$ & $95.35 \pm 1.06$  
     & $\pmb{24.31 \pm 0.96}$ & $\pmb{56.28 \pm 0.96}$ \\\addlinespace[0.4em]

\bottomrule
\end{tabular}
\end{sc}
\end{small}
\end{center}
\vskip -0.1in
\end{table*}

\subsection{Precision and Efficiency}
\label{subsec-comp-others}

We use a simple model with one convolutional layer and one dense layer and use each of the clipping methods on the convolutional layer while the model is being trained on MNIST and the target clipping value is $1$. We compute the true spectral norm after each epoch.~\Cref{fig:simple-clip-compare} shows the results of this experiment for $4$ convolutional layers with different settings (e.g., kernel size and padding type). This figure shows our method is the only one that correctly clips various convolutional layers.


\begin{figure*}[t]
\centering
\includegraphics[width=1.\linewidth]{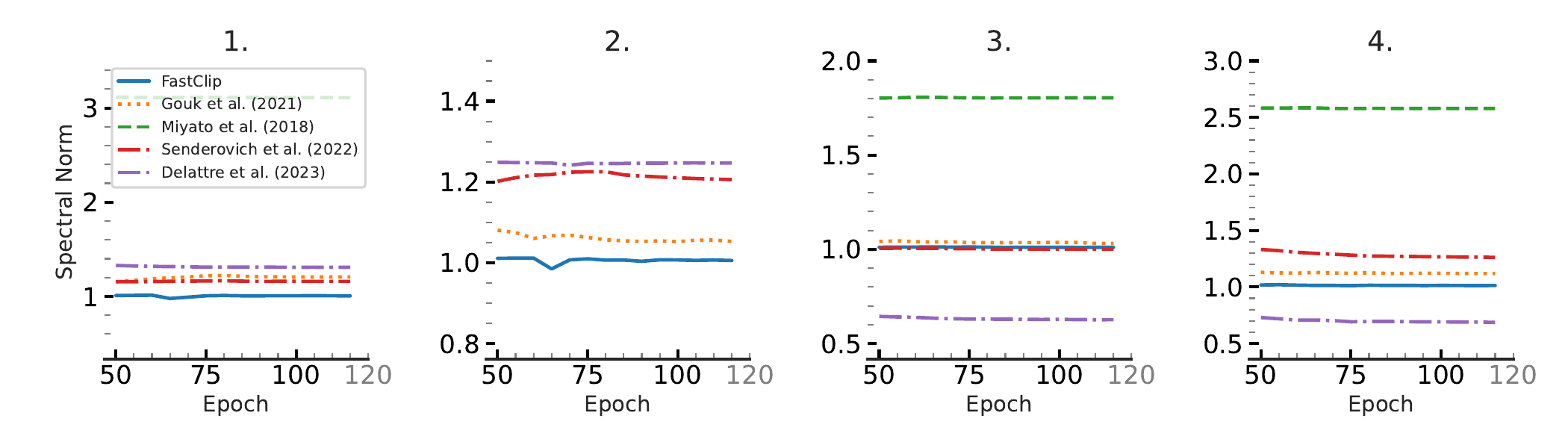}
\caption{Comparison of the clipping methods in a simple network with only one convolutional layer and one dense layer, where \textbf{the target value is $\pmb{1}$}. Our method is the only one that clips this layer correctly for all different settings: 1. Kernel of size $3$ with reflect padding, 2. Kernel of size $3$ with same padding, 3. Kernel of size $3$ and zeros padding with stride of $2$, and 4. Kernel of size $5$ with same replicate padding and stride of $2$.}
\label{fig:simple-clip-compare}
\end{figure*}


In~\Cref{fig:resnet-clip-compare}a, we evaluated the distribution of the layer-wise spectral norms of the ResNet-18 models trained on CIFAR-10 using each of the clipping methods (for which the performance can also be found in~\Cref{tab:bn-cifar}) and showed our method is the most precise one, while being much more efficient than the relatively accurate alternatives~\cite{senderovich2022towards,delattre2023efficient}. To further evaluate the precision of our clipping method in comparison to the other methods, we also did the same experiment with the ResNet-18 model trained on the MNIST dataset. As~\Cref{fig:resnet-clip-compare}b shows the precision results are similar to the plot we had for CIFAR-10, with the times per epoch slightly smaller for all the methods due to the smaller input size. To check if the same precision results hold for other models and how the running time changes for larger models, we performed the same experiment for the DLA models that are trained using each of the clipping methods (for which the performance on the test set and adversarial examples can be found in~\Cref{tab:dla-results}). As~\Cref{fig:resnet-clip-compare}c shows, our method is still the most precise one among the clipping methods, while still being much faster than the more accurate alternatives. Another interesting point is that the factor by which the running times have increased for the larger model is smaller for our method compared to the methods introduced by~\cite{senderovich2022towards} and~\cite{delattre2023efficient}.

\subsection{Generalization and Robustness}
\label{sec-generalization}

Since the Lipschitz constant of a network may be upper-bounded by multiplying together the spectral norms of its constituent dense and convolutional layers, it is intuitive that regularizing per-layer spectral norms improves model generalization~\cite{bartlett2017spectrally}, and also adversarial robustness~\cite{szegedy2013intriguing}.
Therefore, we tested all the clipping models by using them during the training and computing the accuracy of the corresponding models on the test set and adversarial examples generated by two common adversarial attacks, Projected Gradient Descent (PGD)~\cite{madry2017towards} and Carlini \& Wanger Attack (CW)~\cite{carlini2017towards}.  

For the ResNet-18 model, we used the same variant used in prior work~\cite{gouk2021regularisation,senderovich2022towards,delattre2023efficient}. This variant divides the sum of the residual connection and convolutional output by $2$ before passing that as input to the next layer. This will make the whole layer $1$-Lipschitz if the convolutional layer is $1$-Lipschitz (the residual connection is $1$-Lipschitz). As~\Cref{tab:bn-cifar} shows, our clipping method leads to the best improvement in test accuracy while making the models more robust to adversarial attacks. The reason for the lack of the expected boost in the robustness of the models when clipping their spectral norms is shown in Figure~\ref{fig:resnet18-spectral-norm}b. This figure shows that although the models are clipped, the concatenation of some convolutional layers with batch normalization layers forms linear operators with large spectral norms. As~\Cref{fig:resnet18-spectral-norm}a suggests, clipping the convolutional layer to smaller values will further increase the spectral norm of the batch norm layer. Still, as this figure suggests, there might be an overall decrease in the spectral norm of their concatenation which causes the slight improvement in the robustness of the clipped models. Therefore, we also trained a version of the model with all the batch norm layers removed. As~\Cref{tab:bn-cifar} shows, this leads to a huge improvement in the robustness of the clipped models; however, the clipped models achieve worse accuracy on the test set.

\begin{figure*}[t]
\centering
  \includegraphics[width=.67\linewidth]{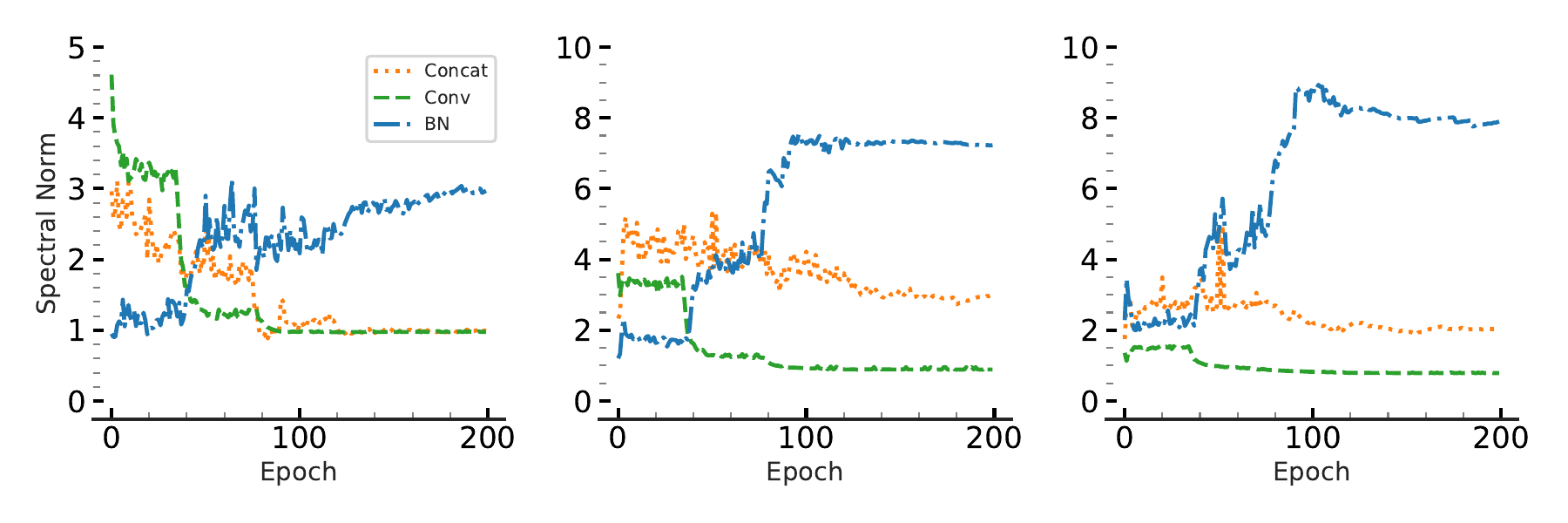}
\hfill
  \includegraphics[width=.28\linewidth]{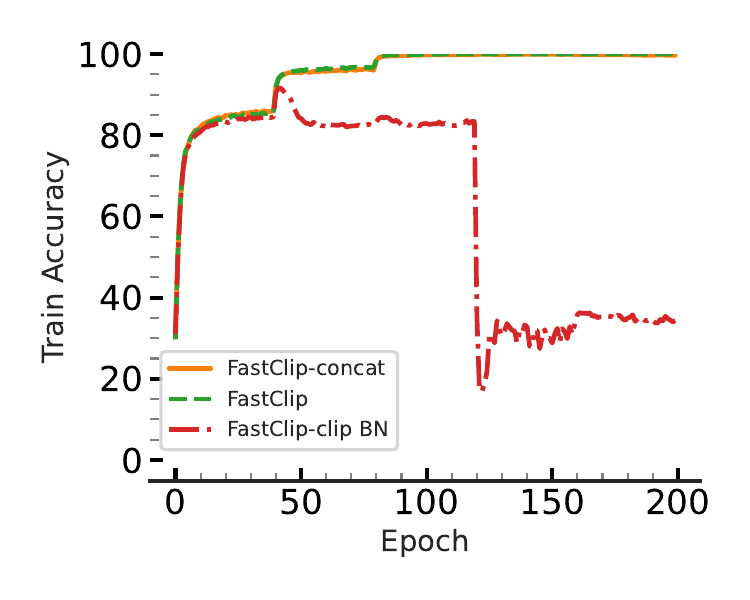}
\caption{\textbf{(a)} Each of these three subplots shows the spectral norms of a convolutional layer, its succeeding batch norm layer, and their concatenation in a ResNet-18 model trained on CIFAR-10. The convolutional layers in this model are clipped to $1$. Instead of clipping the batch normalization layer, our method has been applied to the concatenation to control its spectral norm. \textbf{(b)} The rightmost subplot shows the training accuracy for the ResNet-18 model that is trained on CIFAR-10. One curve belongs to the model with the convolutional layers clipped to $1$ using FastClip and the batch norm layers clipped using the direct method used by prior works (FastClip-clip BN). The other two belong to FastClip and FastClip-concat.}
\label{fig:resnet18-catclip-study}
\end{figure*}


We also performed the same experiment with the DLA model on both MNIST and CIFAR-10. These models are more complex than the ResNet-18 models and have more convolutional layers and parameters. An interesting observation with the DLA model was that they could not be trained at all without batch normalization layers with the default settings, even if the spectral norm is clipped using any of the clipping methods. Therefore, unlike the results for ResNet-18, we do not show the numbers for the \texttt{No BN} case (the test accuracies were almost the same as the random classifier in all cases). As~\Cref{tab:dla-results} shows, our clipping method makes the most improvement in generalization on both datasets and in terms of adversarial robustness, it is either better than the other methods or achieves competitive results. Note that, unlike ResNet-18 models in which according to the experiments by~\cite{gouk2021regularisation} we divided the sum of the two layers in the residual modules to make the whole output $1$-Lipschits at each layer, we used DLA models as they are designed without further modification. Since the DLA model contains modules in which the output of two layers are summed up to generate the inputs of the succeeding layer, making each one of the layers $1$-Lipshitz does not make each module $1$-Lipschitz. This might affect the robustness of models; however, we wanted to use this model as it is to show the benefits of our method applied to an original model without any modification.

\begin{table}[t]
\caption{The accuracy on the test set and adversarial examples generated with PGD-$(50)$ and CW for DLA trained on MNIST and CIFAR-10, for the models with their convolutional and dense layers clipped to $1$. The accuracy of the original model on the test set, PGD-generated examples, and CW-generated examples for MNIST are $\pmb{99.39 \pm 0.08}$, $\pmb{0.35 \pm 0.48}$, and $\pmb{50.78 \pm 6.58}$, respectively. For CIFAR-10, these values are $\pmb{94.82 \pm 0.12}$, $\pmb{21.67 \pm 1.42}$, and $\pmb{11.19 \pm 1.35}$.}
\label{tab:dla-results}
\vskip 0.15in
\begin{center}
\begin{small}
\begin{sc}
\begin{tabular}{@{} l  c c  c  c  c c  @{}}
 \toprule
 

\textbf{Method} &  \textbf{Test Set} & \textbf{PGD-(50)} & \textbf{CW}  \\\addlinespace[0.4em]


 & \multicolumn{3}{@{}c}{CIFAR-10}   \\ \addlinespace[0.2em]
\cmidrule(r){2-4}                        
 
 \cite{miyato2018spectral} &  $95.06 \pm 0.09$ & $\pmb{23.62 \pm 1.27}$ 
 & $\pmb{17.75 \pm 1.34}$  \\\addlinespace[0.3em]

 \cite{gouk2021regularisation} & $92.44 \pm 0.14$ &  $19.39 \pm 1.34$ 
 & $13.41 \pm 1.08$  \\\addlinespace[0.3em]

 \cite{senderovich2022towards} &  $93.39 \pm 2.30$ & $20.37 \pm 2.67$ 
 & $15.39 \pm 0.87$   \\\addlinespace[0.3em]

\cite{delattre2023efficient} &  $93.66 \pm 0.46$ & $19.90 \pm 2.47$ 
& $14.92 \pm 2.77$    \\\addlinespace[0.3em]

FastClip & $\pmb{95.53 \pm 0.10 }$ & $22.54 \pm 1.02$      
& $17.14 \pm 0.76$  \\\addlinespace[0.5em]


 & \multicolumn{3}{@{}c}{MNIST}   \\ \addlinespace[0.2em]

\cmidrule(r){2-4}                        

 \cite{miyato2018spectral} &  $99.40 \pm 0.05$ & $3.26 \pm 2.95$ 
 & $78.59 \pm 5.34$  \\\addlinespace[0.3em]

 \cite{gouk2021regularisation} & $99.26 \pm 0.06$ &  $3.13 \pm 2.49$ 
 & $79.16 \pm 7.07$  \\\addlinespace[0.3em]

 \cite{senderovich2022towards} &  $99.43 \pm 0.03$ & $14.03 \pm 5.27$ 
 & $83.92 \pm 2.33$   \\\addlinespace[0.3em]

\cite{delattre2023efficient} &  $99.34 \pm 0.04$ & $4.28 \pm 2.11$ 
& $68.02 \pm 5.37$    \\\addlinespace[0.3em]

FastClip & $\pmb{99.44 \pm 0.04}$ & $\pmb{16.74 \pm 7.09}$      
& $\pmb{84.90 \pm 1.59}$  \\\addlinespace[0.4em]

\bottomrule
\end{tabular}
\end{sc}
\end{small}
\end{center}
\vskip -0.1in
\end{table}

\subsection{Clipping Batch Norm}
\label{exp-bn-clip}

 As~\Cref{tab:bn-cifar} shows, the presence of batch normalization can be essential for achieving high test accuracy. In fact, as explained in~\Cref{sec-generalization}, DLA models cannot be trained without batch norm layers. So, instead of removing them, we are interested in a method that allows us to control their adverse effect on the robustness of the model. First, we explored the results for the models with their batch norm layers clipped to strictly less than $1$ by utilizing the method that was used by prior work~\cite{gouk2021regularisation,senderovich2022towards}. As the results show, our method still achieves the best results with this technique; however, this method for clipping the batch norm, although leads to bounded Lipschitz constant for the model, does not lead to a significant improvement in the robustness of the models and leads to low test accuracy, which is due to over-constraining the models and hindering their optimization as discussed in Section~\ref{sec-bn}. In fact, as Figure~\ref{fig:resnet18-catclip-study}b shows, using this method for controlling batch norm layers even hinders the training process (we show this version of our method by \texttt{FastClip-clip BN} in our experiments). 

\begin{table}[t]
\caption{The accuracy on the test set and adversarial examples generated with PGD $(50), \epsilon=0.02$ of ResNet-18 trained on CIFAR-10, for the models with their convolutional and dense layers clipped to $1$ (\texttt{With BN}). For all the models, except \texttt{FastClip-concat}, the batch norm layers are clipped to strictly $1$ using the method by~\cite{gouk2021regularisation}. \texttt{FastClip-concat} uses FastClip method for controlling the batch norm of the concatenation of convolutional and batch norm layers, as described in~\Cref{exp-bn-clip}. As the results show, this method does not impede the optimization of the model and leads to a much better test accuracy while making the models more robust compared to when batch norm layers are not taken into account during the clipping process (see~\Cref{tab:bn-cifar} and \Cref{tab:dla-results}).}
\label{tab:bn-clip}
\vskip 0.15in
\begin{center}
\begin{small}
\begin{sc}
\begin{tabular}{@{} l  c c  c  c  c c  @{}}
 \toprule
 

\textbf{Method} &  \textbf{Test Set} & \textbf{PGD-50 $\epsilon=0.02$} & \textbf{CW $c=0.02$}  \\\addlinespace[0.4em]


 & \multicolumn{3}{@{}c}{ResNet-18 Model}   \\ \addlinespace[0.2em]
\cmidrule(r){2-4}                        
 
 \cite{miyato2018spectral} &  $85.91 \pm 0.27$ & $17.63 \pm 0.44$ 
 & $\pmb{49.94 \pm 2.25}$  \\\addlinespace[0.3em]

 \cite{gouk2021regularisation} & $27.33 \pm 2.11$ &  $13.89 \pm 0.68$ 
 & $11.23 \pm 3.53$  \\\addlinespace[0.3em]

 \cite{senderovich2022towards} &  $69.27 \pm 4.35$ & $16.22 \pm 1.86$ 
 & $25.27 \pm 3.43$   \\\addlinespace[0.3em]

\cite{delattre2023efficient} &  $30.44 \pm 3.59$ & $12.99 \pm 4.16$ 
& $10.17 \pm 6.63$    \\\addlinespace[0.3em]

FastClip & $90.59 \pm 0.36 $ & $\pmb{25.97 \pm 0.88}$      
& $41.50 \pm 2.02$  \\\addlinespace[0.6em]

FastClip-concat & $\pmb{94.63 \pm 0.08 }$ & $25.02 \pm 1.56$      
& $33.77 \pm 3.59$  \\\addlinespace[0.5em]


 & \multicolumn{3}{@{}c}{DLA Model}   \\ \addlinespace[0.2em]

\cmidrule(r){2-4}                        

 \cite{miyato2018spectral} &  $80.91 \pm 0.95$ & $16.29 \pm 2.79$ 
 & $\pmb{40.64 \pm 5.34}$  \\\addlinespace[0.3em]

 \cite{gouk2021regularisation} & $29.35 \pm 1.19$ &  $13.07 \pm 3.13$ 
 & $11.95 \pm 3.97$  \\\addlinespace[0.3em]

 \cite{senderovich2022towards} &  $74.94 \pm 1.03$ & $13.92 \pm 1.25$ 
 & $34.14 \pm 3.52$   \\\addlinespace[0.3em]

\cite{delattre2023efficient} &  $32.31 \pm 4.38$ & $12.18 \pm 3.14$ 
& $12.07 \pm 5.45$    \\\addlinespace[0.3em]

FastClip & $89.27 \pm 0.25$ & $23.40 \pm 1.46$      
& $39.16 \pm 2.76$  \\\addlinespace[0.6em]

FastClip-concat & $\pmb{95.02 \pm 0.07 }$ & $\pmb{25.93 \pm 1.31}$      
& $28.16 \pm 0.81$  \\\addlinespace[0.4em]
   
\bottomrule
\end{tabular}
\end{sc}
\end{small}
\end{center}
\vskip -0.1in
\end{table}

 The capability of our clipping method to be applied to the concatenation of implicitly linear layers provides an alternative approach to control the spectral norm of the concatenation of convolutional layers and batch norm layers. This still leads to the desired Lipschitz constants for the model, without over-constraining each individual layer. For this purpose, we clip the convolutional layers of the model to the target value, and meanwhile, we pass the modules that represent the composition of batch norm layers and their preceding convolutional layers to our clipping method while leaving the batch norm layers themselves unclipped. We will refer to this version of our clipping method as \texttt{FastClip-concat} in our experiments. In~\Cref{fig:resnet18-catclip-study}a, we show the effect of this approach on $3$ of the layers from the ResNet-18 model trained using \texttt{FastClip-concat}. As the plot shows, the convolutional layers are correctly clipped to $1$, and the spectral norm of the concatenations are approaching the target value $1$, while the spectral norm of the batch norm layers might increase up to an order of magnitude larger than the target clipping value. The convergence of the spectral norm of the concatenation to the target value is slower because we used a much smaller $\lambda$ value (see~\Cref{alg:clip}) to make the clipping method stable without changing the other hyperparameters.  
 
 ~\Cref{fig:resnet18-catclip-study}a shows the trajectory of the training accuracies for \texttt{FastClip}, \texttt{FastClip-concat}, and \texttt{FastClip-clip BN} (which uses \texttt{FastClip} together with the direct batch norm clipping method used in prior works~\cite{gouk2021regularisation}). As the figure shows, direct clipping of the batch norm layers hinders the optimization of the model and hence leads to poor results presented in~\Cref{tab:bn-clip}, while \texttt{FastClip-concat} follows the same trajectory as \texttt{FastClip} in terms of the training accuracy, which shows less interference with the optimization of the model. Next, we elaborate on the results presented in~\Cref{tab:bn-clip}.

We investigated the performance of both ResNet-18 and DLA models on the test set, as well as their adversarial robustness, when the clipping method for batch norm layers introduced by~\cite{gouk2021regularisation} is used. These results are presented in~\Cref{tab:bn-clip}. Moreover, we investigate the application of our clipping method to the concatenation of convolutional layers and their succeeding batch norm layers, rather than controlling the batch norm layers in isolation. We present the latter results, which is unique to our method, as \texttt{FastClip-concat} in~\Cref{tab:bn-clip}. In this setting, we use a smaller value for $\lambda$ (see~\Cref{alg:clip}), and apply the clipping every $500$ steps. As the results show, with the regular batch norm clipping introduced by~\cite{gouk2021regularisation} our method still achieves the best test accuracy among the models and increases the robustness compared to the original model, however, neither of the clipping methods can achieve generalization or adversarial robustness better than the \texttt{FastClip} model with the batch norm layers removed (see~\Cref{tab:bn-cifar} and~\Cref{tab:dla-results}). This shows the previously suggested clipping method for batch norm layers, although theoretically bounding the Lipschitz constant of the model, does not help us in practice. On the other hand, \texttt{FastClip-concat} helps us improve the robustness compared to the original model and \texttt{FastClip}, but still achieves an accuracy on the test set which is close to \texttt{FastClip} and much better than the models with their batch norm layers removed.

 As these results show, \texttt{FastClip-concat} is more successful in achieving its goal; it bounds the spectral norm of the concatenation and leads to improved robustness without incurring as much loss to the test accuracy compared to removing the batch normalization layers. Further optimization of the hyperparameters (e.g., clipping value of the convolutional layer, clipping value of the concatenation, $\lambda$, number of steps for clipping, etc.) for the convolutional layers and the concatenations, and finding the best combination is left for future work.

\subsection{Robustness vs. Transferability}
\label{sec:tradeoff}

In this section, we evaluate our conjecture from \S~\ref{sec:conjecture} which was motivated by Proposition~\ref{prop:trans}. For this, we compare the ensemble of three ResNet-18 models trained without any modification (represented as \texttt{Orig}) to ensembles of models in which all the layers of each model is $C$-Lipschitz (by controlling the spectral norm). We vary this Lipschitz constant to see how it affects the robustness and transferability. The chosen Lipschitz constant for each ensemble is used to represent that in results: for example, $C=1.0$ shows that the affine layers of the models in the ensemble are all clipped to $1.0$. The clipping of the layers was achieved using FastClip~\cite{boroojeny2024spectrum}. 

Figure~\ref{fig:vanilla_compare} shows the results when the batch norm layers are removed. The first two subfigures show the changes in the average accuracy and robust accuracy of \textit{individual} ResNet-18 models. The rightmost plot shows the average $T_{rate}$ between any pair of the models in each ensemble as the layer-wise clipping value (spectral norm) changes. As the figure shows, although the robustness of \textit{individual} models increases with decreasing the clipping value, the $T_{rate}$ among the models increases. 


\subsection{Efficacy of LOTOS}
\label{exp:layer-wise_ortho}



We first evaluate the effectiveness of \texttt{LOTOS} in decreasing the $T_{rate}$ among the clipped models using the white-box attack. For this, we first use ensembles of three ResNet-18 models. ~\Cref{fig:ortho_basic_compare} shows the results for $3$ different methods of training ensembles (\texttt{Orig}, $C=1$, and \texttt{LOTOS}). The left-most subfigure shows the average test accuracy of the \textit{individual} models in each ensemble and the middle subfigure shows the average robust accuracy of the \textit{individual} models in the ensemble. The middle plot shows, as expected, that the individual models in both $C=1$ and \texttt{LOTOS} ensembles are much more robust than the ones in the \texttt{Orig} ensembles because of their Lipschitzness property; however, the \textit{individual} models in \texttt{LOTOS} are not as robust as the ones in $C=1$ because in \texttt{LOTOS} we are enforcing orthogonalization in addition to the Lipschitzness property. Because of this trade-off, as the plot shows, by increasing the value of \texttt{mal} (from $0$ to $0.8$) the robustness of the individual models becomes more similar to the ones in $C=1$ ensembles. As the right-most subfigure shows, the $T_{rate}$ between the models in ensembles trained with \texttt{LOTOS} is much lower than $C=1$ and \texttt{Orig}, and as the \texttt{mal} value decreases (the orthogonalization becomes more strict) the $T_{rate}$ decreases. 

\begin{figure}[t!]
\centering
    \includegraphics[width=.99\linewidth]{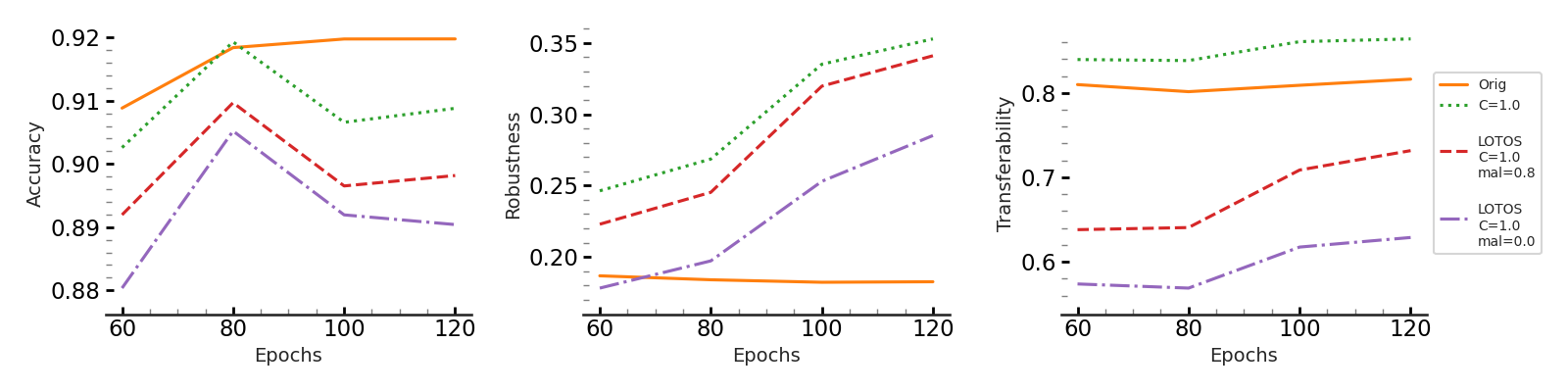}
\caption{\footnotesize  {\bf Reducing transferability while maintaining the benefits of Lipschitzness.} The average test accuracy (left-most plot) and average robust accuracy (middle plot) of the \textit{individual} models in an ensemble of three ResNet-18 models, along with the $T_{rate}$ of adversarial examples between the models of the ensemble. \texttt{LOTOS} keeps the robust accuracy of individual models in the ensemble much higher than those of the \texttt{Orig} ensemble and as \texttt{mal} increases, it becomes more similar to the models in $C=1$. \texttt{LOTOS} leads to a much lower transferability ($\simeq 20\%$) and the difference increases as \texttt{mal} decreases (right-most plot). These benefits come at a slight cost to the average accuracy of the individual models (left-most plot).
}
\label{fig:ortho_basic_compare}
\end{figure}

So far, we observed that \texttt{LOTOS} leads to a noticeable decrease in the transferability at a slight cost in test accuracy and robust accuracy of \textit{individual} models; to make sure that the former overpowers the latter and derives more robust ensembles, we evaluate the robustness of the ensembles against black-box attacks using independently trained surrogate models. We perform this experiment for both ResNet-18 and DLA models and use both CIFAR-10 and CIFAR-100 for the analysis. The results are presented in~\Cref{tab:bn-dla-resnet}. As the table shows, for each choice of the model architecture and dataset, we train ensembles with either of the 3 methods (i.e., \texttt{Orig}, $C=1$, and \texttt{LOTOS}) and compute their test accuracy and robust accuracy. As the table shows, \texttt{LOTOS} achieves higher robust accuracy in all cases with slight cost to the test accuracy in some cases. 

\begin{table*}[ht!]
\begin{center}
\begin{small}
\begin{sc}
\resizebox{\textwidth}{!}{
\begin{tabular}{@{} l  c c  c  c  c c  @{}}
 \toprule
 
 & \multicolumn{3}{@{}c}{\textbf{CIFAR-10}} & \multicolumn{3}{@{}c}{\textbf{CIFAR-100}} \\\addlinespace[0.3em]

   &  \texttt{Orig} & $C=1.0$ & \texttt{LOTOS}
 & \texttt{Orig} & $C=1.0$ & \texttt{LOTOS}  \\\addlinespace[0.2em]

     \cmidrule(r){2-4}
    \cmidrule(r){5-7}

 & \multicolumn{6}{@{}c}{Ensembles of ResNet-18 models}   \\ \addlinespace[0.4em]

 Test Acc & $\bf 95.3 \pm 0.06$ & $94.7 \pm 0.24$ & $ 94.6 \pm 0.19$
 & $\bf 77.2 \pm 0.17$ & $76.6 \pm 0.01$ & $76.6 \pm 0.10$ \\\addlinespace[0.3em]

  Robust Acc  & $30.3 \pm 1.63$ & $35.2 \pm 0.72$ & $\bf 36.3 \pm 0.88$
 & $15.2 \pm 0.45$  & $18.9 \pm 0.40$ & $\bf 20.2 \pm 0.47$ \\\addlinespace[0.3em]
 


& \multicolumn{6}{@{}c}{Ensembles of DLA models}   \\ \addlinespace[0.5em]



 Test Acc & $\bf 95.4 \pm 0.12$ & $95.2 \pm 0.05$ & $95.05 \pm 0.09$
 & $77.1 \pm 0.09$ & $\bf 78.8 \pm 0.31$ & $78.3 \pm 0.38$ \\\addlinespace[0.3em]

 Robust Acc & $26.7 \pm 0.58$ & $32.8 \pm 1.28$ & $\bf 34.5 \pm 0.63$  
 & $16.5 \pm 0.78$ & $19.4 \pm 0.32$ & $\bf 21.0 \pm 0.39$ \\\addlinespace[0.3em]
 


\bottomrule
\end{tabular}
}
\end{sc}
\end{small}
\end{center}
\vskip -0.1in
\caption{\footnotesize {\bf Robust accuracy against black-box attacks in ensembles of ResNet-18 models and ensembles of DLA models trained on CIFAR-10 and CIFAR-100 }. The surrogate models are a combination of both original models and clipped models trained with multiple random seeds. The target models are ensembles of three models from each architecture choice that are trained using either of the three training methods.}
\label{tab:bn-dla-resnet}
\end{table*}

\newpage

\chapter{Machine Unlearning}
\label{sec:amun-chapter}

Given a training set $\D$ and a subset $\Df \subset \D$ of the samples that have to be unlearned from a model trained on $\D$, recent works on unlearning have emphasized the use of evaluation metrics that measure the similarity to the behavior of the models that are retrained from scratch on $\D - \Df$. However, prior unlearning methods do not effectively incorporate this evaluation criterion in the design of their algorithm. In this paper, we first characterize the expected behavior of the retrained-from-scratch models on $\D - \Df$. Using this characterization, we propose Adversarial Machine UNlearning (\amun{}). \amun{} is a method that, when applied to the models trained on $\D$, replicates that (desired) behavior after a few iterations. The success of \amun{} relies on an intriguing observation: fine-tuning a trained model on the adversarial examples of the training data does not lead to a catastrophic forgetting and instead has limited effect on the deterioration of model's test accuracy. 

Upon receiving a request for unlearning a subset $\Df$ of the training set $\D$, \amun{} finds adversarial examples that are {\em as close as possible} to the samples in $\Df$. It then utilizes these adversarial examples (with the wrong labels) during fine-tuning of the model for unlearning the samples in $\Df$. Fine-tuning the model on these adversarial examples, which are naturally mispredicted by the model, decreases the confidence of the predictions on $\Df$. This decreased confidence of model's predictions on $\Df$ is similar to what is observed in the models that are retrained on $\D - \Df$. The distance of these adversarial examples to their corresponding samples in $\Df$ is much smaller than the distance of $\Df$ to other samples in $\D - \Df$; this localizes the effect of fine-tuning to the vicinity of the samples in $\Df$ and prevents significant changes to the decision boundary of the model and hurting the model's overall accuracy (see \S~\ref{sec:amun_motivation}). 

As we will show in \S~\ref{sec:amun_experiments}, \amun{} outperforms prior state-of-the-art (SOTA) unlearning methods~\cite{fan2023salun} in unlearning random subsets of the training data from a trained classification model and closes the gap with the retrained models, even when there is no access to the samples in $\D - \Df$ during the unlearning procedure.

\section{Preliminaries}
\label{sec:amun_prelim}

We begin by introducing the notation we use. We proceed to define various terms in the paper, and conclude by introducing our method.

\subsection{Notation} 
\label{subsec:notation}

Assume a probability distribution $\mathbb{P}_\mathcal{X}$ on the domain of inputs $\mathcal{X}$ and $m$ classes $\mathcal{Y} = \{1,2,\dots,m\}$. We consider a multi-class classifier $\mathcal{F}:\mathcal{X} \rightarrow \mathcal{Y}$ and its corresponding prediction function $f(x)$ which outputs the probabilities corresponding to each class (e.g., the outputs of the softmax layer in a neural network). The loss function for model $\mathcal{F}$ is denoted $\ell_\mathcal{F}: \mathcal{X} \times \mathcal{Y} \rightarrow \mathbb{R_+}$; it uses the predicted scores from $f(x)$ to compute the loss given the true label $y$ (e.g., cross-entropy loss). 

In the supervised setting we consider here, we are given a dataset $\mathcal{D} = \{(x_i, y_i)\}_{i=\{1, \dots, N\}}$ that contains labeled samples $x_i \sim \mathbb{P}_\mathcal{X}$ with $y_i \in \mathcal{Y}$. The model $\F$ is trained on $\D$ using the loss $\ell_\F$ to minimize the empirical risk $\E_{\mathcal{D}} [\ell_\mathcal{F}(x,y)]$ and a set of parameters $\thetafull \sim {\Theta}_\D$ is derived for $\F$; $\Theta_\D$ is the distribution over the set of all possible parameters $\Theta$ when the training procedure is performed on $\D$ due to the potential randomness in the training procedure (e.g., initialization and using mini-batch training). We also assume access to a test set $\Dt$ with samples from the same distribution $\mathbb{P}_\mathcal{X}$. A function $g(x)$ is $L$-Lipschitz if $\| g(x) - g(x^\prime)\|_2 \leq L \| x-x^\prime \|_2, \forall x,x^\prime \in \mathcal{X}$.

\subsection{Definitions}
\label{sec:amun_defs}

\begin{definition}[Attack Algorithm]
\label{def:attack}
    For a given input/output pair $(x,y) \in \mathcal{X} \times \mathcal{Y}$, a model $\mathcal{F}$, and a positive value $\epsilon$, an untargetted attack algorithm $\A_\mathcal{F}(x, \epsilon) = x+\delta_x$ minimizes $\ell_\F(x+\delta_x,y^\prime \neq y)$ such that $\|\delta_x\|_2 \leq \epsilon$, where $y^\prime \in \mathcal{Y}$.
\end{definition}

\begin{definition}[Machine Unlearning]
\label{def:mu}
Given the trained model $\F$, and a subset $\Df \subset \D$ known as the forget set, the corresponding machine unlearning method is a function $\mathcal{M}_{\D,\Df}: \Theta \rightarrow \Theta$ that gets $\thetafull \sim \Theta_\D$ as input and derives a new set of parameters (aka the unlearned model) $\thetaunl \sim \Theta_\Df$, where $\Theta_\Df$ is the distribution over the set of parameters when $\F$ is trained on $\D - \Df$ rather than $\D$. 
\end{definition}

\subsection{Approximate Unlearning}

Using Definition~\ref{def:mu}, it is clear that the most straight-forward, exact unlearning method would be to retrain model $\F$ from scratch on $\D - \Df$; this does not even use $\thetafull$. However, training deep learning models is very costly, and retraining the models upon receiving each unlearning request would be impractical. Thus, approximate unlearning methods are designed to overcome these computational requirements by starting from $\thetafull$ and modifying the parameters to derive $\thetafull^\prime$ s.t.  $\thetafull^\prime \stackrel{\text{d}}{=} \thetaunl$ (i.e., from the same distribution). 
 
In the rest of the paper, we refer to $\Df$ as the forget or unlearning set interchangeably. Its complement, $\Dr = \D -\Df$ is the remain set. We will use the behavior of the models retrained from scratch on $\Dr$ as the goal of approximate unlearning methods, and will refer to them as $\Fr$ for brevity. 

\subsection{Unlearning Settings}
\label{sec:unlear_setting}

Many of the prior methods on approximate unlearning for classification models require access to $\Dr$. However, in practice, this assumption might be unrealistic. The access to $\Dr$ might be restricted, or might be against privacy regulations. 
Prior works do not make a clear distinction based on this requirement when comparing different approximate methods. Therefore, to make a clear and accurate comparison, we perform our experiments (see \S~\ref{sec:amun_experiments}) in two separate settings: one with access to both $\Dr$ and $\Df$, and the other with access to only $\Df$. We report the results for each setting separately. For comparison with prior methods, we adapt them to both settings whenever possible.

\section{Motivation}
\label{sec:amun_methods}

We start by presenting intuition for our proposed unlearning method in \S~\ref{sec:amun_motivation}, and in \S~\ref{sec:adv-finetune} we proceed with describing our observation about fine-tuning a model on its adversarial examples.

\subsection{A Guiding Observation}
\label{sec:amun_motivation}

Before designing a new unlearning method, we would like to first characterize the changes we expect to see after a successful unlearning. Because the retrained models are the gold standard of unlearning methods, we first assess their behavior on $\Dr$, $\Df$, and $\Dt$. To this end, we evaluate the confidence values of $\Fr$ when predicting labels of $\Dr$, $\Df$, and $\Dt$. Since samples in $\Dt$ are drawn from the same distribution as $\D$, we can conclude that samples in $\Dt$ and $\Df$ are from the same distribution. Therefore, we expect $\Fr$ to have similar accuracy and prediction confidence scores on $\Dt$ (test set) and $\Df$. 

\noindent{\bf Results:} Experiments show that the confidence scores for a ResNet-18~\cite{he2016deep} model that has been retrained on $\D - \Df$, where $\D$ is the training set of CIFAR-10~\cite{alex2009learning} and the size of $\Df$ (randomly chosen from $\D$) is $10\%$ and $50\%$ of the size of $\D$ (the first and second sub-figures, respectively). 

\begin{tcolorbox}
\noindent{\bf Key Observation 1:} {\em The main difference between the predictions on $\Dt$ (unseen samples) and $\Dr$ (observed samples) is that the model's predictions are {\em much more confident} for the samples that it has observed compared to the unseen samples.}
\end{tcolorbox}

This basic observation has either been overlooked by the prior research on approximate machine unlearning or has been treated incorrectly. To make the unlearned models more similar to $\Fr$, prior methods have focused on degrading the model's performance on $\Df$ directly by either (a) some variation of fine-tuning on $\Dr$~\cite{warnecke2021machine,liu2024model}, (b) choosing wrong labels for samples in $\Df$ and fine-tuning the model~\cite{golatkar2020eternal,chen2023boundary,fan2023salun}, or (c) directly maximizing the loss with respect to the samples in $\Df$~\cite{thudi2022unrolling}. 
Using the wrong labels for the samples of $\Df$ or maximizing the loss on them make these methods very unstable and prone to catastrophic forgetting~\cite{zhang2024negative} because these samples belong to the correct distribution of the data and we cannot force a model to perform wrongly on a portion of the dataset while preserving it's test accuracy.
For these methods, it is important to use a small enough learning rate along with early stopping to prevent compromising the model's performance while seeking worse prediction confidence values on the samples in $\Df$. Also, most of these methods require access to the set of remaining samples to use it for preventing a total loss of the model's performance (e.g., by continuing to optimize the model on $\Dr$)~\cite{golatkar2020eternal,liu2024model}.

\subsection{Fine-tuning on Adversarial Examples}
\label{sec:adv-finetune}

\begin{figure*}[t!]
\centering
\includegraphics[width=.98\linewidth]{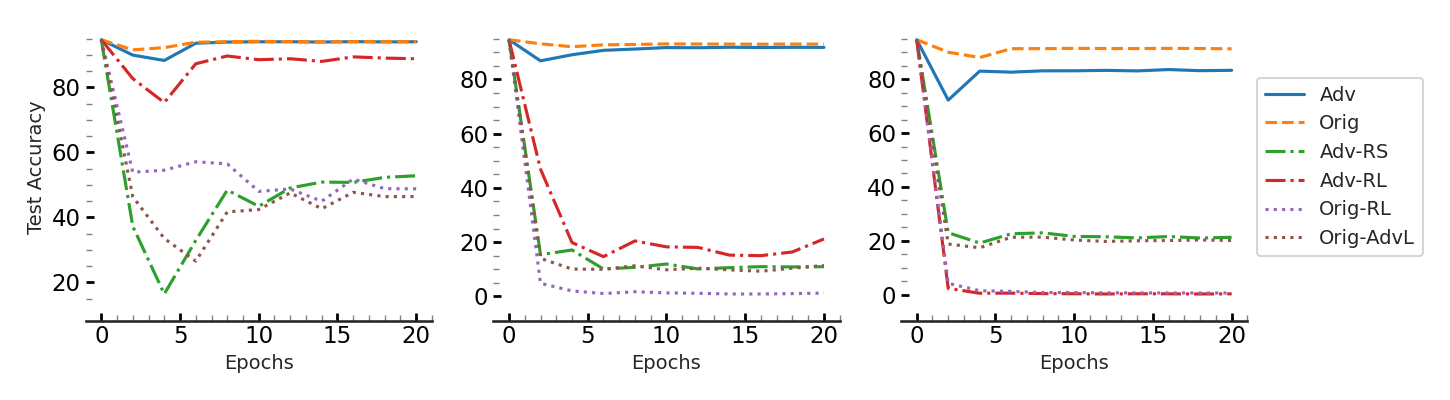}

\caption{{\bf Effect of fine-tuning on adversarial examples.} This figure shows the effect of fine-tuning on test accuracy of a ResNet-18 model that is trained on CIFAR-10, when the dataset for fine-tuning changes. Let $\Df$ contain $10\%$ of the samples in $\D$ and $\Da$ be the set of adversarial examples constructed using Algorithm~\ref{alg:advset}. \texttt{Adv}, from the left sub-figure to right one, shows the results when $\D \cup \Da$, $\Df \cup \Da$, and $\Da$ is used for fine-tuning the model, respectively. \texttt{Orig}, \texttt{Adv-RS}, \texttt{Adv-RL}, \texttt{Orig-RL}, and \texttt{Orig-AdvL} shows the results when $\Da$ for each of these sub-figures is replace by $\Df$, $\Da_{RS}$, $\Da_{RL}$, $\D_{RL}$, and $\D_{AdvL}$, accordingly. As the figure shows, the specific use of adversarial examples with the mis-predicted labels matters in keeping the model's test accuracy because $\Da$, in contrast to the other constructed datasets belong to the natural distribution learned by the trained model.} 
\vspace{-5mm}
\label{fig:fine_tune_10}
\end{figure*}

After training a model $\F$ on $\D$, this model imposes a distribution $f(x)$ (e.g., softmax outputs) for all possible labels $y\in \mathcal{Y}$ given any $x\in \mathcal{X}$. Since the model $\F$ is directly optimized on $\D$, $f(x)$ becomes very skewed toward the correct class for samples in $\D$. For a given sample from $\D$, its adversarial examples (see Definition~\ref{def:attack}) are very close in the input space to the original sample. However, $\F$ makes the wrong prediction on these examples. This wrong prediction is the direct result of the learned parameters $\thetafull$ for the classification model, and these adversarial examples, although predicted incorrectly, belong to the distribution imposed on $\mathcal{X}$ by these learned parameters (i.e., even though that is not the correct distribution, that is what the model has learned).

Now, what happens if we insert one adversarial example $(x_{adv}, y_{adv})$ that corresponds to the sample $(x,y)$ into $\D$ and make an augmented dataset $\D^\prime$ for fine-tuning? Even before fine-tuning starts, the model makes the correct prediction on that (adversarial) example (by predicting the wrong label $y_{adv}$!), but its confidence might not be as high as the samples in $\D$, on which the model has been trained on. Proceeding with
fine-tuning of the model on the augmented dataset increases its confidence on $x_{adv}$ while making the same wrong prediction $y_{adv}$. However, this fine-tuning does not change the model's performance because the newly added sample $(x_{adv}, y_{adv})$ does not contradict the distribution learned by the model. Since $(x,y)\in \D^\prime$, and $x$ and $x_{adv}$ are very close to one another (e.g., very similar images) while having different labels, optimizing the model {\em has to change its decision boundary} in that region of the input space to reach small loss for both of these samples. As a result of this balance, the model tends to decrease its confidence on the original sample compared to the model that was solely trained on $\D$ because there was no opposing components for its optimization on $\D$. Note that $\|x-x_{adv}\| \leq \epsilon$, where $\epsilon$ is often much smaller than the distance of any pairs of samples in $\D$. This helps to localize this change in the decision boundary during fine-tuning, and prevent changes to models' behavior in other regions of the input space~\cite{liang2023towards}. In the following we elaborate on our empirical observations that verify these changes.

\noindent{\bf Setup:} We consider the training set of CIFAR-10 as $\D$ and choose $\Df$ to be a random subset whose size is $10\%$ of $|\D|$. We also compute a set of adversarial examples (using Algorithm~\ref{alg:advset}) corresponding to $\Df$, which we call $\Da$. Fig.~\ref{fig:fine_tune_10} shows the fine-tuning of a trained ResNet-18 model for 20 epochs. In the leftmost sub-figure, the curve presented as \texttt{Orig} represents the test accuracy of the model when it is fine-tuned on $\D$. The curve named \texttt{Adv} is fine-tuned on $\D \cup \Da$, which has a similar test accuracy to \texttt{Orig}.

In the second sub-figure, \texttt{Orig} shows the test accuracy of the model when it is fine-tuned on $\Df$ (two copies of $\Df$ to keep the sample count similar), while \texttt{Adv} represents fine-tuning on $\Df \cup \Da$. As the figure shows, \texttt{Adv} has a small degradation in test accuracy compared to \texttt{Orig}. 

The rightmost sub-figure shows the case where \texttt{Orig} is fine-tuning of the model on $\Df$, and \texttt{Adv} is fine-tuning on only $\Da$. Although the degradation in test accuracy increases for this case, surprisingly we see that the model still remains noticeably accurate despite being fine-tuned on a set of samples that are all mislabeled. 

\noindent{\bf Results:} As Figure~\ref{fig:fine_tune_10} shows, the test accuracy of the model does not deteriorate, even when it is being fine-tuned on only $\Da$ (the dataset with wrong labels). 

\begin{tcolorbox}
\noindent{\bf Key Observation 2:} \textit{Fine-tuning a model on the adversarial examples does not lead to catastrophic forgetting!}
\end{tcolorbox}

\section{Methods}
\label{sec:amun}

We utilize our novel observation about the effect of fine-tuning on adversarial examples (see \S~\ref{sec:adv-finetune}) to achieve the intuition we had about the retrained models (see \S~\ref{sec:amun_motivation}). We utilize the existing flaws of the trained model in learning the correct distribution, that appear as adversarial examples in the vicinity of the samples in $\Df$, to decrease its confidence on those samples while maintaining the performance of the model. 

Formally, Adversarial Machine UNlearning (\amun{}) uses Algorithm~\ref{alg:advset} to find an adversarial example for any sample in $(x,y) \in \Df$. This algorithm uses a given untargeted adversarial algorithm $\mathcal{A}_\mathcal{F}$, that finds the solution to~\Cref{def:attack}, for finding an adversarial example $x_{adv}$. To make sure $\epsilon$ is as small as possible, Algorithm~\ref{alg:advset} starts with a small $\epsilon$ and runs the attack $\mathcal{A}_\mathcal{F}$; if an adversarial algorithm is not found within that radius, it runs $\mathcal{A}_\mathcal{F}$ with a larger $\epsilon$. It continues to perform $\mathcal{A}_\mathcal{F}$ with incrementally increased $\epsilon$ values until it finds an adversarial example; it then adds it to $\Da$. The algorithm stops once it finds adversarial examples for all the samples in $\Df$. 

The reason behind minimizing the distance of $\epsilon$ for each sample is to localize the changes to the decision boundary of the model as much as possible; this prevents changing the model's behavior on other parts of the input space. For our experiments, we use PGD-50~\cite{madry2017towards} with $\ell_2$ norm bound as $\mathcal{A}_\mathcal{F}$. We set the step size of the gradient ascent in the attack to $0.1\times\epsilon$, which changes with the $\epsilon$ value. 

\begin{algorithm}[H]
\caption{Build Adversarial Set ($\mathcal{F}, \mathcal{A}, \Df, \epsilon_{init}$)}
\label{alg:advset}
\begin{algorithmic}[1]
\STATE {\bfseries Input:} Model $F$, Attack algorithm $A$, Forget set $\Df$, and Initial $\epsilon$ for adversarial attack
\STATE {\bfseries Output:} $\Da$: Adversarial set for $\Df$ 
\STATE $\Da = \{\}$
    \FOR{$(x,y)$ {\bfseries in} $\Df$}
        \STATE $\epsilon = \epsilon_{init}$
        \WHILE{TRUE}
            \STATE $x_{adv} = \mathcal{A}(x, \epsilon)$
            \STATE $y_{adv} = \mathcal{F}(x_{adv})$ 
            \IF{$y_{adv} \, != \, y$}
                \STATE Break
            \ENDIF
            \STATE $\epsilon = 2 \epsilon$
        \ENDWHILE
        \STATE Add $(x_{adv}, y_{adv})$ to $\Da$ 
    \ENDFOR

\STATE {\bfseries Return} $\Da$

\end{algorithmic}

\end{algorithm}

\vspace{-4pt}

Once Algorithm~\ref{alg:advset} constructs $\Da$, \amun{} utilizes that to augment the dataset on which it performs the fine-tuning. If $\Dr$ is available, \amun{} fine-tunes the model on $\Dr \cup \Df \cup \Da$  and when $\Dr$ is not accessible, it performs the fine-tuning on $\Df \cup \Da$. Also, in the setting where the size of the $\Df$ is very large, we noticed some improvement when using only $\Dr \cup \Da$ and $\Da$, for those settings, repsectively.

\section{Theoretical Results}

Although most prior SOTA methods in approximate unlearning are not accompanied by theoretical guarantees, we prove a theorem that derives an upper-bound on the $2$-norm of the difference of the parameters of the unlearned model and the retrained model (which are gold-standard for unlearning). To prove this theorem, we make assumptions that are common in the certified unlearning literature. Our derived upper-bound implies enhanced effectiveness of our method when:

\begin{enumerate}
    \item The distance between the forget sample and its corresponding adversarial example becomes smaller.
    \item The Lipschitz constant of the model becomes smaller.
    \item The quality of the adversarial example becomes stronger (causes a larger loss for the correct label).
    \item The adversarial example transfers better to the retrained model.
    \item The retrained model generalizes better to the (clean) unseen samples.
\end{enumerate}

Hence, the proved theorem also justifies our earlier intuitions about the need for good quality adversarial examples that are as close as possible to the original samples (which is the goal of Algorithm 1), and also justifies that by fine-tuning the model on these adversarial examples, we can derive an upper-bound on the distance between the retrained model and the unlearned one. We believe that the presented empirical results, along with the provided theorem, will motivate further theoretical studies in future work.

\begin{theorem}
\label{theorem:amun}
    Let $\mathcal{D} = \{(x_i, y_i)\}_{i=\{1, \dots, N\}}$ be a dataset of $N$ samples and without loss of generality let $(x_n, y_n)$ (henceforth represented as $(x,y)$ for brevity) be the sample that needs to be forgotten and $(x^\prime, y^\prime)$ be its corresponding adversarial example such that $\|x-x^\prime\|_2 = \delta$. Let $\hat{\mathcal{R}}(w)$ represent the (unnormalized) empirical loss on $\mathcal{D}^\prime = \mathcal{D} \cup \{(x^\prime, y^\prime)\}$ for a model $f$ that is parameterized with $w$. We assume that $f$ is $L$-Lipschitz with respect to the inputs and $\hat{\mathcal{R}}$ is $\beta$-smooth and convex with respect to the parameters. Let $\theta_o$ represent the parameters corresponding to the model originally trained on $\mathcal{D}$ and $\theta_u$ be the parameters derived when the model is trained on $\mathcal{D} - \{(x,y)\}$. We also assume that both the original and retrained models achieve near-$0$ loss on their training sets. After performing fine-tuning on $\mathcal{D}^\prime$ using one step of gradient descent with a learning rate of $\frac{1}{\beta}$ to derive parameters $\theta^\prime$, we get the following upper-bound for the distance of the unlearned model and the model retrained on $\mathcal{D} - \{(x,y)\}$ (gold standard of unlearning):

    \begin{align*}
        \|\theta^\prime - \theta_u\|_2^2 \leq \|\theta_o - \theta_u\|_2^2 + \frac{2}{\beta} (L \delta - C),
    \end{align*}

    \noindent where $C$ is a term that depends on:

    \begin{itemize}
        \item The quality of adversarial example in increasing the loss for the correct label on the original model: $\ell(f_{\theta_o}(x^\prime), y)$
        \item Transferability of the adversarial example to the retrained model to decrease its loss for the wrong label: $\ell(f_{\theta_u}(x^\prime), y^\prime)$
        \item Early stopping and using appropriate learning rate during fine-tuning to avoid overfitting to the adversarial example: $\ell(f_{\theta^\prime}(x^\prime), y^\prime)$
        \item The generalization of the retrained model to the unseen samples: $\ell(f_{\theta_u}(x), y) $

    \end{itemize}

    More specifically, $C = \ell(f_{\theta_o}(x^\prime), y) + \ell(f_{\theta^\prime}(x^\prime), y^\prime) -\ell(f_{\theta_u}(x), y)  - \ell(f_{\theta_u}(x^\prime), y^\prime) $.

\end{theorem}


\section{Empirical Results}
\label{sec:amun_evaluation}

\label{sec:amun_experiments}

We wish to answer the following questions: (1) does \amun{} lead to effective unlearning of any random subset of the samples when evaluated by a SOTA MIA?; (2) does the choice of $\Da$ matter in \amun{}, or can it be replaced with a dataset that contains different labels or different samples that are within the same distance to the corresponding samples in $\Df$?; (3) is \amun{} effective on adversarially robust models?; (4) does the choice of attack method matter in Algorithm~\ref{alg:advset} used by \amun{} and does transferred attack work as well?; and (5) how does \amun{} compare to other unlearning methods when used for performing multiple unlearning requests on the same model?

As a quick summary, our results show that: (1) \amun{} effectively leads to unlearning the samples in $\Df$: after unlearning $10\%$ of the samples of CIFAR-10 from a trained ResNet-18, RMIA cannot do better than random guessing (\S~\ref{sec:amun_results}); (2) If we replace $\Da$ with any of the aforementioned substitutes, the model's accuracy significantly deteriorates, especially when there is no access to $\Dr$; (3) \amun{} is as effective for unlearning on models that are adversarially robust; (4) using weaker attack methods, such as FGSM, in \amun{} hurts the effectiveness by not finding the adversarial examples that are very close to the samples in $\Df$. However, they still outperform prior methods. The transferred adversarial examples are effective as well; and (5) \amun{} outperforms other unlearning methods when handling multiple unlearning requests. We elaborate on the results for (1) here. For details on the rest, please see our paper. For more details please see our published version.

\subsection{Baseline Methods}
\label{sec:amun_baselines}

We compare \amun{} with \texttt{FT}~\cite{warnecke2021machine}, \texttt{RL}~\cite{golatkar2020eternal}, \texttt{GA}~\cite{thudi2022unrolling}, \texttt{BS}~\cite{chen2023boundary}, $l_1$\texttt{-Sparse}~\cite{liu2024model}, and \texttt{SalUn}~\cite{fan2023salun}. We also combine the weight saliency idea for masking the model parameters to limit the changes to the parameters during fine-tuning with \amun{} and present its results as \amun{}$_{+SalUn}$ (see Appendix~\ref{apx:amun_salun} for more details). We use the same hyper-parameter tuning reported by prior works. For further details, see Appendix~\ref{apx:impl_details}.

\subsection{Evaluation Metrics}
\label{sec:amun_metrics}

The metic used by recent works in unlearning to evaluate the unlearning methods~\cite{liu2024model,fan2023salun} considers the models retrained on $\Dr$ as the goal standard for comparison. They compute the following four values for both the retrained models and the models unlearned using approximate methods:

\begin{itemize}
\itemsep0em
\item \textit{Unlearn Accuracy:} Their accuracy on $\Df$.
\item \textit{Retain Accuracy:} Their accuracy on $\Dr$.
\item \textit{Test Accuracy:} Their accuracy on $\Dt$.
\item \textit{MIA score}: Scores returned by membership inference attacks on $\Df$ 
\end{itemize}

Once these four values are computed, the absolute value of the difference of each of them with the corresponding value for $\mathcal{F}_R$ (the retrained models) is computed. Finally, the average of the four differences (called the {\em Average Gap}) is used as the metric to compare the unlearning methods. 

The MIAs used in the recent unlearning methods by~\cite{liu2024model,fan2023salun} are based on the methods introduced by~\cite{yeom2018privacy,song2019privacy}.
Although these MIAs have been useful for basic comparisons, recent SOTA MIAs significantly outperform their earlier counterparts, albeit with an increase in complexity and computation cost. To perform a comprehensive comparison of the effectiveness of the unlearning methods, we utilized a SOTA MIA called RMIA~\cite{zarifzadeh2024low}, in addition to using the MIAs from prior works. In RMIA, the area under the ROC curve (AUC) of the MIA scores for predicting the training samples from the unseen samples is reported. Recall that in machine unlearning, the samples are split to three sets: $\Dr$, $\Df$, and $\Dt$. For an unlearning method to be effective, as discussed in \S~\ref{sec:amun_motivation}, we expect the AUC of RMIA for distinguishing the samples in $\Df$ from the ones in $\Dt$ to be the same as random guessing ($50\%$ assuming balanced data). As shown in Table~\ref{tab:rmia}, this expectation holds for the models retrained on $\Dr$. 

We report the results of our comparisons for both the MIAs from prior unlearning literature and the new SOTA MIA. We will present the former one as \texttt{MIS}
, and the latter one as \texttt{FT AUC} (the AUC of predicting $\Df$ from $\Dt$).

\subsection{Unlearning Settings} 
\label{sec:amun_settings}

Another important factor missing in the comparisons of the unlearning methods in prior works is the possibility of access to $\Dr$. So, for our experiments we consider two settings, one with access to $\Dr$ and one with access to only $\Df$. We adapt each of the unlearning methods to both of these settings, and perform the comparisons in each of these settings separately. The prior unlearning methods that do not adapt to the setting where there is no access to $\Dr$ \cite{warnecke2021machine,liu2024model} are excluded for the presented results in that setting.

Therefore, we perform different sets of experiments to evaluate the unlearning methods in both settings, and hope this becomes the norm in future works in machine unlearning. In each of these two settings, we evaluate unlearning of $10\%$ or $50\%$ of the samples randomly chosen from $\D$. For all the experiments we train three models on $\D$. For each size of $\D$, we use three random subsets and for each subset, we try three different runs of the unlearning methods. This leads to a total of $27$ runs of each unlearning method using different initial models and subsets of $\D$ to unlearn.

\subsection{Effectiveness of \amun{}}
\label{sec:amun_results}

In this subsection we report the results on the comparisons of \amun{} to other unlearning methods (see \S~\ref{sec:amun_baselines}). We consider the unlearning settings discussed in \S~\ref{sec:amun_settings}, and the evaluation metrics discussed in \S~\ref{sec:amun_metrics}. We use ResNet-18 models trained on CIFAR-10 for this experiment. 

\noindent{\bf Results:} Table~\ref{tab:rmia} shows the results of evaluation using RMIA when the unlearning methods {\em have access to $\Dr$}. Table~\ref{tab:rmia_forgetonly} shows these results when there is {\em no access to $\Dr$}. As the results show, \amun{} {\em clearly outperforms prior unlearning methods in all settings}. This becomes even more clear when there is no access to $\Dr$. Note that, for the models retrained on $\Dr$, the AUC score of RMIA for predicting $\Dr$ from $\Dt$ (which can be considered as the worst case for \texttt{FT AUC} score) are $64.17$ and $69.05$ for unlearning $10\%$ and $50\%$ accordingly. 

We also present the results when MIS is used as the evaluation metric in Tables~\ref{tab:mia} and~\ref{tab:mia_forgetonly} in Appendix~\ref{apx:svc_mia}, which similarly shows \amun{}'s dominance in different unlearning settings. Moreover, we evaluate the combination of \amun{} and \texttt{SalUn} (see Appendix~\ref{apx:amun_salun} for details) and present its results as \amun{}$_{SalUn}$ in these tables. \amun{}$_{SalUn}$ slightly improves the results of \amun{} in the setting where there is no access to $\Dr$, by filtering the parameters that are more relevant to $\Df$ during fine-tuning.

\begin{figure}[t!]
\centering
\includegraphics[width=.98\columnwidth]{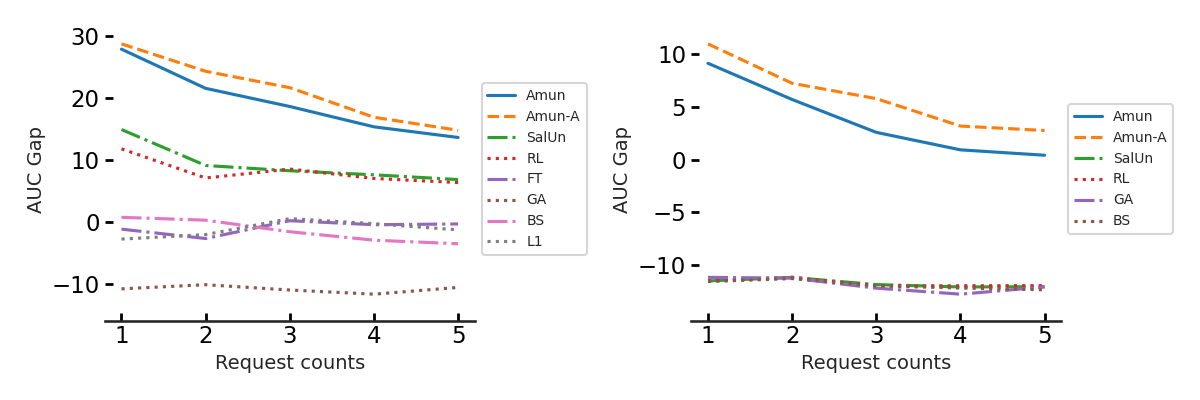}
\caption{{\bf Multiple unlearning requests.} This figure shows the evaluation of unlearning methods when they are used for unlearning for five times and each time on $2\%$ of the training data. We train a ResNet-18 model on CIFAR-10 when $\Dr$ is available (left) and when it is not (right). After each step of the unlearning, we use the MIA scores generated by RMIA to derive the area under the ROC curve (AUC) for $\Dr$ vs. $\Df$ and $\Df$ vs. $\Dt$. The values on the y-axis shows the difference of these two AUC scores. A high value for this gap means the samples in $\Df$ are far more similar to $\Dt$ rather than $\Dr$ and shows a more effective unlearning.} 
\vspace{-5mm}
\label{fig:adaptive}
\end{figure}

\begin{table*}[th!]
\begin{center}
\begin{small}
\begin{sc}
\resizebox{\textwidth}{!}{
\begin{tabular}{@{} l  c c  c  c c | c c c c c @{}}
 \toprule


& \multicolumn{5}{@{}c}{\textbf{Random Forget ($10\%$)}} & \multicolumn{5}{@{}c}{\textbf{Random Forget ($50\%$)}} \\\addlinespace[0.3em]

 \multicolumn{1}{c}{\scriptsize \textbf{}} & 
 \multicolumn{1}{c}{\scriptsize Unlearn Acc} & 
 \multicolumn{1}{c}{\scriptsize Retain Acc} & 
 \multicolumn{1}{c}{\scriptsize Test Acc} & 
 \multicolumn{1}{c}{\scriptsize FT AUC} & 
 \multicolumn{1}{c}{\scriptsize Avg. Gap} & 
 \multicolumn{1}{c}{\scriptsize Unlearn Acc} &
 \multicolumn{1}{c}{\scriptsize Retain Acc} & 
 \multicolumn{1}{c}{\scriptsize Test Acc} & 
 \multicolumn{1}{c}{\scriptsize  FT AUC} & 
 \multicolumn{1}{c}{\scriptsize Avg. Gap} 
 \\\addlinespace[0.3em]

 \cmidrule(r){2-6}
 \cmidrule(r){7-11}



 
 

 Retrain & $94.49$ {\tiny $\pm 0.20$} & $100.0 $ {\tiny $\pm 0.00$} & $94.33$ {\tiny $\pm 0.18$} & $50.00$ {\tiny $\pm 0.42$} & $0.00$ 
 & 
 
 $92.09$ {\tiny $\pm 0.37$}  & $100.0$ {\tiny $\pm 0.00$} & $91.85$ {\tiny $\pm 0.33$} & $50.01$ {\tiny $\pm 0.12$} & $0.00$
 \\\addlinespace[0.3em]
  \cmidrule(r){1-11}

 FT  & $95.16$  { \tiny $\pm 0.29$ }  & $96.64$  { \tiny $\pm 0.25$ }  & $92.21$  { \tiny $\pm 0.27$ }  & $52.08$  { \tiny $\pm 0.34$ }  & $2.06$  { \tiny $\pm 0.10$ } 
 & 
 
 $94.24$  { \tiny $\pm 0.30$ }  & $95.82$  { \tiny $\pm 0.31$ }  & $91.21$  { \tiny $\pm 0.33$ }  & $51.74$  { \tiny $\pm 0.36$ }  & $2.17$  { \tiny $\pm 0.13$ } 
 \\\addlinespace[0.3em]
 
 RL & $95.54$  { \tiny $\pm 0.14$ }  & $97.47$  { \tiny $\pm 0.08$ }  & $92.17$  { \tiny $\pm 0.10$ }  & $51.33$  { \tiny $\pm 0.63$ }  & $1.74$  { \tiny $\pm 0.18$ }  
 & 
 
 $94.83$  { \tiny $\pm 0.44$ }  & $99.79$  { \tiny $\pm 0.04$ }  & $90.08$  { \tiny $\pm 0.16$ }  & $50.78$  { \tiny $\pm 0.14$ }  & $1.38$  { \tiny $\pm 0.09$ } 
 \\\addlinespace[0.3em]

  GA & $98.94$  { \tiny $\pm 1.39$ }  & $99.22$  { \tiny $\pm 1.31$ }  & $93.39$  { \tiny $\pm 1.18$ }  & $60.96$  { \tiny $\pm 2.93$ }  & $4.28$  { \tiny $\pm 0.47$ }  
  & 
 
 $100.00$  { \tiny $\pm 0.00$ }  & $100.00$  { \tiny $\pm 0.00$ }  & $94.65$  { \tiny $\pm 0.07$ }  & $63.39$  { \tiny $\pm 0.26$ }  & $4.62$  { \tiny $\pm 0.05$ } 
 \\\addlinespace[0.3em]

 BS & $99.14$  { \tiny $\pm 0.31$ }  & $99.89$  { \tiny $\pm 0.06$ }  & $93.04$  { \tiny $\pm 0.14$ }  & $57.85$  { \tiny $\pm 1.12$ }  & $3.48$  { \tiny $\pm 0.32$ } 
 & 
  
  $55.24$  { \tiny $\pm 5.11$ }  & $55.67$  { \tiny $\pm 4.90$ }  & $50.16$  { \tiny $\pm 5.28$ }  & $55.19$  { \tiny $\pm 0.42$ }  & $32.01$  { \tiny $\pm 3.86$ } 
 \\\addlinespace[0.3em]

  $l_1$-Sparse & $94.29$  { \tiny $\pm 0.34$ }  & $95.63$  { \tiny $\pm 0.16$ }  & $91.55$  { \tiny $\pm 0.17$ }  & $51.21$  { \tiny $\pm 0.32$ }  & $2.16$  { \tiny $\pm 0.06$ } 
  & 
  
  $98.00$  { \tiny $\pm 0.17$ }  & $98.71$  { \tiny $\pm 0.13$ }  & $92.79$  { \tiny $\pm 0.10$ }  & $54.44$  { \tiny $\pm 0.47$ }  & $2.67$  { \tiny $\pm 0.11$ } 
 \\\addlinespace[0.3em]

  SalUn & $96.25$  { \tiny $\pm 0.21$ }  & $98.14$  { \tiny $\pm 0.16$ }  & $93.06$  { \tiny $\pm 0.18$ }  & $50.88$  { \tiny $\pm 0.54$ }  & $1.44$  { \tiny $\pm 0.12$ } 
  & 

  $96.68$  { \tiny $\pm 0.35$ }  & $99.89$  { \tiny $\pm 0.01$ }  & $91.97$  { \tiny $\pm 0.18$ }  & $50.86$  { \tiny $\pm 0.18$ }  & $1.36$  { \tiny $\pm 0.04$ } 
 \\\addlinespace[0.5em]

 \textbf{Amun}  & $95.45$  { \tiny $\pm 0.19$ }  & $99.57$  { \tiny $\pm 0.00$ }  & $93.45$  { \tiny $\pm 0.22$ }  & $50.18$  { \tiny $\pm 0.36$ }  & $\bf 0.62$  { \tiny $\pm 0.05$ } 
 & 
 
  $93.50$  { \tiny $\pm 0.09$ }  & $99.71$  { \tiny $\pm 0.01$ }  & $92.39$  { \tiny $\pm 0.04$ }  & $49.99$  { \tiny $\pm 0.18$ }  & $\bf 0.33$  { \tiny $\pm 0.03$ } 
 \\\addlinespace[0.3em]

 \textbf{Amun}$_{+SalUn}$ & $95.02$  { \tiny $\pm 0.18$ }  & $99.58$  { \tiny $\pm 0.04$ }  & $93.29$  { \tiny $\pm 0.04$ }  & $50.72$  { \tiny $\pm 0.79$ }  & $\underline{0.68}$  { \tiny $\pm 0.18$ } 
 & 
 
 $93.56$  { \tiny $\pm 0.07$ }  & $99.72$  { \tiny $\pm 0.02$ }  & $92.52$  { \tiny $\pm 0.20$ }  & $49.81$  { \tiny $\pm 0.40$ }  & $\underline{0.36}$  { \tiny $\pm 0.07$ }
 \\\addlinespace[0.3em]

 

 
 
 \bottomrule
\end{tabular}
}
\end{sc}
\end{small}
\end{center}
\caption{\footnotesize {\bf Unlearning with access to $\Dr$.} Comparing different unlearning methods in unlearning $10\%$ and $50\%$ of $\D$. Avg. Gap (see \S~\ref{sec:amun_metrics}) is used for evaluation (lower is better). The lowest value is shown in bold while the second best is specified with underscore. As the results show, \amun{} outperforms all other methods by achieving lowest Avg. Gap and \amun{}$_{SalUn}$ achieves comparable results.\vspace{-1mm}}
\label{tab:rmia}

\end{table*}

\begin{table*}[th!]
\begin{center}
\begin{small}
\begin{sc}
\resizebox{\textwidth}{!}{
\begin{tabular}{@{} l  c c  c  c c | c c c c c @{}}
 \toprule


& \multicolumn{5}{@{}c}{\textbf{Random Forget ($10\%$)}} & \multicolumn{5}{@{}c}{\textbf{Random Forget ($50\%$)}} \\\addlinespace[0.3em]

 \multicolumn{1}{c}{\scriptsize \textbf{}} & 
 \multicolumn{1}{c}{\scriptsize Unlearn Acc} & 
 \multicolumn{1}{c}{\scriptsize Retain Acc} & 
 \multicolumn{1}{c}{\scriptsize Test Acc} & 
 \multicolumn{1}{c}{\scriptsize FT AUC} & 
 \multicolumn{1}{c}{\scriptsize Avg. Gap} & 
 \multicolumn{1}{c}{\scriptsize Unlearn Acc} &
 \multicolumn{1}{c}{\scriptsize Retain Acc} & 
 \multicolumn{1}{c}{\scriptsize Test Acc} & 
 \multicolumn{1}{c}{\scriptsize  FT AUC} & 
 \multicolumn{1}{c}{\scriptsize Avg. Gap} 
 \\\addlinespace[0.3em]

 \cmidrule(r){2-6}
 \cmidrule(r){7-11}



 

 Retrain & $94.49$ {\tiny $\pm 0.20$} & $100.0 $ {\tiny $\pm 0.00$} & $94.33$ {\tiny $\pm 0.18$} & $50.00$ {\tiny $\pm 0.42$} & $0.00$  & 
 
 $92.09$ {\tiny $\pm 0.37$}  & $100.0$ {\tiny $\pm 0.00$} & $91.85$ {\tiny $\pm 0.33$} & $50.01$ {\tiny $\pm 0.12$} & $0.00$
 
 \\\addlinespace[0.3em]
  \cmidrule(r){1-11}

 RL & $100.00$  { \tiny $\pm 0.00$ }  & $100.00$  { \tiny $\pm 0.00$ }  & $94.45$  { \tiny $\pm 0.09$ }  & $61.85$  { \tiny $\pm 0.25$ }  & $4.31$  { \tiny $\pm 0.06$ }  
 & 
 
 $100.00$  { \tiny $\pm 0.00$ }  & $100.00$  { \tiny $\pm 0.00$ }  & $94.57$  { \tiny $\pm 0.14$ }  & $61.99$  { \tiny $\pm 0.10$ }  & $4.29$  { \tiny $\pm 0.03$ }  
 \\\addlinespace[0.3em]

  GA & $4.77$  { \tiny $\pm 3.20$ }  & $5.07$  { \tiny $\pm 3.54$ }  & $5.09$  { \tiny $\pm 3.38$ }  & $49.78$  { \tiny $\pm 0.34$ }  & $68.53$  { \tiny $\pm 2.45$ } 
  &
 
  $100.00$  { \tiny $\pm 0.00$ }  & $100.00$  { \tiny $\pm 0.00$ }  & $92.65$  { \tiny $\pm 0.09$ }  & $63.41$  { \tiny $\pm 0.24$ }  & $5.13$  { \tiny $\pm 0.04$ } 
 \\\addlinespace[0.3em]

 BS & $100.00$  { \tiny $\pm 0.00$ }  & $100.00$  { \tiny $\pm 0.00$ }  & $94.48$  { \tiny $\pm 0.04$ }  & $61.41$  { \tiny $\pm 0.29$ }  & $4.20$  { \tiny $\pm 0.07$ } 
 & 
  
 $100.00$  { \tiny $\pm 0.00$ }  & $100.00$  { \tiny $\pm 0.00$ }  & $94.58$  { \tiny $\pm 0.08$ }  & $62.43$  { \tiny $\pm 0.14$ }  & $4.40$  { \tiny $\pm 0.05$ } 
 \\\addlinespace[0.3em]

  SalUn  & $100.00$  { \tiny $\pm 0.00$ }  & $100.00$  { \tiny $\pm 0.00$ }  & $94.47$  { \tiny $\pm 0.10$ }  & $61.09$  { \tiny $\pm 0.40$ }  & $4.11$  { \tiny $\pm 0.09$ } 
  & 

   $100.00$  { \tiny $\pm 0.00$ }  & $100.00$  { \tiny $\pm 0.00$ }  & $94.59$  { \tiny $\pm 0.12$ }  & $62.45$  { \tiny $\pm 0.37$ }  & $4.40$  { \tiny $\pm 0.07$ } 
 \\\addlinespace[0.5em]

 \textbf{Amun} & $94.28$  { \tiny $\pm 0.37$ }  & $97.47$  { \tiny $\pm 0.10$ }  & $91.67$  { \tiny $\pm 0.04$ }  & $52.24$  { \tiny $\pm 0.23$ }  & $\underline{1.94}$  { \tiny $\pm 0.13$ } 
 & 
 
 $92.77$  { \tiny $\pm 0.52$ }  & $95.66$  { \tiny $\pm 0.25$ }  & $89.43$  { \tiny $\pm 0.19$ }  & $52.60$  { \tiny $\pm 0.22$ }  & $\underline{2.51}$  { \tiny $\pm 0.09$ } 
 \\\addlinespace[0.3em]

 \textbf{Amun}$_{+SalUn}$ & $94.19$  { \tiny $\pm 0.38$ }  & $97.71$  { \tiny $\pm 0.06$ }  & $91.79$  { \tiny $\pm 0.12$ }  & $51.93$  { \tiny $\pm 0.12$ }  & $\bf 1.77$  { \tiny $\pm 0.06$ }
 & 
 
 $91.90$  { \tiny $\pm 0.63$ }  & $96.59$  { \tiny $\pm 0.31$ }  & $89.98$  { \tiny $\pm 0.44$ }  & $52.32$  { \tiny $\pm 0.56$ }  & $\bf 2.00$  { \tiny $\pm 0.17$ }
 \\\addlinespace[0.3em]

 \bottomrule
\end{tabular}
}
\end{sc}
\end{small}
\end{center}
\caption{\footnotesize {\bf Unlearning with access to only $\Df$.}  Comparing different unlearning methods in unlearning $10\%$ and $50\%$ of $\D$. Avg. Gap (see \S~\ref{sec:amun_metrics}) is used for evaluation (lower is better) when only $\Df$ is available during unlearning. As the results show, \amun{}$_{SalUn}$ significantly outperforms all other methods, and \amun{} achieves comparable results. \vspace{-3mm}}
\label{tab:rmia_forgetonly}

\end{table*}

\subsection{Ablation Studies}
\label{sec:amun_ablation}

In this subsection, we first elaborate on the effect of fine-tuning a model on its adversarial examples and compare it to the cases where either the samples or labels of this dataset change (\S~\ref{sec:ablation-finetune}). We then discuss \amun{}'s efficacy on models that are already robust to adversarial examples (\S~\ref{sec:ablation-robust}). We present other ablation studies on using weaker, but faster, adversarial attacks in Algorithm~\ref{alg:advset} (Appendix~\ref{apx:weak_attack}). In Appendix~\ref{apx:transfer_attack}, we utilize transferred adversarial examples for unlearning, as this can expedite handling the unlearning from a newly trained model for which adversarial examples on similar architectures are available.

\subsubsection{Fine-tuning on Adversarial Examples}
\label{sec:ablation-finetune}

We want to verify the importance of $\Da$ (created by Algorithm~\ref{alg:advset}) in preserving the model's test accuracy. To this end, we build multiple other sets to be used instead of $\Da$ when fine-tuning. Let us assume that $\mathcal{A}_\mathcal{F}(x,y) = (x_{adv},y_{adv})$. Then, these other sets are:

\begin{itemize}
\vspace{-3mm}
\itemsep0em
    \item $\D_{AdvL}$: $\{(x, y_{adv})\}_{\forall (x,y) \in \Df}$
    \item $\D_{RL}$ : $\{(x, y^\prime)$, s.t. $y^\prime \neq y, y_{adv}\}_{\forall (x,y) \in \Df}$
    \item $\Da_{RL}$: $\{(x_{adv}, y^\prime) \text{, s.t.} y^\prime \neq y, y_{adv}\}_{\forall (x,y) \in \Df}$
    \item $\Da_{RS}$: $\{(x^\prime, y_{adv})$, s.t. $x^\prime \sim \mathrm{Uniform(X_\delta)} $, where $X_\delta = \{\forall \hat{x}: \|x_\delta - x\|_2 = \delta\}\}_{\forall (x,y) \in \Df}$
\end{itemize}
\vspace{-3mm}

In this experiment, we evaluate the effect of fine-tuning on test accuracy of a ResNet-18 model that is trained on CIFAR-10, when $\Da$ is substituted with other datasets that vary in the choice of samples or their labels. We assume that $\Df$ contains $10\%$ of the samples in $\D$ and $\Da$ is the set of corresponding adversarial examples constructed using Algorithm~\ref{alg:advset}. 

\noindent{\bf Results:} In Fig.~\ref{fig:fine_tune_10}, \texttt{Adv}, from the left sub-figure to the right sub-figure, shows the results when $\D \cup \Da$, $\Df \cup \Da$, and $\Da$ is used for fine-tuning the model, respectively. \texttt{Orig}, \texttt{Adv-RS}, \texttt{Adv-RL}, \texttt{Orig-RL}, and \texttt{Orig-AdvL} show the results when $\Da$ for each of these sub-figures is replaced by $\Df$, $\Da_{RS}$, $\Da_{RL}$, $\D_{RL}$, and $\D_{AdvL}$, respectively. As the figure shows, the specific use of adversarial examples with the mispredicted labels matters in keeping the model's test accuracy, especially as we move from the leftmost sub-figure (having access to $\Dr$) to the right one (only using $\Da$ or its substitutes). This is due to the fact that the samples in $\Da$, in contrast to the other constructed datasets, belong to the natural distribution learned by the trained model. Therefore, even if we only fine-tune the ResNet-18 model on $\Da$, we still do not lose much in terms of model's accuracy on $\Dt$. This is a surprising observation, as $\Da$ contains a set of samples with wrong predictions! Fig.~\ref{fig:fine_tune_50} in Appendix~\ref{apx:ablation-finetune} shows similar results when size of $\Df$ is $50\%$ of $|\D|$.

\subsection{Adversarially Robust Models}
\label{sec:ablation-robust}

We evaluate the effectiveness of \amun{} when the trained model is adversarially robust. 
One of the most effective methods in designing robust models is adversarial training which targets smoothing the model's prediction function around the training samples~\cite{salman2019provably}. This has been shown to provably enhance the adversarial robustness of the model~\cite{cohen2019certified}. However, this method is very costly and impractical for larger models. There is a separate line of work that try to achieve this smoothness by controlling the Lipschitz constant of the models~\cite{szegedy2013intriguing}. The method proposed by~\cite{boroojeny2024spectrum} is much faster than adversarial training and their results show a significant improvement in the robust accuracy. Since evaluations with RMIA require training 128 reference models, we use their clipping method to evaluate the effectiveness of \amun{} for unlearning $10\%$ and $50\%$ of the samples from robust ResNet-18 models trained on CIFAR-10. 

\noindent{\bf Results:} As shown in Table~\ref{tab:rmia-clipped-remFalse}, even when there is no access to $\Dr$, \amun{} still results in effective unlearning for these models; RMIA does not do better than random guessing for predicting $\Df$ from $\Dt$. Table~\ref{tab:rmia-clipped-remTrue} in Appendix~\ref{apx:ablation_robust} shows similar results for the case with access to $\Dr$. As Fig.~\ref{fig:retrain_conf} (right) shows, in the robust models, more than $97\%$ of the adversarial examples are further away from their corresponding training samples, compared to this distance for the original models. However, this does not interfere with the performance of \amun{} because these robust models are smoother and tend to be more regularized. This regularization, which prevents them from overfitting to the training samples is in fact a contributing factor to the improved generalization bounds for these models~\cite{bartlett2017spectrally}. This in itself contributes to enhanced resilience against MIAs. As seen in Tables~\ref{tab:rmia-clipped-remFalse} and~\ref{tab:rmia-clipped-remTrue}, even for the clipped models retrained on $\Dr$, the AUC score of RMIA for predicting $\Dr$ from $\Df$ (\texttt{FR AUC}) is very low, which shows that these smoother models .

\begin{table}[th!]
\begin{center}
\begin{small}
\begin{sc}
\resizebox{0.76\columnwidth}{!}{
\begin{tabular}{@{} l  c c c | c c c @{}}
 \toprule

& \multicolumn{3}{@{}c}{\textbf{Random Forget ($10\%$)}} & \multicolumn{3}{@{}c}{\textbf{Random Forget ($50\%$)}} \\\addlinespace[0.3em]

 \multicolumn{1}{c}{\scriptsize \textbf{}} & 
 \multicolumn{1}{c}{\scriptsize FT AUC} & 
 \multicolumn{1}{c}{\scriptsize FR AUC} & 
 \multicolumn{1}{c}{\scriptsize Test Acc} & 
 \multicolumn{1}{c}{\scriptsize FT AUC} &
 \multicolumn{1}{c}{\scriptsize FR AUC} & 
 \multicolumn{1}{c}{\scriptsize Test Acc} 
 \\\addlinespace[0.3em]

 \cmidrule(r){2-4}
 \cmidrule(r){5-7}

 Retrain & $49.95$ {\tiny $\pm 0.24$} & $54.08 $ {\tiny $\pm 0.16$} & $89.01$ {\tiny $\pm 0.21$} & 
 
 $50.19$ {\tiny $\pm 0.15$}  & $55.61$ {\tiny $\pm 0.05$} & $85.76$ {\tiny $\pm 0.41$}
 
 \\\addlinespace[0.3em]
  \cmidrule(r){1-7}

 \textbf{Amun} & $49.55$ {\tiny $\pm 0.13$} & $54.01$ {\tiny $\pm 0.23$} & $87.55$ {\tiny $\pm 0.44$} & 
 
 $49.64$ {\tiny $\pm 0.31$} & $53.23$ {\tiny $\pm 0.21$} & $87.39$ {\tiny $\pm 0.61$}
 \\\addlinespace[0.3em]
               
 \bottomrule
\end{tabular}
}
\end{sc}
\end{small}
\end{center}
\caption{\footnotesize {\bf Unlearning on adversarially robust models.} Evaluating the effectiveness of \amun{} in unlearning $10\%$ and $50\%$ of the training samples when the models are adversarially robust and there is no access to $\Dr$. For this experiment we use models with controlled Lipschitz constant which makes them provably and empirically more robust to adversarial examples. \vspace{-3mm}}
\label{tab:rmia-clipped-remFalse}

\end{table}

\subsection{Continuous Unlearning}
\label{sec:adaptive}

We evaluate the performance of the unlearning methods when they are used to perform multiple consecutive unlearning from a trained model. This is a desirable capability for unlearning methods because in real-world applications there might be multiple unlearning requests and it is preferred to minimize the number of times that a model needs to be retrained from scratch. 
The setting we envision is as follows: models are updated at each request for unlearning. For \amun{}, this means that $\Da$ is computed on an updated model after each set of unlearning requests (shown as \amun{}\texttt{-A}). 
In addition to comparing \amun{}\texttt{-A} to the other unlearning methods, we also compare it to a version (shown as \amun{}) that computes all the adversarial examples on the original model so it can handle the unlearning requests faster upon receiving them i.e., $\Da$ is not computed on an updated model after each request; the set of requests are batched and $\Da$ is computed on the entire batch. For this experiment, we use a ResNet-18 model trained on training set of CIFAR-10 ($50$K samples). Our goal is to unlearn $10\%$ of the training samples ($5$K), but this time in $5$ consecutive sets of size $2\%$ ($1$K) each. We then evaluate the effectiveness of unlearning at each step using RMIA.

\noindent{\bf Results:} Fig.~\ref{fig:adaptive} shows an overview of the results for both settings of unlearning (with or without access to $\Dr$). This figure shows the effectiveness of unlearning by depicting how the samples in $\Df$ are more similar to the test samples ($\Dt$) rather than the remaining samples ($\Dr$). The value on the y-axis shows the difference of the area under the ROC curve (AUC) for predicting $\Dr$ from $\Df$ and $\Df$ from $\Dt$. For the plots of each of these values separately, see Appendix~\ref{apx:adaptive}. 
\amun{}\texttt{-A} performs better than all the other unlearning methods for all the steps of unlearning. Although \amun{} also outperforms all the prior unlearning methods, it slightly under-performs compared to \amun{}\texttt{-A}. This is expected, as the model's decision boundary slightly changes after each unlearning request and the adversarial examples generated for the original model might not be as effective as those ones generated for the new model. Note that for this experiment, we did not perform hyper-parameter tuning for any of the unlearning methods, and used the same ones derived for unlearning $10\%$ of the dataset presented in \S~\ref{sec:amun_results}. For further discussion of the results see Appendix~\ref{apx:adaptive}.

\section{Future Work}
\label{sec:proposed-diffusion}


\cite{gandikota2023erasing} frame concept removal as a targeted fine-tuning problem: Stable Diffusion is optimized only on prompts that describe the unwanted concept, together with a KL-style regularizer that keeps all other generations close to the original model. They makes two inferences, one on samples conditioned on the unlearning concept and one unconditioned. It then uses the negative of the difference in the added noise to perform an optimization on the parameters in the reverse direction of the noise relevant to that concept. Because no extra real images are needed, ESD can delete copyrighted styles or unsafe subjects in a few hundred gradient steps while maintaining overall image quality. 
\cite{jiangunlearning} propose SGU that treats continuous-time diffusion as a stochastic differential equation and adds an unlearning drift that gradually diverts trajectories away from the undesirable data manifold. The process is flexible—compatible with both SGM and DDPM backbones—and significantly lowers the likelihood of regenerating the target content while keeping FID drops small.
\cite{maharanatowards} aim for safety-critical deployment, this work combines adversarial prompt mining with region-specific weight masking to localise where a sensitive concept lives in U-Net layers. Layer-wise low-rank edits then suppress those activations; extensive experiments show higher erasure fidelity and lower collateral damage than previous baselines across multiple concept categories.

\cite{fuchi2024erasing} fine-tune only the text encoder of a text-to-image pipeline using a handful of real images that illustrate the forbidden concept. By steering the encoder’s embeddings toward semantically distant regions, the approach implicitly redirects the generative process to benign latent concepts, achieving concept removal without degrading unrelated generations. \cite{alberti2025data} consider the problem of removing individual training images rather than concepts to propose SISS, a weighted loss that pushes the score estimates away from the forbidden data while preserving fidelity on the rest of the training distribution.~\cite{fan2023salun} propose using an idea similar to random labeling used for unlearning in classification models, to align the unlearning concept with misaligned images.

Inspired by the these prior work and their benefits and shortcoming, and utilizing the main motivations behind \amun{}, we plan to investigate the use of adversarial perturbations in the latent features at each diffusion step in the forward path, such that they lead to generating an image from a different concept in its backward path (image generation). More specifically, consider the diffusion process given by:

\vspace*{-5mm}
{\small \begin{align}
    \hat{\epsilon}_{\btheta}(\mathbf x_t | c) = (1-w) {\epsilon}_{\btheta}(\mathbf x_t | \emptyset ) + w {\epsilon}_{\btheta}(\mathbf x_t | c ),
    \label{eq: condition_diffusion}
\end{align}}%
where $\hat{\epsilon}(\mathbf x_t | c)$ represents the noise estimation derived by utilizing the conditional diffusion model given a concept $c$. $ w \in [0,1]$ is a hyper-parameter guiding the generation of the noise accoding to the concept $c$ or unconditional generation ${\epsilon}(\mathbf x_t | \emptyset )$. The inference stage initiates with Gaussian noise $z_T \sim \mathcal{N}(0, 1)$, and the backward denoised process is used to generate the data at $t = 0$ ($\hat{\epsilon}_{\btheta}(\mathbf x_T | c)$ to obtain $z_{T-1}$ and so forth)~\cite{fan2023salun}. For training this model, authentic data from concept $c$ is used to learn the parameters using an MSE loss. For unlearning,~\cite{fan2023salun} use images that do not belong to concept $c$, to fine-tune the parameters and align concept $c$ with other randomly chosen images. Instead, in our proposed methods, we use fine-tuning the trained model on the images generated using adversarial perturbations to drift the generation away from the forbidden concept, without compromising the generation quality for other topics, as shown in the results of section~\ref{sec:adv-finetune} and Figure~\ref{fig:fine_tune_10}. Also, as observed in the results for effectiveness of \amun{} in the setting where $\Dr$ is not available (see section~\ref{sec:amun_results}), we expect having access to only samples from the unlearning concept would be enough for this unlearning process (one of the main benefits of some previously introduced methods~\cite{gandikota2023erasing}). Also, this method could be beneficial for the problem of removing individual training, as in the work by~\cite{alberti2025data}, by focusing on the adversarial samples for the target images only. As the results for \amun{} shows, unlearning a subset of the samples from one class does not lead to under-perfomance for the remaining samples from that class. 

\newpage
\chapter{Class Unlearning}
\label{sec:trw_intro}

The setting of class unlearning is different from unlearning random training samples because the unlearned model has to be depleted of any sign of the unlearned class. Failing to consider this difference, we show that all the existing methods are susceptible to our newly designed variant of membership inference attacks (MIAs), which is based on the \emph{nearest neighbor} of the target class (MIA-NN). This vulnerability has remained unnoticed because previously used evaluations and metrics are mainly designed or inspired by prior work on unlearning random subsets of data, and without paying attention to the properties of a model that is retrained from scratch on all the data other than the target class.

We then propose a new training objective that enhances the robustness to MIA-NN, without computational overhead over finetuning-based unlearning methods. When a class is marked for removal, this method first reassigns its predicted probability mass proportionally to the remaining classes, and then tilts this distribution according to inter-class similarities to better approximate the distribution of the models retrained from scratch on the remaining classes. Our proposed method not only exhibits strong robustness against MIA-NN, it also does not fail the evaluations using prior evaluation metrics.

To evaluate our proposed approaches, we conduct comprehensive experiments on MNIST, CIFAR-10, CIFAR-100, and Tiny-ImageNet, using ten state-of-the-art unlearning methods. In addition to common evaluations in prior work MIA~\cite{hayes2024inexact}, we show robustness against stronger MIA (U-LiRA~\cite{hayes2024inexact}). Additionally, we present the results for robustness against \textbf{MIA-NN}, which we specifically designed to reveal residual leakage of existing methods missed by standard evaluations. 
Our results show that simple output-space interventions effectively approximates the distribution of retrained models on samples from the forget class without significant computational overhead. Our contributions can be summarized as follows:

\begin{itemize}
    
    \item \textbf{Membership inference attack via nearest neighbors.}
    We propose a new membership inference attack, MIA-NN, that utilizes the probabilities assigned to the nearest-neighbor class of the forget class to detect whether its samples have been used for training the model. In contrast to prior work, which primarily relies on simple accuracy drop or binary membership prediction used for random forgetting, our metrics offer deeper insights into how and where leakage persists. MIA-NN reveals vulnerabilities in existing baselines under stronger adversarial scrutiny, setting a new benchmark for effective and privacy-preserving unlearning. Notably, MIA-NN does not assume any access to the training data, which is a more restrictive and realistic setting for adversarial attacks.

    \item \textbf{Lightweight unlearning loss modification for output reweighting.} 
    We propose Tilted ReWeighting (TRW), a simple yet effective modification to the objective of fine-tuning that removes the influence of a forgotten class by approximating the distribution of scratch-retrained models. Unlike prior methods, which fail to replicate the fine-grained behavior of the retrained models, TRW is more robust against MIA-NN and other existing attacks. 
    We accomplish this by applying a post-hoc probability reassignment that adjusts the model’s decision boundaries for retained classes. 
    The proposed modification achieves superior performance to existing methods while remaining computationally efficient, making it suitable for deployment in real-world systems that demand rapid unlearning.

\end{itemize}
\section{Methods}
\label{sec:trw_methods}

To build a precise understanding of unlearning objectives and challenges, we first formalize the problem in~\cref{sec:defn_unlearn}.  In~\cref{subsec:trw_motivation}, we analyze how a retrained model behaves on forgotten samples—an aspect largely overlooked in prior work—and use this insight to design a stronger variant of MIA in section~\ref{subsec:guiding_mia} that reveals vulnerabilities in existing unlearning methods. Motivated by this observation, in section~\ref{sec:our_method} we introduce a unlearning method that mitigates this shortcoming and enhances the effectiveness of unlearning.

\subsection{Problem Definition: Class Unlearning}
\label{sec:defn_unlearn}

Let $\mathcal{D} = \{(\mathbf{x}_i, y_i)\}_{i=1}^N$ be the full training set with the set of labels $\mathcal{Y}$ containing $K$ different classes. and let $y_f \in \mathcal{Y}$ denote the class to forget. The set of samples to forget is defined as $\mathcal{D}_f = {(\mathbf{x}_i, y_i) \subset \mathcal{D} \mid_{y_i = y_f}}$, and the retained set is $\mathcal{D}_r = \mathcal{D} \setminus \mathcal{D}_f$.
The goal of class unlearning is to produce a model that behaves as if trained only on $\mathcal{D}_r$, i.e., with no influence from $\mathcal{D}_f$.

\subsection{Motivation}
\label{subsec:trw_motivation}

Models retrained on $\mathcal{D}_r$ often assign skewed predictions to $\mathcal{D}_f$ based on semantic similarity. For example, see~\cref{fig:confusion mat 2} that shows the prediction on a ResNet18 model trained on CIFAR-10. The \emph{Retrain} model (a model trained on $D_r$) that has never seen \texttt{automobile} samples tends to misclassify them as \texttt{truck}, which is visually and semantically similar. This behavior is not specific to model architecture but emerges from underlying data-level class similarities. This behavior has been overlooked by prior evaluation methods in class unlearning literature which focus only on the logit or probability for only the forget class.

\begin{figure}[t]
\centering
\includegraphics[width=0.95\linewidth]{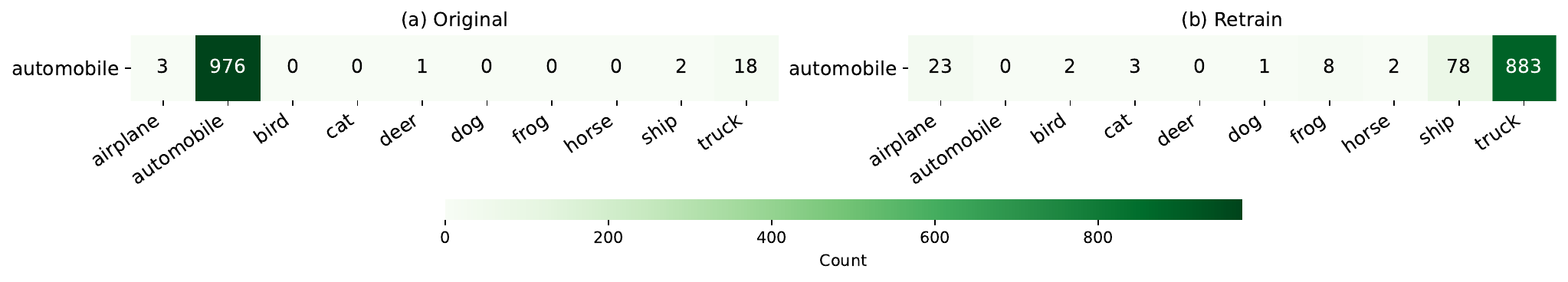}
\caption{The predicted labels and their corresponding counts for samples belonging to the \texttt{automobile} class for the original model (left) and the retrained model (right). Retrain model's predictions on the forget class is skewed toward similar classes.}
\label{fig:confusion mat 2}
\end{figure}

\begin{tcolorbox}
\noindent \textbf{Key Insight:} {\em The Retrain models exhibit structured misclassifications for forgotten classes—typically to semantically close retained classes. Approximate methods that do not replicate this behavior are susceptible to stronger privacy attacks.}
\end{tcolorbox}

\subsection{MIA on the Nearest Neighbor (MIA-NN)}
\label{subsec:guiding_mia}
Motivated by the insight in section~\ref{subsec:trw_motivation}, we propose \emph{Membership Inference Attack via Nearest Neighbor} \textbf{MIA-NN}, that utilizes the probabilities assigned to the class closest to the forget class. We denote by $\mathcal{M}_o$ the \emph{target} model, by $\mathcal{M}_u$ the \emph{unlearned} model, and by $\mathcal{M}_t$ the \emph{retrained} model. Suppose we have $n$ Retrain models: $\mathcal{M}_{t}^{[n]} = \{ \mathcal{M}_t^{(1)}, \mathcal{M}_t^{(2)}, \dots, \mathcal{M}_t^{(n)} \}.$
Let $\mathcal{R} = \{ r_1, r_2, \dots, r_{K-1} \}$ be the set of remaining classes, and let $y_f$ denote the forget class.  
For each class $r_i$, the test set can be split to the following disjoint subsets:
$D_{r_i\text{-test}}$ (i.e., test samples belonging to class $r_i$), $D_{f\text{-test}}$ (i.e., test samples belonging to the forget class), and $D_{r_{\hat{i}}\text{-test}}$ where $r_{\hat{i}}=r_j \in \mathcal{R}, j\neq i$ (i.e., the remaining of the test samples). 

For a model $\mathcal{M}$ and class $r_i$, let $z_{\mathcal{M}}(x, r_i)$ denote the logit value corresponding to class $r_i$ when sample $x$ is given as input. Now, for each remaining class $r_i$, and for each Retrain model $\mathcal{M}_t^{(j)}$, we construct the training data
\[
\mathcal{T}_{i}^{(j)} =
\left \{ \left( z_{\mathcal{M}_t^{(j)}}(x, r_i), 1\right) \mid x \in D_{r_i\text{-test}} \right \}
\cup
\left \{ \left (z_{\mathcal{M}_t^{(j)}}(x, r_i), 0 \right) \mid x \in D_{r_{\hat{i}}\text{-test}} \right \}.
\]

A binary classifier (e.g., SVM) $h_{i}^{(j)} : \mathbb{R} \to \{0,1\}$ is then trained to discriminate $D_{r_i\text{-test}}$ vs. $D_{r_{\hat{i}}\text{-test}}$. The accuracy of this classifier on the forget-class test data is defined as
\[
\text{Acc}_{i}^{(j)} = \frac{1}{|D_{f\text{-test}}|} \sum_{x \in D_{f\text{-test}}}
\mathbf{1}\!\left( h_{i}^{(j)}(z_{\mathcal{M}_t^{(j)}}(x, r_i)) = 1 \right).
\]

For each class $r_i$, we compute the mean accuracy across all the Retrain models to derive: $\bar{\text{Acc}}_{i} = \frac{1}{n} \sum_{j=1}^{n} \text{Acc}_{i}^{(j)}$. Finally, we select the \emph{nearest neighbor} class as:
\[
r_n \,\coloneq \, \arg\max_{r_i \in \mathcal{R}} \bar{\text{Acc}}_{i}.
\]


To evaluate the unlearned model $\mathcal{M}_u$, we compute
\[
\text{Acc}_{r_n}^{\mathcal{M}_u} = \frac{1}{|D_{f\text{-test}}|} \sum_{x \in D_{f\text{-test}}}
\mathbf{1}\!\left( h_{r_n}^{\mathcal{M}_u}(z_{\mathcal{M}_u}(x, r_n)) = 1 \right),
\]
where $h_{r_n}^{\mathcal{M}_u}$ is trained the same way as the previous classifiers but using logits from $\mathcal{M}_u$ on $D_{r_{n}\text{-test}}$ and $D_{r_{\hat{n}}\text{-test}}$. We can then consider the gap between $\text{Acc}_{r_n}^{\mathcal{M}_u}$ and $\bar{\text{Acc}}_{r_n}$ as a measure of unlearning effectiveness for model $\mathcal{M}_u$.

Table~\ref{tab:miann_all} shows the the value of $\bar{\text{Acc}}_{i}$ (the \texttt{Retrain} column) when forgetting a certain class from a ResNet18 model trained on MNIST, CIFAR-10 and CIFAR-100. As the results show, there is a large gap between $\bar{\text{Acc}}_{i}$ and  $\text{Acc}_{r_n}^{\mathcal{M}_u}$, especially for the more recent unlearning methods that try to optimize for regular MIA score (see Table~\ref{tab:main_results}). This reveals a major shortcoming of the current evaluation methods for unlearning and the need for complementary methods such as MIA-NN that perform a more comprehensive analysis of the behavior of unlearned models rather than only focusing on the logit value for the forget class.

\begin{table}[htbp]
  \centering
  \resizebox{\textwidth}{!}{%
  \begin{tabular}{l|l|llllllllllll}
    \toprule
    Dataset   & Retrain & FT   & RL   & GA   & SalUn & BU   & l1   & SVD  & SCRUB & SCAR & l2ul & TRW \\
    \midrule
    MNIST (8 $\rightarrow$ 3)  & 73.9 & 58.2 & 45.6 & 18.5 & 6.1 & 9.6 & 4.7 & 48.3 & 7.7 & 42.8 & 31.1 & \textbf{70.5} \\
    CIFAR-10 (auto$\rightarrow$ truck)  & 89.6 & 75.7 & 34.3 & 18.1 & 5.2 & 19.8 & 9.1  & 50.0 & 8.8  & 51.9 & 20.4 & \textbf{82.1} \\
    CIFAR-100 (beaver $\rightarrow$ shrew)  & 75.0 & 54.0 & 8.9  & 13.5 & 4.0 & 12.5 & 3.9  & 39.6 & 6.7  & 43.1 & 15.8 & \textbf{69.5} \\
    \bottomrule
  \end{tabular}}
  \vspace{2mm}
  \caption{MIA-NN accuracy on the samples from forgotten classes across datasets. Higher values indicate better unlearning (harder to infer membership) based on neighboring classes. The gap with the Retrain models reveals under-performance in many of the SOTA unlearning methods that have been evaluated using only regular MIAs.}
  \label{tab:miann_all}
\end{table}

\begin{tcolorbox}
\noindent \textbf{Main Takeaway:} {\em Existing SOTA methods leak membership under variants of MIAs that probe the fine-grained behavior of the model, such as MIA-NN.}
\end{tcolorbox}

\subsection{Tilted ReWeighting (TRW)}
\label{sec:our_method}

As mentioned in prior sections, the susceptibility of existing unlearning methods to MIA-NN arises from the fact that they fail to mimic the fine-grained behavior of the Retrain models on samples from the forget class. A basic solution to enforce this similarity is to utilize the probabilities that the original model assigns to other classes when predicting on the forget samples. That provides us with an estimate of how the probabilities should redistribute when the forget label is enforced to be zero. More specifically, let $p(y \mid \mathbf{x})$ be the original model’s output distribution.
We perform a rescaling to remove the forget class $y_{f}$ and derive the reweighted distribution on the remaining classes:

\begin{align}
\label{equ:reweighted}
    \tilde p(y\mid x)\;=\;\frac{p(y\mid x)}{1-p(y_f\mid x)} \quad (y\neq y_f),\qquad \tilde p(y_f\mid x)=0,
\end{align}

Using $\tilde p$ as the target distribution when fine-tuning on the samples of the forget class, enforces zero probability for the forget label and rescales the probability for other labels to sum up to $1$. However, this assumes that the probabilities of other classes would increase proportionally in the absence of the forget class. As figures~\ref{fig:score_c1} and~\ref{fig:score_c6} in~\Cref{apx:tilt} show, this assumption does not hold and the Retrain models will have a much higher bias when predicting on the forget samples. This systematic bias arises from the fact that the forget class has different levels of similarity to other classes, and a model that has never seen the forget class, would assign higher probabilities to more similar classes.

To capture the systematic bias toward more similar classes, while adding minimal additional constraint to the our target distribution, we impose a first-moment constraint given a set of similarity scores between the forget class and the remaining classes. More specifically, let $s_y,\,y\neq y_f$ show the similarity between class $y$ and $y_f$. Then the set of probability functions over the remaining classes that satisfy this constraint is defined as follows:

\[
\mathcal A\ :=\ \Big\{q\in\Delta^{K-1}:\ \sum_{y\neq f} q(y\mid x)\,s_y=c\Big\}.
\]

Now we need a candidate $q^*(y|x) \in \mathcal{A}$, such that it remains close to the original distribution $p(y|x)$ (and consequently to $\tilde p$). In Proposition~\ref{lem:iproj-tilt} we show that our desired distribution will be a \textit{tilted} version of $\tilde p(y|x)$ using the score values and has the following form:

\begin{align}
\label{equ:tilted}
    q^*(y\mid x)
    \;=\;
    \frac{\tilde p(y\mid x)\,\exp(\beta\,s_y)}
    {\sum_{j\neq y_f}\tilde p(j\mid x)\,\exp(\beta\,s_j)},
    \qquad
    q^*(y_f\mid x)=0.
\end{align}

Note that $\beta\in\mathbb{R}$ controls the strength of the tilt. We do not know the exact value of $c$ a-priori to set it in the constraints, but the hyperparameter $\beta$ helps us achieve various levels of $c$. 
When $\beta=0$, the tilted distribution $q^*$ reduces to the plain renormalized $\tilde p$, and larger $\beta>0$ redistributes more of the forget mass to classes similar to $y_f$.

The intuition behind tilted reweighting is shown in Figure~\ref{fig:decision_boundary}. The original model has learned decision boundaries to separate class $B$ from the other two classes (first figure from the left). For a model that has been retrained on class $A$ and $C$, the samples from class $B$ will mostly fall within class $A$ and will be assigned to class $A$ with a much higher certainty (second figure). Many of the samples of class $B$ in the original model have a higher probability of being assigned to class $C$ than to class $A$, so a basic rescaling of the conditional probabilities will lead to a decision boundary different from that of the retrained model (third figure). However, by using a tilting function that uses the inverse of the Euclidean distances among the class centroids as the similarity measure, we can recover the decision boundary of the retrained model (code available in the supplementary material).

\begin{figure}[t]
\centering
\includegraphics[width=0.99\linewidth]{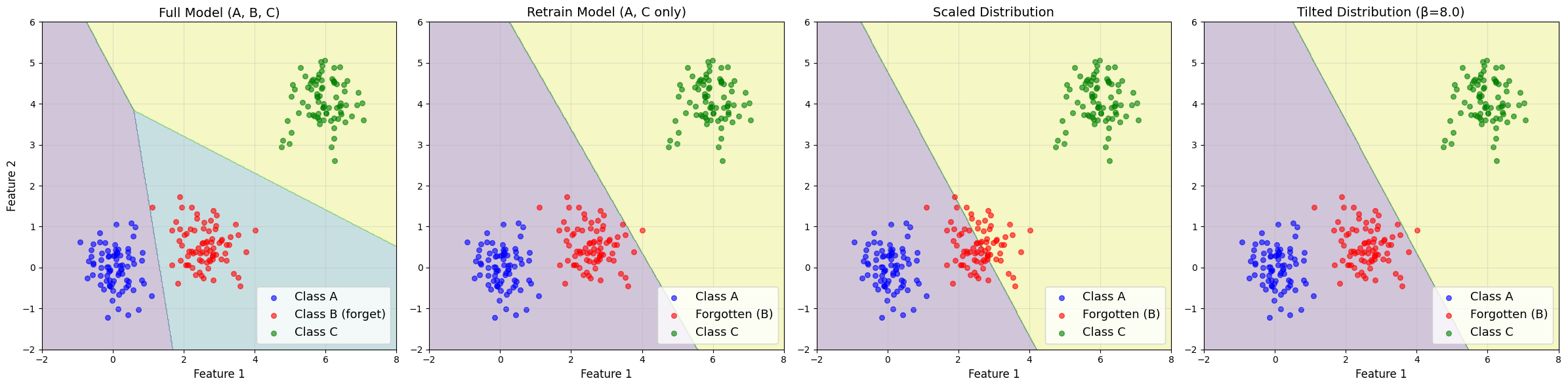}
\caption{The first figure (from the left) shows the decision boundaries when class B exists. In the Retrain model, $p(y_A|x)/p(y_C|x)$ would mostly increase for $x\in B$ due to the higher similarity of class A and class B (second figure). Third figure shows the decision boundary for a basic rescaling of original model's distribution, while the fourth one shows the tilted distribution, which correctly predicts the decision boundary in the Retrain model.}
\label{fig:decision_boundary}
\end{figure}

\noindent\textbf{Theoretical results.}
In Proposition~\ref{lem:iproj-tilt}, we show that~\eqref{equ:tilted} is equivalent to an \emph{information projection} of the original model’s distribution onto the retained simplex under an additional linear constraint on expected similarity (i.e., $\mathcal{A}$).
In other words, $q^{*}$ is the maximum-entropy distribution on the remaining classes that (i) remains close to the baseline $p$ and (ii) while satisfying the bias according to class similarities $s_y$ (see~\Cref{apx:proofs} for the proof). 

\begin{proposition}
\label{lem:iproj-tilt}
Let $p(\cdot\mid x)\in\Delta^{K}$ be the distribution of the target model for input $x$,
and let $S=\{s_y\}_{y\neq f}\subset\mathbb{R}$ be fixed similarity scores with $y_f$.
Given $c\in\mathbb{R}$ in the convex hull of $S$, the information projection of $p$ onto the probability simplex of retrained classes (i.e., $q(\cdot \mid x) \in \Delta^{K-1}$) with linear constraint  $\sum_{y\neq f} q(y)\,s_y=c$, has form of~\eqref{equ:tilted}, where $\beta$ is some unique scalar such that $\sum_{y} q^\star(y\mid x)\,s_y=c$.
\end{proposition}

\begin{remark}
If $c=\max s_y$ or $c=\min s_y$, the solution concentrates all mass on the maximizing or minimizing classes (corresponding to $\beta^\star=\pm\infty$).
Besides, adding a constant shift to all $s_y$ does not change $q^\star$. We set $\beta=10$ in all experiments to bias the distribution toward more similar classes (see section~\ref{sec:ablation} for an ablation study on the value of $\beta$). 
\end{remark}

\noindent\textbf{Score function.} For the similarity scores we use cosine similarity of the weight vectors corresponding to the logits of each class (see section~\ref{apx:score} for details). We have evaluated other similarity scores based on the distances in the embedding space derived from the target model, but they do not reach the same performance.We leave designing improved similarity functions that better approximate Retrain model behavior to future work.

\noindent\textbf{Updated loss function.}
 Now that we have an approximate distribution for how the Retrain model behaves on samples from the forget class, we can utilize it in our loss function for fine-tuning the model. More specifically, we fine-tune the model to minimize the following cross-entropy loss on samples from the forget class:
\vspace{-4pt}

\begin{equation}
\mathcal{L}_{\mathrm{forget}}(\mathbf{x}) = -\sum_{y=1}^{K} q^*(y \mid \mathbf{x})
\log p(y \mid \mathbf{x}).
\end{equation}

Therefore, our new objective, Tilted ReWeighting (TRW) loss, can be formulated as:

\begin{equation}
\label{eq:forget-objective}
    \min_{W} 
    \ \underbrace{\sum_{(x,y)\in \mathcal{D}_{\mathrm{retain}}} \big[ -\log p(y\mid x) \big]}_{\text{supervised loss on retained-class samples}}
    \quad + \underbrace{\sum_{x\in \mathcal{D}_{f}} \big[ \mathcal{L}_{\mathrm{forget}}(x) \big]}_{\text{reweight term on forget samples}}
\end{equation}

Note that the tilted reweighting loss is a rather general prescription which could be utilized in many of the prior unlearning methods that perform fine-tuning on the target model. For example, some of the prior methods in unlearning focus on sparsification of the parameters that get updated during fine-tuning~\cite{jia2023model,fan2023salun}. Although we perform a thorough comparison to these methods, we leave further evaluations on the combination of these methods with TRW to future work.

\section{Evaluation Setup}
\label{sec:trw_exp}


In section~\ref{subsec:trw_datasets}, we elaborate on the datasets and model architectures used in our experiments. In section~\ref{sec:baseline main} the baseline methods are explained, and in section~\ref{subsec:trw_metrics} we provide the details on the evaluations metrics used for comparison.
\vspace{-5pt}

\subsection{Datasets and Models}
\label{subsec:trw_datasets}
\noindent\textbf{Datasets.} To evaluate the effectiveness of our method, we use four image datasets: ~\textbf{\textsc{MNIST}}~\cite{deng2012mnist}, \textbf{\textsc{CiFar}-10, \textsc{CiFar}-100}~\cite{krizhevsky2009learning}, and \textbf{\textsc{Tiny-ImageNet}}~\cite{le2015tiny}. For single-class forgetting, we report results averaged over different target classes for each dataset. For \textbf{MNIST} and \textbf{CIFAR-10}, the results are averaged across all 10 classes. For \textbf{CIFAR-100}, we compute the average over unlearning experiments conducted on the following $10$ randomly selected classes: \textit{apple}, \textit{aquarium fish}, \textit{baby}, \textit{bear}, \textit{beaver}, \textit{bed}, \textit{bee}, \textit{beetle}, \textit{bicycle}, and \textit{bottle}. 
We use average results on multiple Retrain models as an ideal reference, and assess unlearning methods by how closely they match its performance.

\noindent\textbf{Models.} For \textsc{MNIST} and \textsc{CIFAR}, we use \textsc{ResNet18}~\cite{he2016deep} and \textsc{VGG19}~\cite{simonyan2014very} as the original model and for \textsc{Tiny-ImageNet} we use the pretrained \textsc{ResNet18}. Full training and hyperparameter configurations are detailed in~\cref{sec:exp setting}.

\vspace{-5pt}
\subsection{Baseline Methods}
\label{sec:baseline main}

We compare our method against ten of state-of-the-art machine unlearning baselines, including fine-tuning (FT)~\cite{warnecke2021machine}, random labeling (RL)~\cite{golatkar2020eternal}, gradient ascent (GA)~\cite{thudi2022unrolling}, sparse retraining (l1)~\cite{jia2023model}, boundary unlearning (BU)~\cite{chen2023boundary}, saliency-guided unlearning (SalUn)~\cite{fan2023salun}, SVD-based feature suppression (SVD)~\cite{kodge2024deep}, SCRUB~\cite{kurmanji2023towards}, L2UL~\cite{cha2024learning} and SCAR~\cite{bonato2024retain}. The Retrain models, are considered as the gold standard. Detailed descriptions of all baselines and their hyperparameter configurations are detailed in~\cref{sec:baselines} and~\cref{sec:exp setting}, accordingly.
\vspace{-5pt}

\subsection{Evaluation Metrics}
\label{subsec:trw_metrics}
In section~\ref{sec:main results}, following the evaluation metrics used in prior work, we evaluate the unlearned models using three metrics: the accuracy on the remaining set $\text{ACC}_r$, the accuracy on the forgetting set $\text{ACC}_f$, and the Membership Inference Attack (MIA) score~\cite{shokri2017membership}. We applied the MIA score used by~\cite{kodge2024deep}. 

Ideally, the MIA score of the unlearned model is expected to match that of the retraining model. 
To perform a comprehensive comparison of the effectiveness of the unlearning methods, we utilized a SOTA MIA called U-LiRA~\cite{hayes2024inexact} in addition to using the MIAs from prior works. It extends the per-example Likelihood Ratio Attack (LiRA) to the unlearning setting by constructing shadow model distributions that incorporate both the training and unlearning procedures, enabling a more fine-grained, sample-specific membership inference analysis.


\section{Results}
In this section, we present a comprehensive empirical evaluation of our proposed unlearning methods across multiple datasets, architectures, and various evaluation metrics. We also report the results of a version of TRW that applies to fine-tuning of only $2$ randomly chosen layers of the model, averaged over multiple trials. In section~\ref{sec:main results} we present a thorough comparison using evaluation metrics used in prior unlearning work. In section~\ref{sec:ulira} we present the results using a SOTA MIA method that is stronger than prior MIA methods used in unlearning evaluations and can complement the results of MIA-NN for a more comprehensive evaluation. We also evaluate the computation time of our method in~\cref{sec:time analysis}. The ablation study on the value of $\beta$ is presented in section~\ref{sec:ablation}, and the results for unlearning multiple classes can be found in section~\ref{sec:multi-class}.

\subsection{Comparison to Existing Methods}
\label{sec:main results}

\Cref{tab:main_results} presents the results for single-class forgetting across CIFAR-10 using VGG19 and ResNet18. Similar results for MNIST and CIFAR-100 are shown in section~\ref{sec:mnist_cifar} and evaluations on Tiny-ImageNet dataset for ResNet18 can be found in section~\ref{sec:imagenet}.
Table~\ref{tab:main_results} shows that our method consistently achieves perfect forgetting, with \(\mathrm{ACC}_f = 0\) across all datasets and architectures. Importantly, the retained class accuracy (\(\mathrm{ACC}_r\)) remains competitive with or superior to baseline unlearning methods, indicating minimal interference with the retained knowledge. Additionally, the MIA scores are either comparable to or better than prior work, suggesting that our method does not compromise membership privacy. These results demonstrate that while our approach is resilient to MIA-NN (see Table~\ref{tab:miann_all}), it remains competitive to SOTA unlearning methods in common evaluation metrics used in prior work.
\begin{table}[htbp]
\small
\centering
\resizebox{\textwidth}{!}
{\begin{tabular}{@{}clc c cc c c@{}}
\toprule
\multirow{2}{*}{\textbf{Data}} & \multirow{2}{*}{\textbf{Method}} & \multicolumn{3}{c|}{\textbf{VGG19}~\cite{simonyan2014very}} & \multicolumn{3}{c}{\textbf{ResNet18}~\cite{he2016deep}} \\ 
\addlinespace
& & $ACC_r$ (\textuparrow) & $ACC_f$ (\textdownarrow) & $MIA$ (\textuparrow) & $ACC_r$ (\textuparrow) & $ACC_f$ (\textdownarrow) & $MIA$ (\textuparrow) \\
\midrule
\multirow{8}{*}{\rotatebox{90}{\textsc{CIFAR}-10}} 
    & Original &92.68 $\pm$ 0.05 &92.14$\pm$ 6.80  &0  & 94.74 $\pm$ 0.09 & 94.42 $\pm$ 5.45 & 0.02 $\pm$ 0.02\\
    & Retrain &93.45 $\pm$ 0.10  &0  & 100 $\pm$ 0& 94.83 $\pm$ 0.13 & 0 & 100 $\pm$ 0 \\
    & FT~\cite{warnecke2021machine} &89.33 $\pm$ 1.86  &0 & 96.94  $\pm$ 0.91 &85.60  $\pm$ 2.35& 0 & 96.53 $\pm$ 1.16\\
    & RL~\cite{golatkar2020eternal} &85.57 $\pm$ 1.29 &0 &94.07 $\pm$ 0.90 &84.74 $\pm$ 4.25 & 0& 94.99 $\pm$ 1.82 \\
    & GA~\cite{thudi2022unrolling} &87.26 $\pm$ 0.30 &14.4 $\pm$ 1.59 &97.10 $\pm$ 0.87 & 90.25 $\pm$ 0.28 &14.12 $\pm$ 2.17 &96.70 $\pm$ 0.10\\
    & l1~\cite{jia2023model} &90.12 $\pm$ 2.18  &0.3 $\pm$ 0.03 &92.61 $\pm$ 7.23 &93.21 $\pm$ 0.13 &0.9 $\pm$ 0.05 & 100 $\pm$ 0 \\
    & BU~\cite{chen2023boundary} &87.32$\pm$ 3.08 &0 &85.94 $\pm$ 3.71 &87.68 $\pm$ 2.23&0 &85.91 $\pm$ 3.97\\
    & SalUn~\cite{foster2024fast} &89.76 $\pm$ 1.00 &0 &97.35 $\pm$ 0.65&92.11 $\pm$ 0.65& 0 & 96.33 $\pm$ 2.37\\
    & SVD~\cite{kodge2024deep} &91.34 $\pm$ 0.45 &0 &98.10 $\pm$ 1.90 &94.17 $\pm$ 0.57 &0 & 97.20 $\pm$ 3.77\\
    & SCRUB~\cite{kurmanji2023towards} &89.84 $\pm$ 1.16 &0 &84.28 $\pm$ 1.91 & 91.07 $\pm$ 0.79& 0&  85.01 $\pm$ 1.02\\
    & SCAR~\cite{bonato2024retain} &92.59 $\pm$ 1.80 &0 &98.06 $\pm$ 2.11 &93.57 $\pm$ 0.03 &0 &95.87 $\pm$ 3.58 \\
    & l2ul~\cite{cha2024learning} &89.15 $\pm$ 0.75 & 0 &96.82 $\pm$ 1.45 &87.86 $\pm$ 1.79 &0 & 94.62 $\pm$ 0.15\\
    & \textbf{TRW} &91.58 $\pm$ 0.23 &0&99.26 $\pm$ 0.74 &94.28 $\pm$ 0.47 &0&97.65 $\pm$ 2.06\\
    & \textbf{TRW-2R} & 91.91 $\pm$ 0.63 &0& 99.52 $\pm$ 0.74&94.39 $\pm$ 2.05 &0.32 $\pm$ 0.32&96.35 $\pm$ 0.65\\
\addlinespace
\bottomrule
\end{tabular}}
\vspace{2mm}
\caption{
TRW and TRW-2R achieve competitive or superior results, on common unlearning evaluation metrics, when compared to other baselines.  
}
\label{tab:main_results}
\end{table}

\subsection{Stronger MIA evaluation}
\label{sec:ulira}

We evaluate our method under a stronger MIA using the U-LiRA framework~\cite{hayes2024inexact}, designed to assess whether unlearned models retain residual information about the forgotten class. U-LiRA simulates an adversary with access to shadow models to perform inference attacks on the forget samples. 

Following the U-LiRA protocol, we train three shadow ResNet18 models on CIFAR-10. Each is unlearned with consistent hyperparameters to generate shadow unlearned models. Separately, we retrain three additional models with the forget class excluded to serve as shadow Retrain models. The attacker is trained to distinguish whether a given prediction comes from an unlearned or Retrain model based on class-conditional statistics, and the learned decision boundary is applied in a leave-one-out manner. A perfect unlearning method should yield $50\%$ accuracy—indicating the accuracy of a random classifier. Based on the results reported in Table~\ref{tab:ulira-results},
our methods achieve the lowest U-LiRA accuracy, demonstrating strong resilience to adaptive attacks.

\begin{table}[htbp]
\centering
\resizebox{\textwidth}{!}{\begin{tabular}{l|c c c c c c c c c c c c }
\toprule
Metric &TRW &TRW-2R & {FT} & {RL} & {GA} & {SalUn} & {SVD} & {l1} & {BU} & {SCRUB} & {SCAR} & {l2ul} \\
\midrule
U-LiRA Accuracy (\%) & 71.12 &\textbf{67.72} &72.57 & 96.79& 84.47 &97.50  & 79.30 &98.32  & 81.08 & 71.91 & 92.38  & 85.67\\
\bottomrule
\end{tabular}}
\vspace{5pt}
\caption{U-LiRA Membership Inference Attack accuracy on CIFAR10 with ResNet18. Lower is better ($50\%$ indicates ideal unlearning). TRW-2R and TRW achieve the best performance.}
\label{tab:ulira-results}
\end{table}

\section{Limitations}
Although we demonstrate the effectiveness of our approach through various empirical studies, the evidence presented is entirely empirical. Further theoretical work could provide deeper insights into the effectiveness of the tilted reweighting objective and the design of appropriate scoring functions to better approximate the distribution of retrained models. The scoring function used in this work is sample-independent; evaluating finer-grained scores that account for individual samples could be an interesting direction for future research. Moreover, adding higher-order constraints could be studied for capturing the systematic bias introduced in the retrained models. While MIA-NN reveals shortcomings in prior work, it remains heuristic. But we believe it can serve as a great complement to other membership attacks and evaluation metrics for future work in class unlearning. Further improvement in this method could be another interesting avenue for exploration.

\newpage
\chapter{Future Directions}
\label{sec:LLM-unlearning}

Large Language Models (LLMs) increasingly rely on large-scale corpora, often scraped from the internet and blending high-utility public information with sensitive or proprietary material. The societal and regulatory need to support the ``right to be forgotten'' has made \emph{machine unlearning} a central problem: removing the influence of specific training data without retraining entire models from scratch. Despite recent progress, existing approaches suffer from three major limitations. First, forgetting is inefficient: naively retraining or attempting to scrub a large set of text is computationally prohibitive and increases the risk of hurting model's performance. Second, unlearning a whole \emph{concept} (e.g., Harry Potter), rather than a given forget set (Harry Potter books), from generative models pose new challenges, as concepts are distributed across high-dimensional generative pathways. Third, theoretical guarantees are scarce, leaving unlearning procedures without clear assurances of how close they come to true retraining. This proposal integrates advances in multiple domains across machine learning to achieve the following objectives: (i) develop PN/PS-guided methods for identifying core-sets of causally decisive forget samples, (ii) designing unlearning methods for removing a whole concept from generative models, and (iii) characterizing the geometry of the target concept to ensure unlearning without harm to knowledge of the model about other concepts.

These subproblems can be completed and evaluated separately. For example, for core-set construction, any forget set can be used as input, and the derived subset could be used in existing unlearning method to verify an improvement in the quality of unlearning various concepts. Therefore, we plan to complete these projects in the given order. In the following, we will provide more details on the methodology that we will be using for each of these directions. We also make some connections to the existing literature to motivate the proposed methods.

\section{Proposed Method}

In this section we elaborate on the proposed methods for each thrust. First thrust investigates the construction of a core-set in an efficient way. The second one utilizes the solution to the first thrust to ease a given \emph{concept}. Finally, we investigate the characterization of the geometry of the concept in order to better understand the effects of unlearning and how to preserve the model's utility on overlapping concepts.


\vspace{1mm}\noindent{\bf Task A: Causal Core-Set Construction.} We pose the following research question: {\em Given a set of samples to forget ($D_F$) how can we fine a core set $C \subseteq D_F$ that, when used for unlearning, suppresses knowledge of $D_F$?} For each sample $x \in D_F$, we localize causally relevant model activation $Z_x$ using attribution patching methods. By performing \emph{activation patching} ($\doop(Z_x{=}1)$) and \emph{ablation} ($\doop(Z_x{=}0)$), we estimate probabilities of necessity and sufficiency: 
\begin{align*}
\PN(x)\triangleq
\PP \big(Y(Z_x{=}0){=}0 \,\big|\, Z_x{=}1,\, Y{=}1\big),\quad
\PS(x)\triangleq
\PP \big(Y(Z_x{=}1){=}1 \,\big|\, Z_x{=}0,\, Y{=}0\big).
\end{align*}

By building a set of prompts using $D_F$ (e.g., paraphrases, question-answer pairs, etc.), and aggregating PN/PS across this set, we derive scores for each forget sample. We then formulate selection as maximizing a submodular coverage function $F(C)$ under a budget constraint, which admits a greedy algorithm with $(1-1/e)$ approximation. This ensures that forgetting focuses on causally decisive samples, improving efficiency and reducing side effects on unrelated tasks.

\vspace{1mm}\noindent{\bf Task B: Concept Unlearning.}
We pose the following research question: {\em Given a general concept (e.g., Harry Potter) and a seed forget set (e.g., passages from the Harry Potter books), how can we effectively remove the knowledge about that concept?} 
Given a finite region \( \mathcal{U} \) in the embedding space of the model corresponding to the target concept, we are interested in finding a family of subsets \( \mathcal{S}=\{S_1,\dots,S_n\}\) with \(S_i\subseteq\mathcal{U}\) containing embeddings corresponding to some possible samples, such that \( \mathcal{U} \subseteq  \bigcup_{S\in\mathcal{S}} S\). Then, the problem reduces to finding
a \emph{cover-set} \( \mathcal{C}\subseteq\mathcal{S} \), such that:
$\bigcup_{S\in\mathcal{C}} S \;=\; \mathcal{U}$. The goal is to find a \emph{minimal cover-set} which can be given to unlearning methods as input to erase the concept.

Our approach leverages \emph{causal activation tracing} to identify the set of neurons and layers most responsible for generating concept-relevant outputs. Starting from a seed forget set, we generate a wide distribution of synthetic examples by prompting the model with randomly selected partial prompts and inciting the concept-related activations. This bootstrapped corpus expands coverage of the target concept beyond the initial seed while maintaining causal relevance. To constrain the size of this synthetic corpus, we introduce filtering mechanisms such as (i) embedding distance thresholds from the seed forget set, ensuring only semantically proximate samples are retained, and (ii) perplexity-based screening to remove low-quality generations. Once this expanded corpus is constructed, we apply our PN/PS-guided core-set algorithm to select a minimal, causally decisive subset of forget examples. This ensures efficient unlearning while capturing the diverse contexts in which the concept manifests. 

\vspace{1mm}\noindent{\bf Task C: Structural Characterization of Concepts.} Once the quality of the computed cover-sets are established, we proceed to studying the geometry of the constituting contents in the representation space. A better understanding of how a concept is presented in embedding space and how it can be distinguished from the representation of model's knowledge from different but related concepts, could lead to a more effective removal of knowledge about arbitrary concepts. Prior work has suggested the idea of removing the components of layer-wise transformation that align with other concepts to prevent the model from unlearning concepts unrelated to the forgetting concept~\cite{lizzo2024unlearn}. Orthogonalizing the additional low-rank transformations (e.g., LoRA layers) to facilitate handling multiple unlearning requests has been also proposed by prior work~\cite{gaolarge}. We believe the tools that we have developed for modifying the spectrum of affine layers and orthogonalizing them with respect to one another (see section~\ref{sec:intro}) will become effective in achieving precise manipulation of model's knowledge about the unlearning concepts.

\section{Related Work}

Here, we explore prior work on each of the discussed steps in our proposed method to utilize the state-of-the-art for each task and build on them to design our approach.

\subsection{Attribution Mapping}

Some prior work consider analyzing the attiribution neuron activations to various responses or facts which can help to changing the responses or negate the learned facts. \cite{krause2020gedi} GeDi pre-trains small class-conditional language models that act as token-level discriminators. At each decoding step, Bayes-rule re-weighting of the base model’s logits by the discriminator’s probabilities effectively excites activations aligned with the target attribute and suppresses the opposite class. The result is fast, attribute-controlled generation (e.g., positive vs. negative sentiment) without modifying the main model’s parameters.~\cite{meng2022locating} introduces causal tracing to pinpoint the minimal sequence of hidden-state locations (mostly mid-layer MLP neurons) that mediate the recall of a specific fact (e.g., “The capital of Italy is Rome”). By ablating or patching each layer while keeping the rest intact, the authors isolate the layers whose activations causally drive the correct prediction; those activations are then edited with a rank-one update (ROME) to overwrite or correct the fact without degrading unrelated knowledge. 

\cite{dai2021knowledge} propose a knowledge-attribution metric that attributes a factual answer to individual feed-forward neurons. For each fact, the method measures how much muting a neuron lowers the log-probability of the correct token, thereby surfacing a sparse set of “knowledge neurons.” Activating or damping these neurons can inject, erase, or flip the stored fact, showing that neuron-level paths encapsulate distinct knowledge items.
\cite{zhang2023towards} systematically evaluate activation patching—copying hidden states from a “clean” run into a “corrupted” run—to localize where in the network a property (syntax, world fact, sentiment) is processed. The paper analyses evaluation metrics and corruption methods, demonstrating that careful patch design identifies a coherent activation path spanning specific heads and MLP blocks, and that these paths differ across tasks. \cite{ortu2024competition} extend causal-tracing to compare factual vs. counter-factual prompts, revealing a two-phase path: early MLPs enrich subject tokens with attributes, while later attention heads route that information to the final position. By analyzing divergence points between factual and counterfactual activations, the authors map a contiguous, interpretable neuron path for factual retrieval.
\cite{stolfo2023mechanistic} use causal mediation analysis to decompose the model’s answer into contributions from individual layers and neurons when solving arithmetic word problems. Mediation reveals a chain of activations—starting with numeral recognition, passing through intermediate symbolic computation neurons, and culminating in answer synthesis—thereby identifying the neuron-level path that implements multi-step reasoning.


\subsection{Topic-wise Attribution}

Some prior work extend this attributions for finding relevant paths to a topic, style, or task. \cite{dathathri2019plug} introduce a new method thatcouples a frozen generative model (e.g., GPT-2) with a lightweight attribute classifier. During generation, gradients from the classifier are back-propagated into the current hidden state; a small, iterative update nudges activations toward directions that increase the classifier’s score while preserving fluency. Because only the forward/backward pass is used—no fine-tuning—the method can “excite” topic, sentiment, or stylistic neurons on the fly and produce controlled text without retraining.~\cite{turner2023activation} show that many control signals can be expressed as fixed activation vectors obtained from contrasting two prompts. Simply adding the difference vector to the hidden state at a chosen layer and position produces consistent steering (e.g., turning a neutral continuation into a Shakespearean one). The approach needs no gradient steps, runs in one forward pass, and demonstrates that selective excitation of latent directions suffices for reliable style or content transfer.
\cite{qiu2024spectral} propose a method that decomposes hidden-state matrices with singular-value analysis to isolate spectral components tied to truthfulness or bias. By amplifying or dampening those components at inference time, the model self-corrects hallucinations or reduces biased language, indicating that specific singular vectors correspond to semantically coherent neuron subspaces whose targeted excitation or suppression modulates generated content. 

The method introduced by~\cite{soo2025steering} build on Contrastive Activation Addition and sparse-autoencoder probes to first learn a dictionary of interpretable “features.” It then selects the subset whose linear combination best matches a user prompt and injects that composite vector once into the transformer stack. Compared with raw activation addition, the feature basis delivers finer topical control and better grammatical coherence, illustrating that steering at the feature rather than individual-neuron level yields higher-quality text. 
\cite{chen2025knowledge} introduce a method that identifies sparse-autoencoder “features” responsible for a sensitive fact, then subtracts or attenuates only those feature activations across layers, erasing the fact while leaving unrelated knowledge intact. The paper demonstrates higher delete-success and lower collateral damage than neuron-level edits, reinforcing the view that carefully targeted excitation (or suppression) of semantically unified activation clusters affords precise control over generated content.
The method introduced by\cite{meng2022mass} extends ROME~\cite{meng2022locating} to thousands of facts simultaneously. It clusters subject–relation keys that share nearby activation directions and applies batched low-rank edits, effectively dampening or deleting large topic blocks (e.g., entire biographical sections) with minimal perplexity increase. The paper shows that large-scale suppression is feasible when the responsible neuron paths are identified and edited in bulk.

\subsection{Geometry of Concepts}

Prior work has studied the linear representations in LLMs~\cite{wang2023concept,jiang2024origins}. Recent formulations of the linear-representation hypothesis established that large language models (LLMs) encode certain binary, contrastive concepts, such as male $\iff$ female, as specific directions in representation space~\cite{park2023linear}. Those formulations rely on counterfactual word pairs (e.g., “king” versus “queen”). More recent studies~\cite{park2024geometry} extend this geometrical study of the representations to concepts lacking a natural opposite (e.g., is animal) or to multi-way categories. The authors show that non-contrastive features are represented not merely as a direction but as a vector with fixed magnitude and extend this finding to show that a categorical concept corresponds to the convex polytope—typically a simplex—spanned by its member vectors. Answering this broader question that how general concepts are geometrically realized within LLM vector spaces, will lead to easier manipulation of the model's knowledge about each concept.

\subsection{Applications to Unlearning}

These attribution methods have already been used by some prior work to directly unlearn a topic or suppress harmful content. 
\cite{yu2023unlearning} aimed at demographic bias, this paper partitions training gradients into bias-carrying and bias-neutral subspaces. By projecting away the bias gradients, it reduces the corresponding neuron activations at inference time, which in turn suppresses biased topical content (e.g., gender stereotypes) without extensive retraining or loss in standard benchmarks.
\cite{guo2024mechanistic} combined circuit-finding tools with iterative unlearning: after tracing the attention/MLP circuit that carries a capability (e.g., instructions for illicit behaviour) it repeatedly masks and fine-tunes only that circuit until the unwanted generations disappear. The approach yields stronger and more persistent suppression than gradient-based editing, highlighting the benefit of first mapping and then weakening the activation path.

\cite{lizzo2024unlearn} introduce UNLEARN, a parameter-efficient framework for selectively erasing or injecting knowledge in large language models without costly full retraining. UNLEARN first performs subspace identification by sequentially training low-rank “task matrices” layer-by-layer to capture the weight subspace that drives a target task; it then applies a modified Gram–Schmidt subspace discrimination to orthogonalize that subspace against those of semantically similar tasks, and finally subtracts the purified matrix from the frozen base weights. Overall, the study demonstrates that low-rank, discriminated subspace edits provide an efficient, scalable route to both knowledge removal and targeted augmentation in state-of-the-art LLMs.~\cite{gaolarge} proposed a framework for continual unlearning in large language models that eliminates the need for a retained (clean) dataset. They use an Orthogonal Low-Rank Adapter (LoRA), whose orthogonality regularizer disentangles successive unlearning requests so they do not interfere with one another. They also use an out-of-distribution detector trained by a contrastive entropy loss so that at inference time they can assign a soft weight that determines how much of the unlearning LoRA to load to derive a balance and forgetting with utility.


\clearpage

\backmatter

\printbibliography[heading=bibintoc,title={References}]

\clearpage
\setcounter{counterforappendices}{\value{page}}
\mainmatter
\setcounter{page}{\value{counterforappendices}}
\appendix
\chapter{Controlling the Spectrum}

\section{Omitted Proofs}
\label{apx-sec-omitted-proofs}

\subsection{Proofs of Section~\ref{sec-extraction}}

We start by proving Proposition~\ref{lemma-shifted-subspace}:

\begin{proof}
 Let $g(X) = f(X) - f(0) = M_W X + b - b = M_W X$. Notice that $\nabla_X \frac{1}{2} \|g(X)\|^2 = M_W ^T M_W X = A_W X$. Let define $A^\prime = A_W + \mu I$. Putting steps $5$ and $6$ of the algorithm together, we have $(Q, R) = \mathrm{QR}(A_W X + \mu X) = \mathrm{QR}(A^\prime X)$. Therefore, the algorithm above can be seen as the same as the subspace iteration for matrix $A^\prime$. Hence, the diagonal values of $R$ will converge to the eigenvalues of $A^\prime$, and columns of $Q$ converge to the corresponding eigenvectors. If $\lambda$ is an eigenvalue of $A^\prime = A_W + \mu I$, then $\lambda - \mu$ is an eigenvalue of $A_W$. Therefore, by subtracting $\mu$ from the diagonal elements of $R$, we get the eigenvalues of $A_W$, which are squared of singular values of $M_W$. The eigenvectors of $A^\prime$ are the same as the eigenvectors of $A_W$ and, therefore, the same as the right singular vectors of $M_W$. 

\end{proof}

\subsection{Proofs of Section~\ref{sec-limitations}}

For proving~\Cref{theorem-limitations}, we first present and prove Lemma~\ref{pro-eigval} which proves a similar result for convolutional layers with only $1$ input and output channel.

\begin{lemma}
Given a convolutional layer with single input and output channel and circular padding applied to an input whose length of vectorized form is $n$, if the vectorized form of the kernel is given by $\textbf{f} = [f_0, f_1, \dots, f_{k-1}]$, the singular values are:
\begin{align}
    \mathcal{S}(\omega) = \left\{\sqrt{c_0 + 2\sum_i^{k-1} c_i \Re(\omega^{j\times i})},\,\, j=0,1,2, \dots, n-1 \right\}
\end{align}

\noindent where $c_i$'s are defined as:
\begin{align*}
    c_0 &:= f_0^2 + f_1^2 + \dots + f_{k-1}^2,\\
    c_1 &:= f_0f_1 + f_1f_2 + \dots + f_{k-2}f_{k-1},\\
    &\vdots\\
    c_{k-1} &:= f_0f_{k-1}.
\end{align*}

\label{pro-eigval}
\end{lemma}
\begin{proof}
The above convolutional operator is equivalent to a circulant matrix $A$ with $[f_m, f_{m+1}, \dots, f_{k-1}, 0, \dots, 0, f_0, f_1, \dots, f_{m-1}]^T$ as its first row, where
$m=\floor{\frac{k}{2}}$~\cite{jain1989fundamentals,sedghi2018singular}. Singular values of $A$ are the eigenvalues of $M=A^TA$. Circulant matrices are closed with respect to
multiplication, and therefore $M$ is a symmetric circulant matrix. It is easy to check that the first row of $M$ is $r = [c_0, c_1, \dots, c_{k-1}, 0, \dots, 0, c_{k-1},\dots, c_1]^T$. 

Now, we know the eigenvalues of a circulant matrix with first row $v = [a_0, a_1, \dots, a_{n-1}]^T$ are given by the set $\Lambda = \{ v\Omega_j, j=0,1,\dots,n-1 \}$, where $\Omega_j = [\omega^{j\times 0}, \omega^{j\times 1}, \dots, \omega^{j\times n-1}]$. So eigenvalues of $M$ are given by the set below:

\begin{align*}
    \{r\Omega_j,\, j=0,1,\dots, n-1\} &= \{c_0 + c_1\omega^{j\times 1} + c_2\omega^{j \times 2} + \dots + c_{k-1}\omega^{j\times (k-1)} \\
    &+ c_{k-1}\omega^{j\times (n-k+1)} + c_{k-2}\omega^{j\times (n-k+2)} + \dots + c_1\omega^{j\times (n-1)}, 
    \, j=0,1,\dots, n-1 \}.
\end{align*}

Note that $\omega^{j\times i}$ has the same real part as $\omega^{j\times (n-i)}$, but its imaginary part is mirrored with respect to the real axis (the sign is flipped). Therefore, $\omega^{j\times i} + \omega^{j\times (n-i)} = 2\Re(\omega^{j\times i})$. Using this equality the summation representing the eigenvalues of $M$ can be simplified to:

\begin{align*}
    \{c_0 + 2\sum_i^{k-1} c_i \Re(\omega^{j\times i}),\, j=0,1,\dots,n-1 \},
\end{align*}

and therefore the singular values can be derived by taking the square roots of these eigenvalues.
\end{proof}

Now using Lemma~\ref{pro-eigval}, we can easily prove~\Cref{theorem-limitations}.

\begin{proof}
The convolutional layer with $m$ output channels and one input channel can be represented by a matrix $M = [M_1, M_2, \dots, M_m]^T$, where each $M_i$ is an $n\times n$ circulant matrix representing the $i-$th channel. Therefore, $A = M^T M$ (or $MM^T$ when there are multiple input channels) can be written as $\sum_{l=1}^m M_l^T M_l$. So, for circulant matrix $A$ we have $c_i= \sum_{l=1}^m c_i^{(l)}$, and the proof can be completed by using Lemma~\ref{pro-eigval}. 
\end{proof}

Next, we focus on Corollary~\ref{cor_eig_2}, which uses the results of~\Cref{theorem-limitations} to give an easy-to-compute lower and upper bounds for the spectral norm of the convolutional layers with either one input or output channel. When the filter values are all positive, these bounds become equalities and provide and easy way to compute the exact spectral norm.

\begin{corollary}
Consider a convolutional layer with $1$ input channel and $m$ output channels or $1$ output channel and $m$ input channels and circular padding. If the vectorized form of the kernel of the $l$-th channel is given by $\textbf{f}^\textbf{(l)} = [f_0^{(l)}, f_1^{(l)}, \dots, f_{k-1}^{(l)}]$, and the largest singular value of the layer is $\sigma_1$, then:
\begin{align}
   \sqrt{\sum_{l=1}^{m} \left(\sum_{i=0}^{k-1} f_i^{(l)}\right)^2} \leq \sigma_1 \leq \sqrt{\sum_{l=1}^{m} \left(\sum_{i=0}^{k-1} |f_i^{(l)}| \right)^2},
\end{align}

\noindent and therefore the equalities hold if all $f_i^{(l)}$s are non-negative.

\label{cor_eig_2}
\end{corollary}

To make the proof more clear, we first present a similar corollary for Lemma~\ref{pro-eigval}, which proves similar results for convolutions with only $1$ input and output channel.

\begin{corollary}
Consider a convolutional layer with single input and output channels and circular padding. If the vectorized form of the kernel is given by $\textbf{f} = [f_0, f_1, \dots, f_{k-1}]$, and the largest singular value of the layer is $\sigma_1$, then:
\begin{align}
    \sum_{i=0}^{k-1} f_i \leq \sigma_1 \leq \sum_{i=0}^{k-1} |f_i|,
\end{align}

\noindent and therefore, the equalities hold if all $f_i$s are non-negative.

\label{cor_eig}
\end{corollary}

\begin{proof}
From Lemma~\ref{pro-eigval}, we can write:

\begin{align*}
    \mathcal{S}(\omega) &= \sqrt{c_0 + 2\sum_i^{k-1} c_i \Re(\omega^i)}
    = \sqrt{|c_0 + 2\sum_i^{k-1} c_i \Re(\omega^i)|}\\
    &\leq \sqrt{|c_0| + 2\sum_i^{k-1} |c_i \Re(\omega^i)|}
    \leq \sqrt{|c_0| + 2\sum_i^{k-1} |c_i|},
\end{align*}

\noindent where the last inequality is due to the fact that $\Re(\omega^i) \leq 1$. Now, for $c_i$s we have:
\begin{align*}
    |c_0| &= f_0^2 + f_1^2 + \dots + f_{k-1}^2,\\
    2|c_1| &\leq 2\left( |f_0||f_1| + |f_1||f_2| + \dots + |f_{k-2}||f_{k-1}|\right),\\
    &\vdots\\
    2|c_{k-1}| &\leq 2\left( |f_0||f_{k-1}|\right).
\end{align*}

Note that the summation of the right-hand sides of the above inequalities is equal to $(\sum_{i=0}^{k-1} |f_i|)^2$ Therefore:
\begin{align*}
    \mathcal{S}(\omega) \leq \sum_{i=0}^{k-1} |f_i| \implies \sigma_1 \leq \sum_{i=0}^{k-1} |f_i|.
\end{align*}

Since $1$ is always one of the $n-$th roots of unity, if all the $f_i$s are non-negative, all the above inequalities hold when $\omega = 1$ is considered.    
Now, note that when $\omega = 1$:

\begin{align*}
    \mathcal{S}^2(1) &= c_0 + 2\sum_i^{k-1} c_i \\
    &=f_0^2 + f_1^2 + \dots + f_{k-1}^2 \,\,\,\,\text{($c_0$)}\\
    &\quad + 2f_0f_1 + 2f_1f_2 + \dots + 2f_{k-2}f_{k-1}\,\,\,\,\text{($c_1$)}\\
    &\quad \,\,\,\,\vdots\\
    &\quad + 2f_0f_{k-1}\,\,\,\,\text{($c_{k-1}$)}\\
    &= (\sum_{i=0}^{k-1} f_i)^2
\end{align*}
Therefore, $\sum_{i=0}^{k-1} f_i$ is a singular value of the convolution layer, which gives a lower bound for the largest one and this completes the proof.

\end{proof}

Now we give the proof for Corollary~\ref{cor_eig_2}. 

\begin{proof}
    From~\Cref{theorem-limitations}, we can write:

\begin{align*}
    \mathcal{S}(\omega) = \sqrt{\sum_{l=1}^m \mathcal{S}^{(l)^2}(\omega)} \leq \sqrt{\sum_{l=1}^m (\sum_{i=0}^{k-1} |f_i^{(l)}|)^2},
\end{align*}

    \noindent where the inequality is a result of using Corollary~\ref{cor_eig} for each $\mathcal{S}^{(l)}(\omega)$. 
    For showing the left side of inequality, we set $\omega = 1$ ($1$ is one of the $n$-th roots of unity), and again use~\Cref{theorem-limitations} to get:
    
\begin{align*}
    \mathcal{S}(1) &= \sqrt{\sum_{l=1}^m \mathcal{S}^{(l)^2}(1)} = 
    \sqrt{c_0^{(l)} + 2\sum_i^{k-1} c_i^{(l)} \Re(1)}  
    = \sqrt{\sum_{l=1}^m (\sum_{i=0}^{k-1} f_i^{(l)})^2},
\end{align*}

\noindent where the last equality was shown in the proof of Corollary~\ref{cor_eig}. This shows the left side of the inequality in (4) is one of the singular values of the layer, and hence is a lower bound for the spectral norm (the largest singular value of the layer).
    
\end{proof}

\section{Methods (Cont.)}
\label{apx-B}

\subsection{Modifying the Whole Spectrum}
\label{sec-modification}

If we use PowerQR to extract singular values and right singular vectors $S$ and $V$ from $M_W = USV^T$, then
given new singular values $S^\prime$, we can modify the spectrum of our function to generate a linear operator $f^\prime: x \rightarrow M_W^\prime x + b$, where $M_W^\prime = US^\prime V^T$, without requiring the exact computation of $U$ or $M_W$:
\begin{align*}
    f_W(V S^{-1} S^\prime V^T x) - f(0) = USV^T V S^{-1} S^\prime V^T x 
    = US^\prime V^T x = f^\prime (x) - b.
\end{align*}

Although $f^\prime$ is a linear operator with the desired spectrum, it is not necessarily in the desired form. For example, if $f_W$ is a convolutional layer with $W$ as its kernel, $f_W(V S S^\prime V^T x)$ will not be in the form of a convolutional layer. To find a linear operator of the same form as $f_W$, we have to find new parameters $W^\prime$ that give us $f^\prime$. For this, we can form a convex objective function and use SGD to find the parameters $W^\prime$ via regression:
\begin{align}
    \min_{W^\prime} \mathbb{E}_x \left\|f_{W^\prime}(x) - f_W(V S^{-1} S^\prime V^T x)\right\|_\mathrm{F}^2,
    \label{equ_modify}
\end{align}

\noindent where $x$ is randomly sampled from the input domain. Note that we do not need many different vectors $x$ to be sampled for solving~\ref{equ_modify}. If the rank of the linear operator $f_W(.)$ is $n$, then sampling as few as $n$ random data points would be enough to find the optimizer $f_{W^\prime}$. This procedure can be seen as two successive projections, one to the space of linear operators with the desired spectrum and the other one to the space of operators with the same form as $f_W$ (e.g., convolutional layers). This method can be very slow because the matrix $V$ can be very large. One might reduce the computational cost by working with the top singular values and singular vectors and deriving a low-ranked operator. This presented method is not practical for controlling the spectral norm of large models during training, similar to the other algorithms that need the whole spectrum for controlling the spectral norm~\cite{sedghi2018singular,senderovich2022towards}.


Another important point to mention here is that depending on the class of linear operators we are working with, the objective~\ref{equ_modify} might not reach zero at its minimizer. For example, consider a 2d convolutional layer whose kernel is $1\times 1$ with a value of $c$. Applying this kernel to any input of size $n\times n$ scales the value of the input by the $c$. The equivalent matrix form of this linear transformation is an $n\times n$ identity matrix scaled by $c$. We know that this matrix has a rank of $n$, and all the singular values are equal to $1$. Therefore, if the new spectrum, $S^\prime$, does not represent a full-rank transformation, or if its singular values are not all equal, we cannot attain a value of $0$ in optimization~\ref{equ_modify}. As we showed in~\Cref{sec-limitations}, this problem is more general for the convolutional layers, as they are restricted in the spectrums they can represent.

\subsection{Clipping Batch Norm}
\label{apx-alg-bn}


Batch Normalization~\cite{ioffe2015batch} has proved to successfully stabilize and accelerate the training of deep neural networks and is thus by now standard in many architectures that contain convolutional layers. However, the adverse effect of these layers on the adversarial robustness of models has been noted in previous research~\cite{xie2019intriguing,benz2021revisiting,galloway2019batch}. As we showed in~\Cref{fig:resnet18-spectral-norm}b, not controlling the spectral norm of the batch normalization layers might forfeit the benefits of merely controlling the spectral norm of convolutional layers. We also revealed the compensating behavior of these layers as the spectral norm of convolutional layers are clipped to smaller values in~\Cref{fig:resnet18-spectral-norm}a. The presented results in~\Cref{tab:bn-cifar} confirm this adverse effect of batch norm layers by showing a boost in the adversarial robustness of the models when these layers are removed from the network; however, as the results show, removing these layers from the model will incur a noticeable loss to the performance of the model on the test set and hinders the optimization. As pointed out in~\Cref{sec-generalization}, removing these layers from more complex modes such as DLA might completely hinder the optimization of the model. Therefore, it is crucial to find a better way to mitigate the adverse effect of these layers on adversarial robustness while benefiting from the presence of these layers in improving the optimization and generalization of the models. 

Batch normalization layers perform the following computation on the output of their preceding convolution layer:

\begin{align*}
    y = \frac{x-\E(x)}{\sqrt{\mathrm{Var}(x) + \epsilon}} \ast \gamma + \beta.
    \label{def-bn}
\end{align*}

As pointed out by~\cite{gouk2021regularisation}, by considering this layer as a linear transformation on $x-\E(x)$, the transformation matrix can be represented as a diagonal matrix with $\gamma_i/\sqrt{\mathrm{Var}(x_i)+ \epsilon}$ values, and therefore its largest singular value is equal to $\max_i \left(|\gamma_i|/\sqrt{\mathrm{Var}(x_i)+ \epsilon} \right)$. So, by changing the magnitude of $\gamma_i$ values we can clip the spectral norm of the batch norm layer. This approach has been followed in prior work~\cite{gouk2021regularisation,senderovich2022towards,delattre2023efficient}. In~\Cref{exp-bn-clip}, we show the results for the application of this clipping method and point out its disadvantages when used in practice. We will also show that the capability of our presented clipping method to work on the concatenation of convolutional and batch norm layers provides us with a better alternative. The full exploration of its potential, however, is left for future work.

\begin{figure}
\centering
  \includegraphics[width=.68\linewidth]{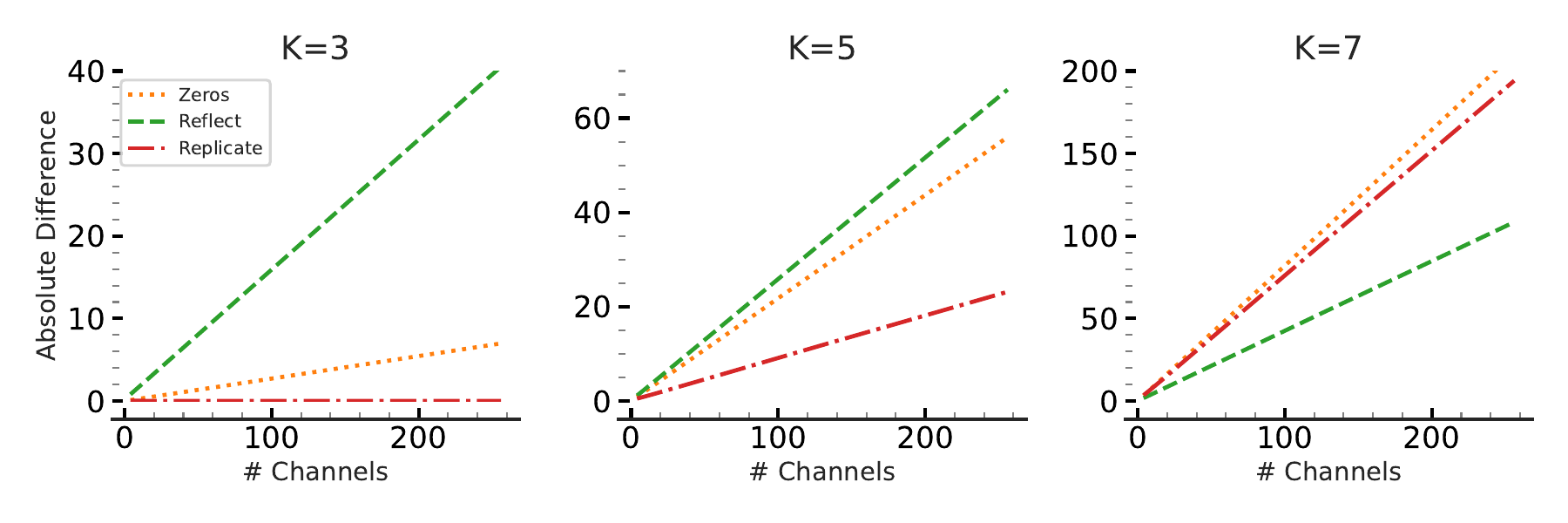}
\hfill
  \includegraphics[width=.28\linewidth]{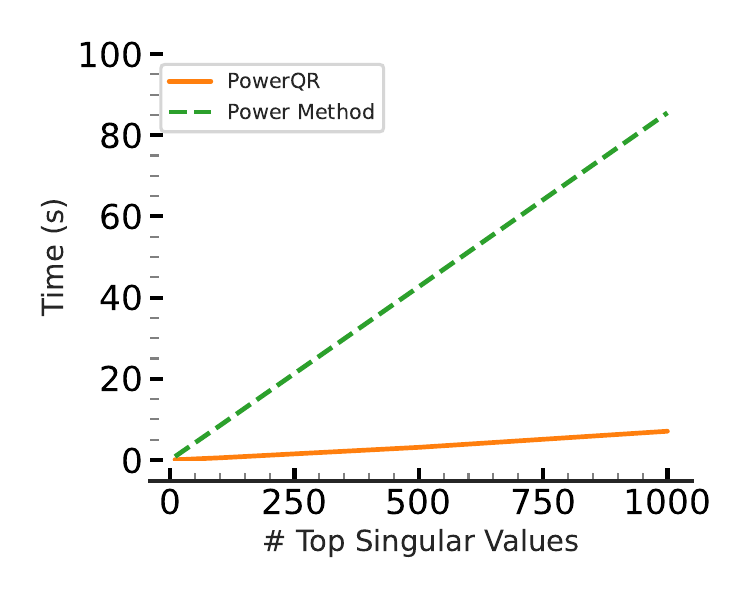}
\caption{\textbf{a.} The absolute difference in the spectral norm of convolutional layers with different padding types and their circulant approximates for various kernel sizes ($3$, $5$, and $7$) and numbers of channels. The values are computed by averaging over $100$ convolutional filters drawn from a normal distribution for each setting. \textbf{b.} Comparison of the run-time of PowerQR (\cref{alg:powerqr}) to that of the pipeline used by~\cite{virmaux2018lipschitz} for computing the top-$k$ singular values. We considered a 2d-convolutional layer with $3\times 3$ filters and $32$ input/output channels. The convolution is applied to a $32\times 32$ image.}
\label{fig:fig-time-fft-small-and-fig-QR-power-time}
\end{figure}

\section{Experiments (Cont.)}
\label{apx-exp}

In this section, we elaborate on the experiments and results pointed out in section~\ref{sec-experiment-clipping} and present some additional experiments to show the advantages of our proposed methods. The models used for the experiments are ResNet-18, as it was used by prior work~\cite{gouk2021regularisation,senderovich2022towards,delattre2023efficient}, which has a small modification compared to the original model; for the residual connections in the model, they simply divide the output of the layer by $2$ so that if both of the layers are $1$-Lipschitz, the output of the module will remain $1$-Lipschitz. Similarly, for the optimization of these models, we use SGD optimizer and a simple scheduler that decays the learning rate by $0.1$ every $40$ steps. We train each of the models for $200$ epochs. In addition to ResNet-18 models, we perform the comparisons on DLA, as it is presented in a publicly available GitHub repository~\footnote[1]{https://github.com/kuangliu/pytorch-cifar} as a model that achieves better results on CIFAR-10. We use the same training procedure for these models as the ones used for the ResNet-18 model. For generating adversarial examples and evaluating the models, we use a publicly available repository~\footnote[2]{https://github.com/AI-secure/Transferability-Reduced-Smooth-Ensemble/tree/main}, along with the strongest default values for the attacks ($c = \epsilon=0.02$ and $c = \epsilon=0.1$ for CW and PGD attack on CIFAR-10 and MNIST, respectively), and also add their training procedure for MNIST to evaluate both ResNet-18 and DLA models on an additional dataset. We also use a simple model which consists of a single convolutional layer and a dense layer on the MNIST dataset for some of our experiments regarding the correctness of clipping methods~(\Cref{fig:simple-clip-compare}) and showing the compensation phenomenon~(\Cref{fig:resnet18-spectral-norm}a).

\subsection{Error in Circulant Approximation}
\label{apx-exp-powerqr}

Our introduced method can be used to compute the exact spectrum of any implicitly linear layer, including different types of convolutional layers. However, the method introduced by~\cite{sedghi2018singular} only computes the spectrum of the convolutional layers when they have circular padding and no strides. This method was extended to convolutional layers with circular padding with strides other than $1$ by~\cite{senderovich2022towards}. These methods, as shown in~\Cref{subsec-comp-others}, are very slow and not practical for large models. These methods are also very memory consuming compared to other methods. More recently,~\cite{delattre2023efficient} introduced an algorithm that uses Gram iteration to derive a fast converging upper-bound for the convolutional layers with circular padding with a stride of $1$. As shown in~\Cref{fig:resnet-clip-compare}, this method is still much slower than our method for clipping and less accurate than our method and also its prior work on circulant convolutional layers. 
In~\Cref{fig:fig-time-fft-small-and-fig-QR-power-time}a we show that the approximation error using these method can be large on random convolutional layers. We use convolutional layers with different standard padding types with various filter sizes and choose the filter values from a normal distribution. Then we use the circulant approximation to approximate the spectral norm of these convolutional layers, and present the absolute difference with the true spectral norm. The presented values are averaged over $100$ trials for each setting. As the figure shows, the approximation error can be large, especially as the number of channels or the kernel size increases.

\subsection{Extracting Multiple Singular Values}
\label{apx-multi}

We also compared the run-time of the PowerQR with $100$ steps for extracting the top-$k$ eigenspace to the run-time of $k$ successive runs of the power method (as suggested in~\cite{virmaux2018lipschitz}). We did not include the deflation step (required for their proposed method to work), which makes the running time of their method even worse. Figure~\ref{fig:fig-time-fft-small-and-fig-QR-power-time}b shows that our implicit implementation of subspace iteration is dramatically more efficient.
\chapter{More on Machine Unlearning}
\label{apx:amun}

\section{Proofs}
\label{apx:proof}

Here we provide the proof of Theorem~\ref{theorem:amun}:

\begin{proof}
    As we perform the unlearning by fine-tuning and performing a gradient descent update to $\theta_o$, we have:
        $\theta^\prime =  \theta_o - \frac{1}{\beta} \nabla \hat{\mathcal{R}} (\theta_o)$.
    Therefore, we can write:

    \begin{align*}
        \|\theta^\prime - \theta_u\|_2^2  &= \|\theta_o - \frac{1}{\beta} \nabla \hat{\mathcal{R}} (\theta_o) - \theta_u\|_2^2 \\
        &=  
        \|\theta_o - \theta_u\|_2^2 - \frac{2}{\beta} \langle \nabla \hat{\mathcal{R}} (\theta_o), \theta_o - \theta_u \rangle + \frac{1}{\beta^2} \|\nabla \hat{\mathcal{R}} (\theta_o)\|_2^2 \\
        &\leq 
        \|\theta_o - \theta_u\|_2^2 + \frac{2}{\beta} (\hat{\mathcal{R}} (\theta_u) - \hat{\mathcal{R}} (\theta_o)) + \frac{2}{\beta} (\hat{\mathcal{R}} (\theta_o) - \hat{\mathcal{R}} (\theta^\prime)) \\
        &= 
        \|\theta_o - \theta_u\|_2^2 + \frac{2}{\beta} (\hat{\mathcal{R}} (\theta_u) - \hat{\mathcal{R}} (\theta^\prime)),
    \end{align*}

    \noindent where the inequality is derived by using the smoothness property ($\|\nabla \hat{\mathcal{R}} (\theta_o)\|_2^2 \leq 2\beta(\hat{\mathcal{R}} (\theta_o) - \hat{\mathcal{R}} (\theta^\prime))$) and the convexity assumption which leads to the inequality: $\hat{\mathcal{R}} (\theta_o)) \geq \hat{\mathcal{R}} (\theta_u) + \langle \nabla \hat{\mathcal{R}} (\theta_o), \theta_o - \theta_u \rangle$.

    Next, we derive an upper-bound for $\hat{\mathcal{R}} (\theta_u) - \hat{\mathcal{R}} (\theta^\prime)$ to replace in the above inequality. By the definition of unnormalized empirical loss on $\mathcal{D}^\prime$:

    \begin{align*}
        &\hat{\mathcal{R}} (\theta_u) - \hat{\mathcal{R}} (\theta^\prime) \\
        &= 
        \sum_{i=1}^{n-1} \ell(f_{\theta_u}(x_i), y_i) + \ell(f_{\theta_u}(x), y)  +  \ell(f_{\theta_u}(x^\prime), y^\prime)
        - \sum_{i=1}^{n-1} \ell(f_{\theta^\prime}(x_i), y_i) - \ell(f_{\theta^\prime}(x), y)  -  \ell(f_{\theta^\prime}(x^\prime), y^\prime) \\
        &=
        \ell(f_{\theta_u}(x), y)  + \ell(f_{\theta_u}(x^\prime), y^\prime)
        - \ell(f_{\theta^\prime}(x), y)  -  \ell(f_{\theta^\prime}(x^\prime), y^\prime),
    \end{align*}

    \noindent where the last equality was derived by the assumption that models are trained until they achieve near-$0$ loss on their corresponding dataset. Therefore, $\sum_{i=1}^{n-1} \ell(f_{\theta_u}(x_i), y_i) = \sum_{i=1}^{n-1} \ell(f_{\theta^\prime}(x_i), y_i) = 0$ since the retrained model has been trained on the remaining samples and the unlearned model has been derived by a single step of gradient descent on the original model, that had been trained on $\mathcal{D}$.

    To further simplify the derived terms above and reaching at our desired inequality, we focus on the term $ - \ell(f_{\theta^\prime}(x), y)$. By adding and decreasing the term $\ell(f_{\theta_o}(x^\prime), y)$ we get:

    \begin{align*}
        - \ell(f_{\theta^\prime}(x), y) &= - \ell(f_{\theta_o}(x^\prime), y) + \ell(f_{\theta_o}(x^\prime), y) - \ell(f_{\theta^\prime}(x), y)  \\
        &\leq 
        - \ell(f_{\theta_o}(x^\prime), y) + \ell(f_{\theta_o}(x^\prime), y)
        - \ell(f_{\theta_o}(x), y) - \langle \nabla \ell(f_{\theta_o}(x), y), \theta^\prime - \theta_o \rangle \\
        &= - \ell(f_{\theta_o}(x^\prime), y) + \ell(f_{\theta_o}(x^\prime), y)
        - \ell(f_{\theta_o}(x), y) \\
        &\leq
        - \ell(f_{\theta_o}(x^\prime), y) + L \delta,
    \end{align*}

    \noindent where the first inequality uses the convexity of the the loss function with respect to the parameters and the third derivations is due to the assumption that the original model achieves a zero loss on its training samples, including $(x,y)$ (hence, $\nabla \ell(f_{\theta_o}(x), y) =0$). The final inequality is due to the Lipschitzness assumption of model $f$ with respect to the inputs.
    
\end{proof}

\section{Implementation Details}
\label{apx:impl_details}

For all the experiments we train three models on $\D$. For each size of $\D$ ($10\%$ or $50\%$), we use three random subsets and for each subset, we try three different runs of each of the unlearning methods. This leads to a total of $27$ runs of each unlearning method using different initial models and subsets of $\D$ to unlearn. Hyper-parameter tuning of each of the methods is done on a separate random subset of the same size from $\D$, and then the average performance is computed for the other random subsets used as $\Df$. 
For tuning the hyper-parameters of the models, we followed the same range suggested by their authors and what has been used in the prior works for comparisons. Similar to prior works~\cite{liu2024model,fan2023salun}, we performed 10 epochs for each of the unlearning methods, and searched for best learning rate and number of steps for a learning rate scheduler. More specifically, for each unlearning method, we performed a grid search on learning rates within the range of $[10^{-6}, 10^{-1}]$ with an optional scheduler that scales the learning rate by $0.1$ for every $1$ or $5$ steps. For \texttt{SalUn}, whether it is used on its own or in combination with \amun{}, we searched for the masking ratios in the range $[0.1,0.9]$.

The original models are ResNet-18 models trained for 200 epochs with a learning rate initialized at 0.1 and using a scheduler that scales the learning rate by 0.1 every 40 epochs. The retrained models are trained using the same hyper-parameters as the original models. 
For evaluation using RMIA, we trained $128$ separate models such that each sample is included in half of these models. As suggested by the authors, we used Soft-Margin Taylor expansion of Softmax (SM-Taylor-Softmax) with a temperature of $2$ for deriving the confidence values in attacks of RMIA. We used the suggested threshold of $2$ for comparing the ratios in computing the final scores ($\gamma$ value).
For controlling the Lipschitz constant of the ResNet-18 models in~\S~\ref{sec:ablation-robust}, we used the default setting provided by the authors for clipping the spectral norm of all the convolutional and fully-connected layers of the model to $1$. For RMIA evaluations, we trained 128 of these models separately such that each sample appears in exactly half of these models.

\section{\amun{} + SalUn}
\label{apx:amun_salun}

The main idea behind SalUn is to limit the fine-tuning of the model, during unlearning, to only a subset of the parameters of the model, while keeping the rest of them fixed. \cite{fan2023salun} show that this technique helps to preserve the accuracy of the model when fine-tuning the model on $\Df$ with randomly-chosen wrong labels. More specifically, they compute a mask using the following equation:

 {\begin{equation*}
      \mathbf m_{\mathrm{S}} =  {\bf 1} \left ( \left |  \nabla_{\thetafull} \ell (\thetafull; \Df) \left . \right | \right  | \geq  \gamma \right ),
     \label{eq: sal_map_hard}
 \end{equation*}%
 which, basically, computes the gradient of the loss function for the current parameters with respect to $\Df$, and uses threshold $\gamma$ to filter the ones that matter more to the samples in $\Df$. Note that,  $ {\bf 1}$ is an element-wise indicator function. Then, during fine-tuning of the model on $\Df$ with random labels they use $\mathbf m_{\mathrm{S}}$ to detect the parameters of $\thetafull$ that get updated.

 In our experiments, we try combining this idea with \amun{} for updating a subset of the parameters that might be more relevant to the samples in $\Df$. We refer to this combination as \amun{}$_{SalUn}$ in Tables~\ref{tab:rmia} and~\ref{tab:rmia_forgetonly} in \S~\ref{sec:amun_results} and Tables~\ref{tab:mia} and~\ref{tab:mia_forgetonly} in Appendix~\ref{apx:svc_mia}. As the results show, \amun{}$_{SalUn}$ constantly outperforms \texttt{SalUn} and for the cases that $\Dr$ is not available it also outperforms \amun{}. In the setting where $\Dr$ is accessible, it performs comparable to \amun{}. This is probably due to the fact that when $\Dr$ is not available and \amun{} has access to only the samples $\Df \cup \Da$, SalUn acts as a regularization for not allowing all the parameters of the model that might not be relevant to $\Df$ be updated. In the setting where $\Dr$ is available, involving it in fine-tuning will be a sufficient regularization that preserves models' utility while unlearning $\Df$.

\begin{figure}[t!]
\centering
\begin{subfigure}
    \centering
    \includegraphics[width=.32\linewidth]{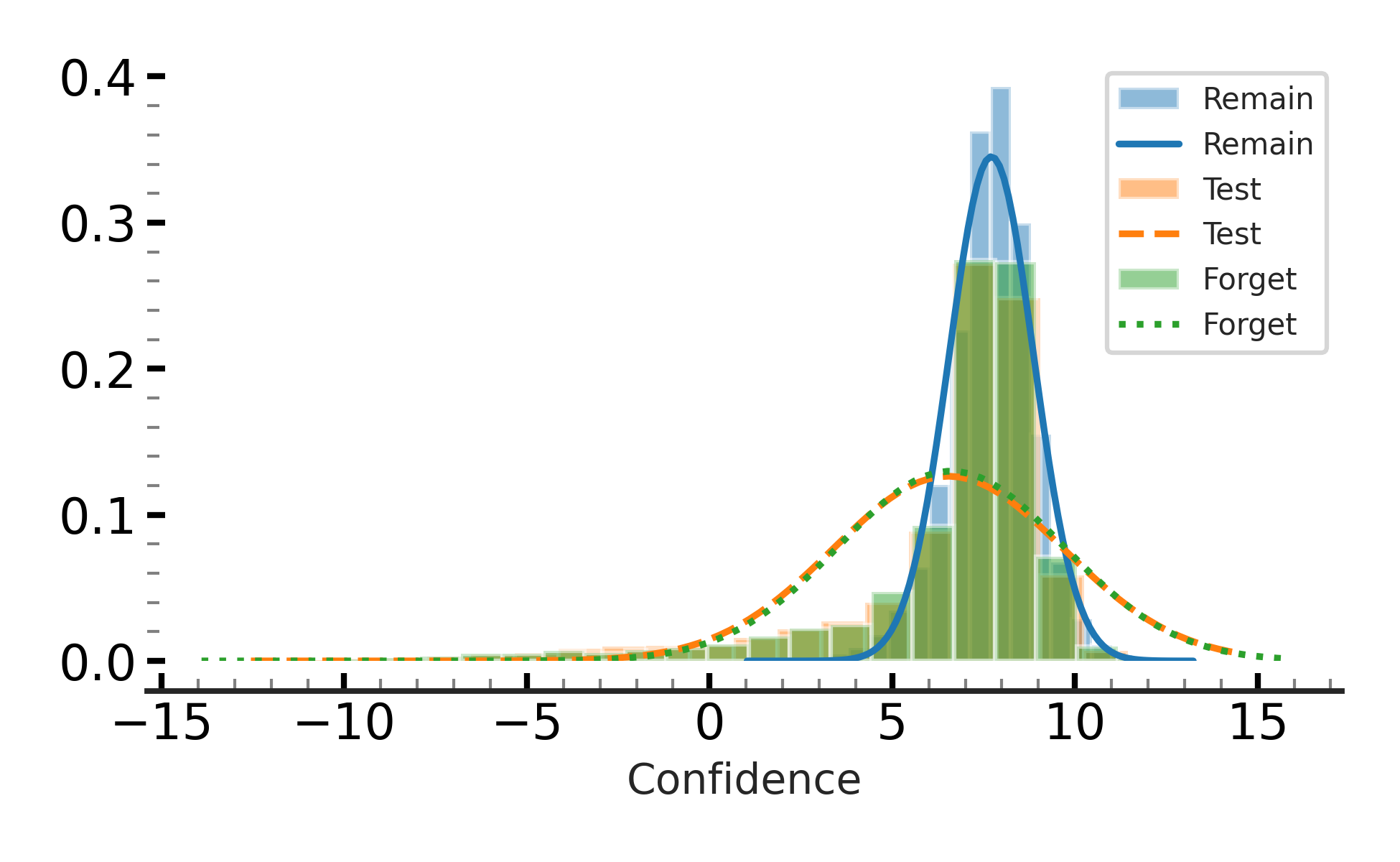}
\end{subfigure}
\begin{subfigure}
    \centering
    \includegraphics[width=.32\linewidth]{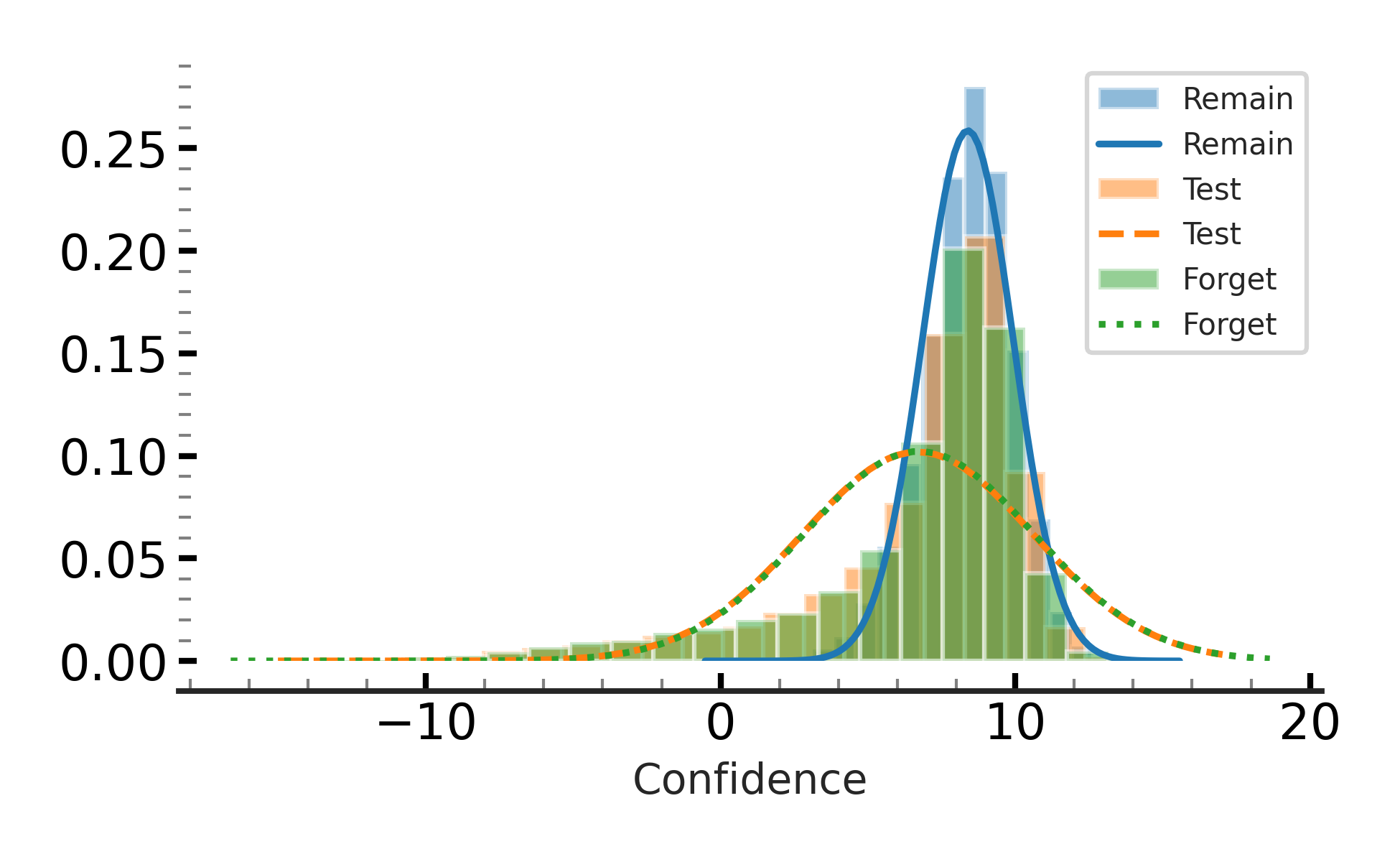}  
\end{subfigure}
\begin{subfigure}
    \centering
    \includegraphics[width=.32\linewidth]{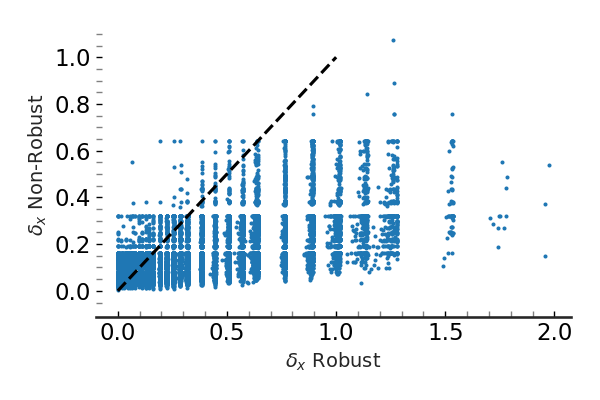}  
\end{subfigure}

\caption{
(left) These two plots show the histogram of confidence values of the retrained model on its predictions for the remaining set (Remain), test set (Test), and forget set (Forget) during the training, when the size of the forget set is $\%10$ (1st plot) and $\%50$ (2nd plot) of the training set. It also shows the Gaussian distributions fitted to each histogram. As the plots show the models perform similarly on the forget set and test set because to the retrained model they are unseen data from the same distribution. (right) This plot compares the $\delta_x$ value in definition~\ref{def:attack} for adversarial examples generated on the original ResNet-18 models (x-axis) and clipped ResNet-18 models (y-axis). The dashed line shows $x=y$ line and more than $97\%$ of the values fall bellow this line. } 
\label{fig:retrain_conf}
\end{figure}

\textbf{}
\section{Ablation Study (cont.)}
\label{apx:ablation_all}

In this section, we further discuss the ablation studies that were mentioned in \S~\ref{sec:ablation}. We also present other ablation studies on using transferred adversarial examples (Appendix~\ref{apx:transfer_attack}) and weaker adversarial attacks (Appendix~\ref{apx:weak_attack}) in \amun{}.

\subsection{Empirical Behavior of Retrained Models}
\label{apx:confidence}

As discussed in \S~\ref{sec:amun_motivation}, assuming the $\Dt$ and $\D$ come from the same distributions, we expect the prediction confidences of the models retrained on $\Dr$ to be similar on $\Df$ and $\Dt$, because both of these sets are considered unseen samples that belong the the same data distribution.  Figure~\ref{fig:retrain_conf} (left) shows the confidence scores for a ResNet-18 model that has been retrained on $\D - \Df$, where $\D$ is the training set of CIFAR-10 and the size of $\Df$ (randomly chosen from $\D$) is $10\%$ and $50\%$ of the size of $\D$ for the left and right sub-figures, respectively. To derive the confidence values, we use the following scaling on the logit values:

\begin{align*}
    \phi(f(x)_y) = \mathrm{log} \left( \frac{f(x)_y}{1-f(x)_y} \right),
\end{align*}

\noindent where $f(x)_y$ is the predicted probability for the correct class. This scaling has been used by~\cite{carlini2022membership} to transform the the prediction probabilities such that they can be better approximated with a normal distribution, which are indeed used by some of the SOTA MIA methods for predicting training samples from the test samples~\cite{carlini2022membership}. Figure~\ref{fig:retrain_conf} (left) shows these fitted normal distribution as well, which perfectly match for $\Df$ and $\Dt$.

\subsubsection{Confidence Values in Unlearned Models}
\label{apx:confidence_after}

In this section, we investigate the confidence values of the model, before and after using \amun{} for unlearning a subset of $10\%$ or $25\%$ of the training samples. For the original model (before unlearning), we expect the distribution of confidence values of samples in $\Df$ to be similar to those of the samples in $\Dr$ because they were both used as the training data and the model has used them similarly during training. However, this distribution is different for the test samples ($\Dt$), as the model has not seen them during the training phase. After unlearning, as discussed in section~\ref{sec:amun_motivation}, we expect the distribution of confidence values for $\Df$ and $\Dt$ to become more similar so that MIAs cannot distinguish them from each other. As Figure~\ref{fig:unlearn_conf} shows, for both unlearning $10\%$ (two leftmost subplots) and $50\%$ (two rightmost subplots), we observe the same behavior. Fur the original models (1st and 3rd subplot), the distribution for $\Df$ and $\Dr$ mathces exactly, but after using \amun{} (2nd and 4th subplot) the distribution for $\Df$ shifts toward that of $\Dt$.

\begin{figure}[t!]
\centering
\begin{subfigure}
    \centering
    \includegraphics[width=.23\linewidth]{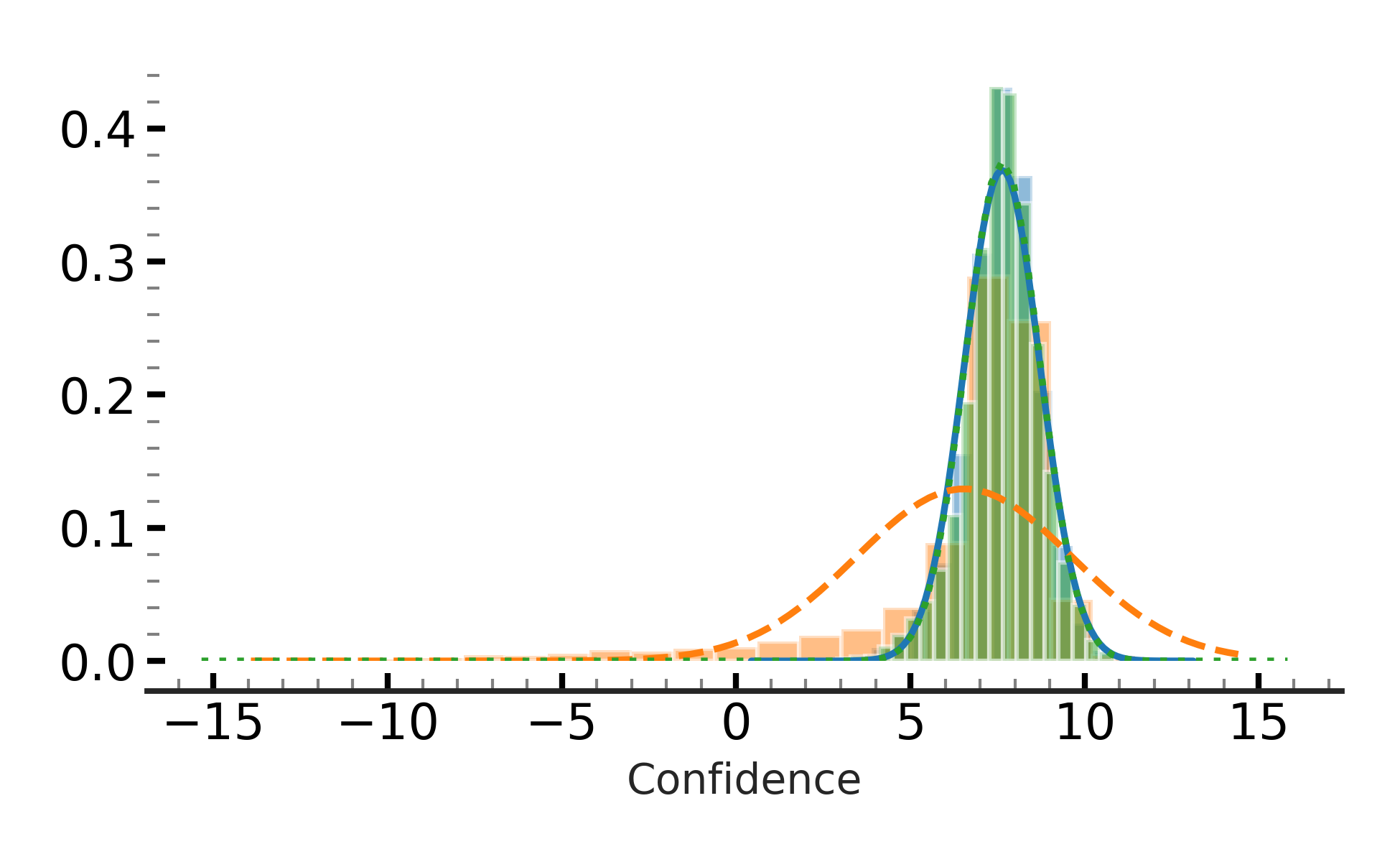}
    \includegraphics[width=.23\linewidth]{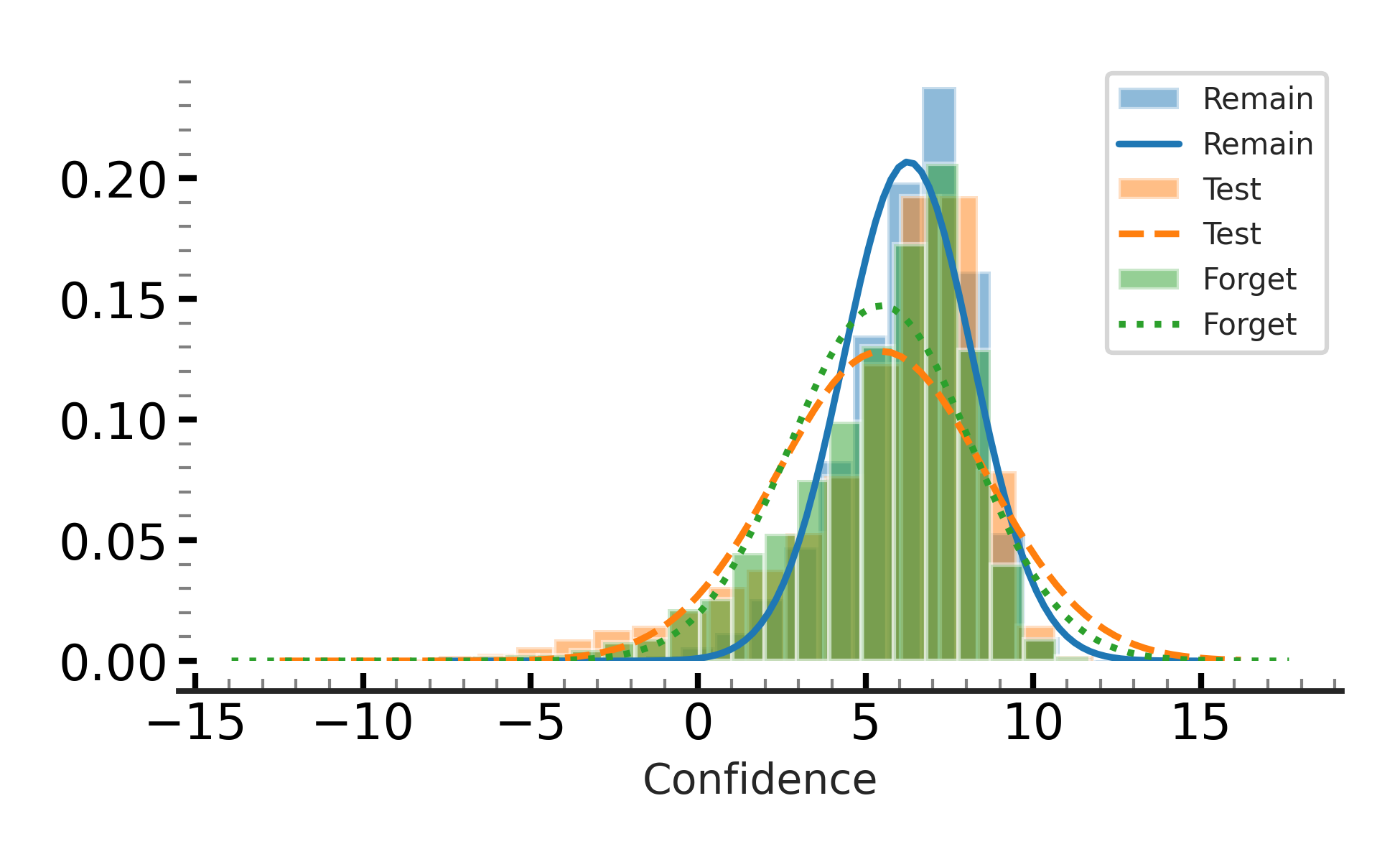}  
\end{subfigure}
~ 
\begin{subfigure}
    \centering
    \includegraphics[width=.23\linewidth]{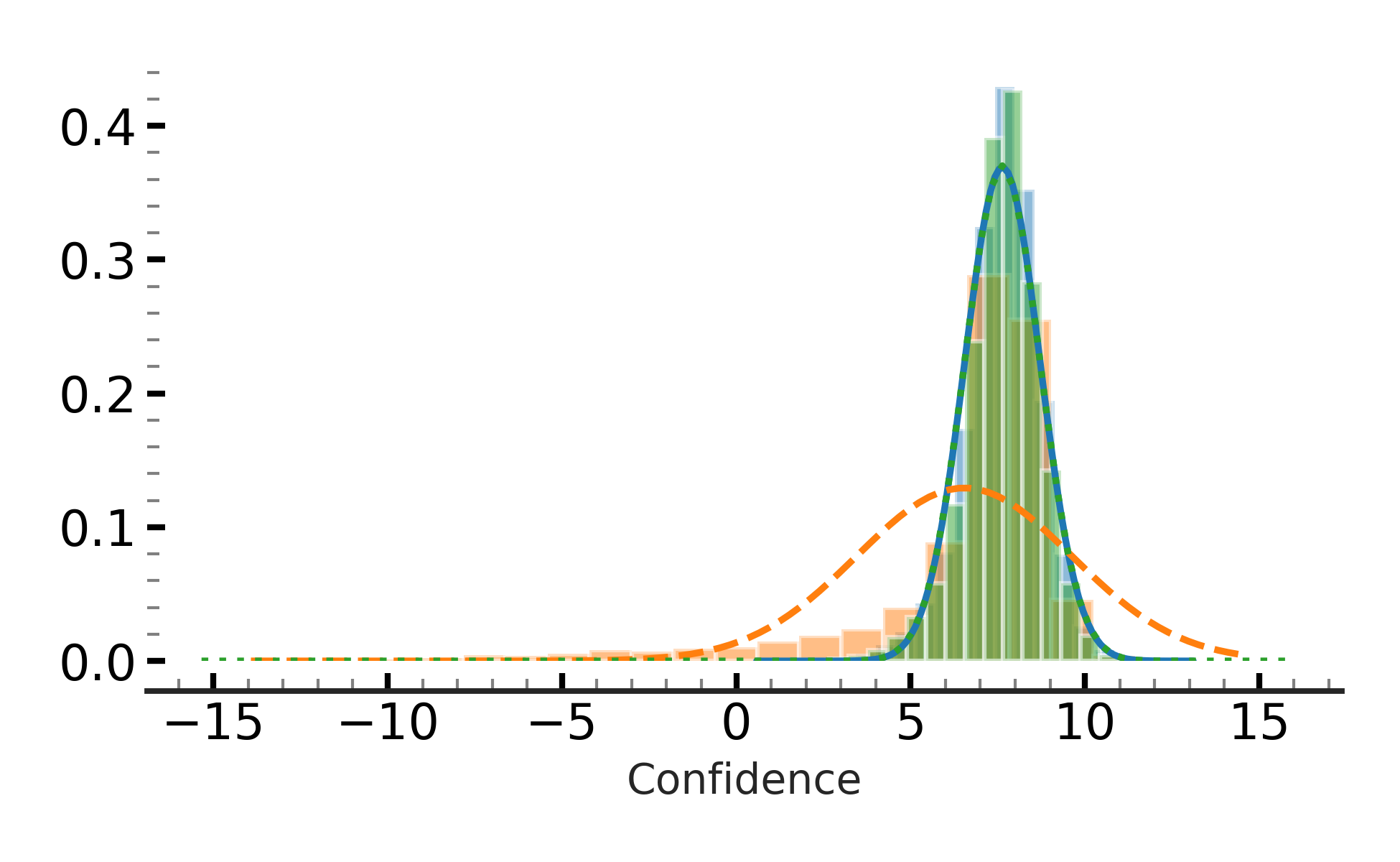}  
    \includegraphics[width=.23\linewidth]{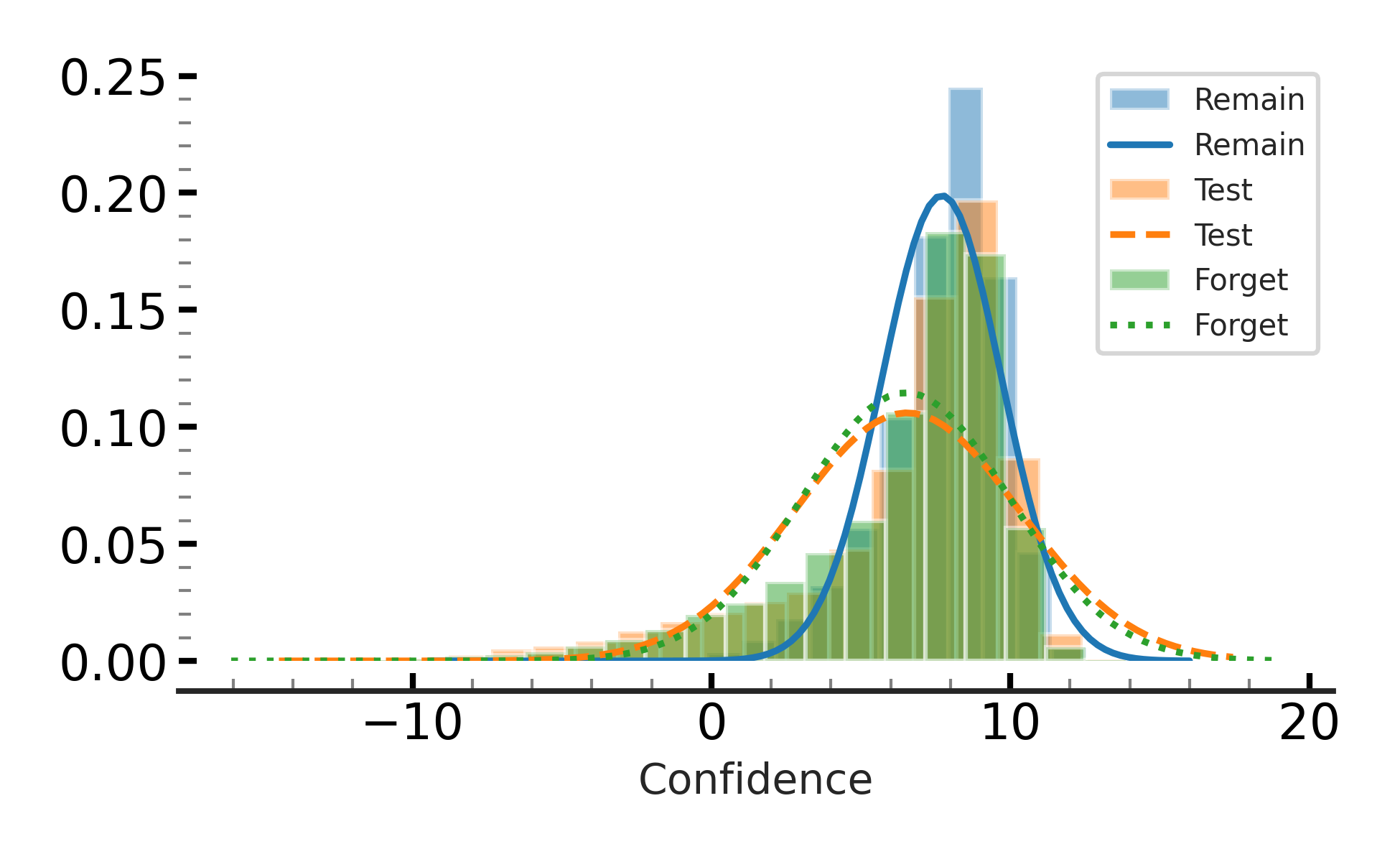}  
\end{subfigure}

\caption{The two left-most subplots show the confidence values before and after unlearning (using \amun{}) of $10\%$ of the training samples. The two right-most subplots show these confidence values for unlearning $50\%$ of the training samples. In both cases, the confidence values of samples in $\Df$ are similar to those of $\Dr$ and their fitted Gaussian distribution matches as expected. After using \amun{} for unlearning the samples in $\Df$, the confidence values on this set gets more similar to the test (unseen) samples.} 
\label{fig:unlearn_conf}
\end{figure}

\subsection{Adversarially Robust Models (cont.)}
\label{apx:ablation_robust}

As discussed in \S~\ref{sec:ablation-robust}, we also evaluatee the effectiveness of \amun{} when the trained model is adversarially robust. For this experiment, we used the ResNet-18 models with $1$-Lipschitz convolutional and fully-connected layers, which are shown to be significantly more robust than the original ResNet-18 models. In Table~\ref{tab:rmia-clipped-remFalse}, we showed the results for unlearning $10\%$ and $50\%$ of the samples from the robust ResNet-18 models trained on CIFAR-10, in the case where $\Dr$ is not accessible. In Table~\ref{tab:rmia-clipped-remTrue}, we showed the corresponding results when the unlearning methods have access to $\Dr$. As the results show, similar to the results discussed in \S~\ref{sec:ablation-robust}, \amun{} effectively unlearns $\Df$ for either of the sizes of the this set.

\begin{table}[th!]
\begin{center}
\begin{small}
\begin{sc}
\resizebox{0.66\columnwidth}{!}{
\begin{tabular}{@{} l  c c c | c c c @{}}
 \toprule


& \multicolumn{3}{@{}c}{\textbf{Random Forget ($10\%$)}} & \multicolumn{3}{@{}c}{\textbf{Random Forget ($50\%$)}} \\\addlinespace[0.3em]

 \multicolumn{1}{c}{\scriptsize \textbf{}} & 
 \multicolumn{1}{c}{\scriptsize FT AUC} & 
 \multicolumn{1}{c}{\scriptsize FR AUC} & 
 \multicolumn{1}{c}{\scriptsize Test Acc} & 
 \multicolumn{1}{c}{\scriptsize FT AUC} &
 \multicolumn{1}{c}{\scriptsize FR AUC} & 
 \multicolumn{1}{c}{\scriptsize Test Acc} 
 \\\addlinespace[0.3em]

 \cmidrule(r){2-4}
 \cmidrule(r){5-7}

 Retrain & $49.95$ {\tiny $\pm 0.24$} & $54.08 $ {\tiny $\pm 0.16$} & $89.01$ {\tiny $\pm 0.21$} & 
 
 $50.19$ {\tiny $\pm 0.15$}  & $55.61$ {\tiny $\pm 0.05$} & $85.76$ {\tiny $\pm 0.41$}
 
 \\\addlinespace[0.3em]
  \cmidrule(r){1-7}

 \textbf{Amun} & $49.12$ {\tiny $\pm 0.19$} & $53.60$ {\tiny $\pm 0.31$} & $86.94$ {\tiny $\pm 0.56$} & 
 
 $49.41$ {\tiny $\pm 0.25$} & $54.22$ {\tiny $\pm 0.16$} & $87.38$ {\tiny $\pm 0.39$}
 \\\addlinespace[0.3em]

\bottomrule
\end{tabular}
}
\end{sc}
\end{small}
\end{center}
\caption{\footnotesize {\bf Unlearning on adversarially robust models.} Evaluating the effectiveness of \amun{} in unlearning $10\%$ and $50\%$ of the training samples when the models are adversarially robust and we have access to $\Dr$. For this experiment we use models with controlled Lipschitz constant which makes them provably and empirically more robust to adversarial examples. }
\label{tab:rmia-clipped-remTrue}

\end{table}

We also evaluated \amun{} for unlearning in models that are adversarially trained. we performed our analysis on ResNet-18 models trained using TRADES loss~\cite{zhang2019theoretically} on CIFAR-10. We performed the experiments for unlearning 10\% of the dataset in both cases where $\Dr$ is accessible and not. As the results in Table~\ref{tab:trades_tinynet_remTrue} show, in both settings \amun{} is effective in unlearning the forget samples and achieving a low gap with the retrained models. This gap is obviously smaller when there is access to $\Dr$.

\begin{table*}[th!]
\begin{center}
\begin{small}
\begin{sc}
\begin{tabular}{@{} l  c c  c  c c @{}}
 \toprule



 \multicolumn{1}{c}{\scriptsize \textbf{}} & 
 \multicolumn{1}{c}{\scriptsize Unlearn Acc} & 
 \multicolumn{1}{c}{\scriptsize Retain Acc} & 
 \multicolumn{1}{c}{\scriptsize Test Acc} & 
 \multicolumn{1}{c}{\scriptsize FT AUC} & 
 \multicolumn{1}{c}{\scriptsize Avg. Gap} 
 \\\addlinespace[0.3em]

 \cmidrule(r){2-6}



 
 

 Retrain & $82.33$ {\tiny $\pm 0.39$} & $94.22 $ {\tiny $\pm 0.21$} & $81.72$ {\tiny $\pm 0.36$} & $50.04$ {\tiny $\pm 0.34$} & $0.00$ 
 \\\addlinespace[0.3em]

 \textbf{Amun}$_{\, \mathrm{With} \, \Dr}$  & $82.65$  { \tiny $\pm 0.62$ }  & $94.33$  { \tiny $\pm 0.84$ }  & $84.99$  { \tiny $\pm 0.91$ }  & $47.18$  { \tiny $\pm 0.50$ }  & $1.02$  { \tiny $\pm 0.18$ } 
 \\\addlinespace[0.3em]

 \textbf{Amun}$_{\, \mathrm{No} \, \Dr}$  & $81.38$  { \tiny $\pm 0.10$ }  & $87.45$  { \tiny $\pm 0.54$ }  & $79.74$  { \tiny $\pm 0.31$ }  & $54.61$  { \tiny $\pm 0.23$ }  & $3.57$  { \tiny $\pm 0.24$ } 
 \\\addlinespace[0.3em]

 \bottomrule
\end{tabular}
\end{sc}
\end{small}
\end{center}
\caption{\footnotesize {\bf Unlearning with access to $\Dr$.} Evaluating \amun{} when applied to ResNet-18 models trained using adversarial training. TRADES loss is used to train the models, and the unlearning is done on $10\%$ of CIFAR-10 Dataset ($\D$). Avg. Gap is used for evaluation (lower is better). The result has been reported in two cases: with and without access to $\Dr$. As the results show, \amun{} is effective in both cases, with slight degradation in the more difficult setting of no access to $\Dr$.}
\label{tab:trades_tinynet_remTrue}

\end{table*}

\subsection{Fine-tuning on Adversarial Examples (cont.)}
\label{apx:ablation-finetune}

As explained in \S~\ref{sec:ablation-finetune}, we evaluate the effect of fine-tuning on test accuracy of a ResNet-18 model that is trained on CIFAR-10, when $\Da$ is substituted with other datasets that vary in the choice of samples or their labels (see \S~\ref{sec:ablation-finetune} for details). In Figure~\ref{fig:fine_tune_10} we presented the results when $\Df$ contains $10\%$ of the samples in $\D$. We also present the results for the case where  $\Df$ contains $50\%$ of the samples in $\D$ in Figure~\ref{fig:fine_tune_50}. As the figure shows, even for the case where we fine-tune the trained models on only $\Da$ which contains the adversarial examples corresponding to $50\%$ of the samples in $\D$ (right-most sub-figure), there is no significant loss in models' accuracy. This is due to the fact that the samples in $\Da$, in contrast to the other constructed datasets, belong to the natural distribution learned by the trained model. To generate the results in both Figures~\ref{fig:fine_tune_10} and~\ref{fig:fine_tune_50}, we fine-tuned the trained ResNet-18 models on all the datasets (see \S~\ref{sec:ablation-finetune} for details) for 20 epochs. We used a learning rate of $0.01$ with a scheduler that scales the learning rate by $0.1$ every 5 epochs.

\begin{figure*}[t!]
\centering
\includegraphics[width=.98\linewidth]{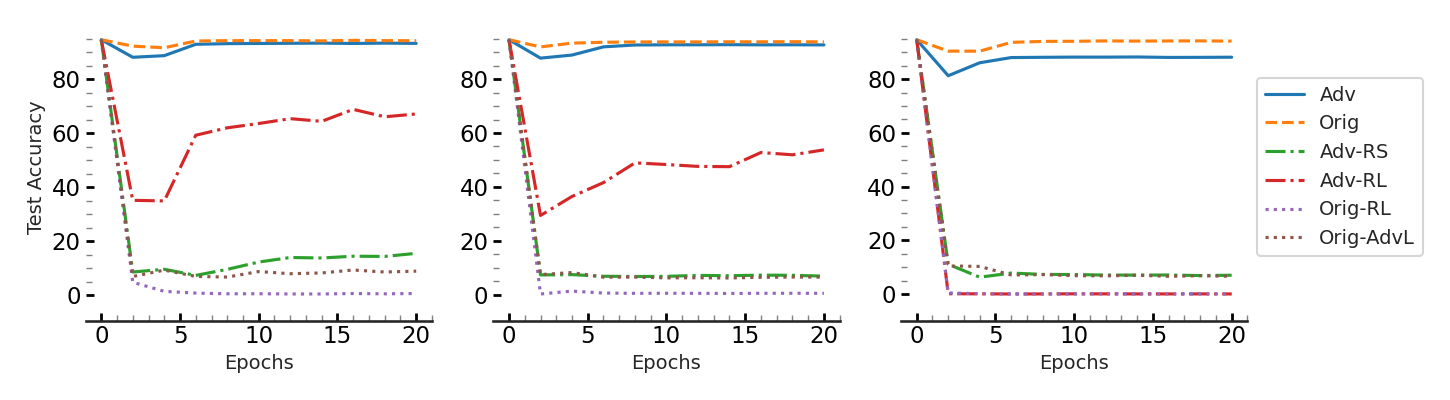}
\caption{This figure shows the effect of fine-tuning on test accuracy of a ResNet-18 model that is trained on CIFAR-10, when the dataset for fine-tuning changes (see \S~\ref{sec:ablation} for details). Let $\Df$ contain $50\%$ of the samples in $\D$ and $\Da$ be the set of adversarial examples constructed using Algorithm~\ref{alg:advset}. \texttt{Adv}, from the left sub-figure to right one, shows the results when $\D \cup \Da$, $\Df \cup \Da$, and $\Da$ is used for fine-tuning the model, respectively. \texttt{Orig}, \texttt{Adv-RS}, \texttt{Adv-RL}, \texttt{Orig-RL}, and \texttt{Orig-AdvL} shows the results when $\Da$ for each of these sub-figures is replace by $\Df$, $\Da_{RS}$, $\Da_{RL}$, $\D_{RL}$, and $\D_{AdvL}$, accordingly. As the figure shows, the specific use of adversarial examples with the mis-predicted labels matters in keeping the model's test accuracy because $\Da$, in contrast to the other constructed datasets belong to the natural distribution learned by the trained model.} 
\label{fig:fine_tune_50}
\end{figure*}

\subsection{Transferred Adversarial Examples}
\label{apx:transfer_attack}

One of the intriguing properties of adversarial attacks is their transferability to other models~\cite{papernot2016transferability,liu2016delving}; Adversarial examples generated on a trained model (source model) mostly transfer successfully to other models (target models). This success rate of the transferred adversarial examples increases if the source model and target model have the same architecture~\cite{papernot2016transferability}. There are other studies that can be used to increase the success rate of this type of attack~\cite{zhao2021success,zhang2022improving,chen2023rethinking,ebrahimpour2024lotos}. In this section, we are interested to see if using the the adversarial examples generated using Algorithm~\ref{alg:advset} for a given model trained on some dataset $\D$ can be used as the $\Da$ dataset for unlearning a portion of $\D$ from a separately trained model. The advantage of using adversarial examples generated for another model is saving the computation cost for other trained models. For this purpose, we train three ResNet-18 models separately on CIFAR-10, we generate the adversarial examples for each of these models using Algorithm~\ref{alg:advset}. We use \amun{} for unlearning $10\%$ and $50\%$ of CIFAR-10 from either of these models, but instead of their adversarial samples, we use the ones derived from the other models. The results in Table~\ref{tab:transfer} shows that using transferred adversarial examples leads to lower performance, specially for the case where there is no access to $\Dr$. All the values for test accuracy are also lower compared to using adversarial examples from the model itself because these adversarial examples from the other models do not all belong to the natural distribution of the model and they do not even always transfer to the other models. Still the results are comparable to the prior SOTA methods in unlearning, and even in the case of no access to $\Dr$ outperforms all prior methods.

\begin{table*}[th!]
\begin{center}
\begin{small}
\begin{sc}
\resizebox{0.98\textwidth}{!}{
\begin{tabular}{@{} l  c c c c c c | c c c c c c @{}}
 \toprule


& \multicolumn{6}{@{}c}{\textbf{With access to $\Dr$}} & \multicolumn{6}{@{}c}{\textbf{No access to $\Dr$}} \\\addlinespace[0.3em]

& \multicolumn{3}{@{}c}{\textbf{Random Forget ($10\%$)}} & \multicolumn{3}{@{}c}{\textbf{Random Forget ($50\%$)}} 
& \multicolumn{3}{@{}c}{\textbf{Random Forget ($10\%$)}} & \multicolumn{3}{@{}c}{\textbf{Random Forget ($50\%$)}} \\\addlinespace[0.3em]

 \multicolumn{1}{c}{\scriptsize \textbf{}} & 
 \multicolumn{1}{c}{\scriptsize Test Acc} & 
 \multicolumn{1}{c}{\scriptsize FT AUC} & 
 \multicolumn{1}{c}{\scriptsize Avg. Gap} & 
 \multicolumn{1}{c}{\scriptsize Test AUC} &
 \multicolumn{1}{c}{\scriptsize FT AUC} & 
 \multicolumn{1}{c}{\scriptsize Avg. Gap} &

  \multicolumn{1}{c}{\scriptsize Test Acc} & 
 \multicolumn{1}{c}{\scriptsize FT AUC} & 
 \multicolumn{1}{c}{\scriptsize Avg. Gap} & 
 \multicolumn{1}{c}{\scriptsize Test AUC} &
 \multicolumn{1}{c}{\scriptsize FT AUC} & 
 \multicolumn{1}{c}{\scriptsize Avg. Gap} 
 \\\addlinespace[0.3em]

 \cmidrule(r){2-4}
 \cmidrule(r){5-7}
 \cmidrule(r){8-10}
 \cmidrule(r){11-13}



 \textbf{Self} & $93.45$ {\tiny $\pm 0.22$} & $50.18 $ {\tiny $\pm 0.36$} & $0.62$ { \tiny $\pm 0.05$ } & 
 
 $92.39$ {\tiny $\pm 0.04$}  & $49.99$ {\tiny $\pm 0.18$} & $0.33$ { \tiny $\pm 0.03$ } &

 $91.67$ {\tiny $\pm 0.04$}  & $52.24$ {\tiny $\pm 0.23$} & $1.94$ { \tiny $\pm 0.13$ } &

 $89.43$ {\tiny $\pm 0.19$}  & $52.60$ {\tiny $\pm 0.22$} & $2.51$ { \tiny $\pm 0.09$ } 
 
 \\\addlinespace[0.3em]

 \textbf{Others} & $92.64$  { \tiny $\pm 0.09$ }  & $48.70$  { \tiny $\pm 0.59$ }  & $1.57$  { \tiny $\pm 0.12$ }  & 
 
$91.49$  { \tiny $\pm 0.03$ }  & $47.36$  { \tiny $\pm 0.63$ }  & $1.15$  { \tiny $\pm 0.23$ }  &

 $90.56$  { \tiny $\pm 0.28$ }  & $48.29$  { \tiny $\pm 0.22$ }  & $3.07$  { \tiny $\pm 0.15$ }  &
 
$83.61$  { \tiny $\pm 0.45$ }  & $51.11$  { \tiny $\pm 0.04$ }  & $6.70$  { \tiny $\pm 0.33$ }

 \\\addlinespace[0.3em]

 

 
 
 \bottomrule
\end{tabular}
}
\end{sc}
\end{small}
\end{center}
\caption{\footnotesize {\bf Transferred adversarial examples.} Comparing the effectiveness of unlearning when instead of using adversarial examples of the model, we use adversarial examples generated using Algorithm~\ref{alg:advset} on separately trained models with the same architecture. As the results show, relying on transferred adversarial examples in \amun{} leads to worse results, specially for test accuracy because the adversarial examples do not necessary belong to the natural distribution learned by the model. However, even by using these transferred adversarial examples \amun{} outperforms prior SOTA unlearning methods, specially when there is no access to $\Dr$.}
\label{tab:transfer}

\end{table*}

\begin{figure*}[t!]
\centering
\includegraphics[width=.98\linewidth]{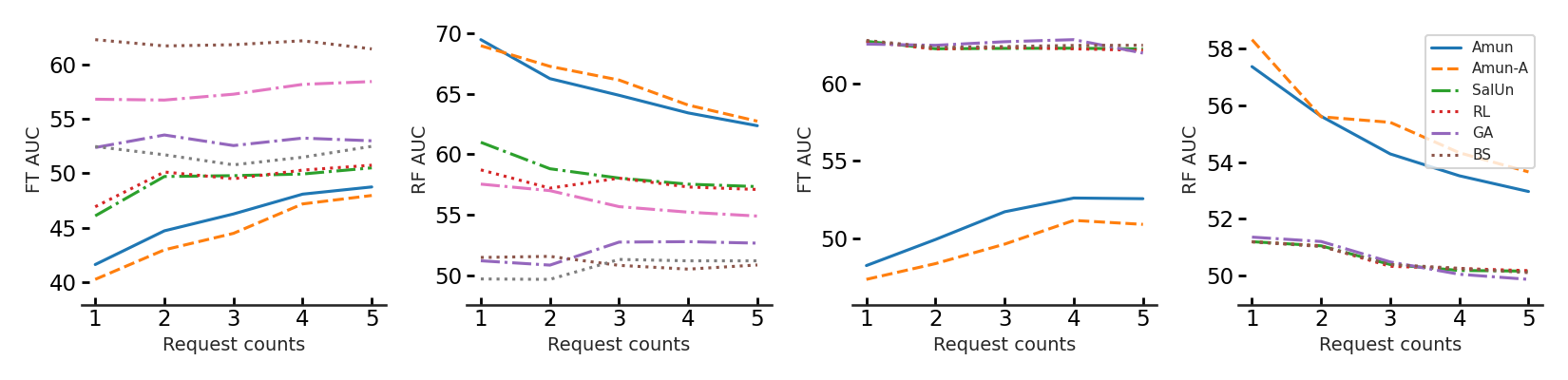}
\caption{This figure shows both \texttt{FT AUC} and \texttt{RF AUC} components of the plots presented in Figure~\ref{fig:adaptive}. The two left-most sub-figures show these values along the number of unlearning requests for the case where there is access to $\Dr$ and the two right-most ones show these values when there is no access to $\Dr$. } 
\label{fig:adaptive_FT_RF}
\end{figure*}

\subsection{Weak Attacks}
\label{apx:weak_attack}

In this section we evaluate the effectiveness of using weaker attacks in Algorithm~\ref{alg:advset}. For this purpose, we perform the unlearning on a ResNet-18 model trained on CIFAR-10 in all unlearning settings mentioned in \S~\ref{sec:amun_evaluation}, and compare the results with the default choice of PGD-50 in \amun{}. The weaker attack that we use is a variation of FFGSM~\cite{wong2020fast}, which itself is a variant of FGSM~\cite{goodfellow2014explaining}. FGSM takes steps toward the gradient sign at a given sample to find adversarial samples. FFGSM takes a small step toward a random direction first, and then proceeds with FGSM. To adapt these method to the format of Algorithm~\ref{alg:advset} we start with FGSM attack; we find the gradient sign and start to move toward that direction in steps of size $\epsilon$ until we find an adversarial example. If the adversarial example is not found after a few iteration of the \texttt{While} loop, we restart the value of $\epsilon$ and add a small random perturbation before the next round of FGSM attack and the \texttt{While} loop. We continue this procedure to find an adversarial sample. After deriving a new set of adversarial examples using this methods, we performed a separate round hyper-parameter tuning for unlearning with the new attack to have a fair comparison. It is notable to mention that this leads to a much faster attack because we only compute the gradient once for each round round of FGSM (at the beginning or after each addition of random perturbation and restarting FGSM). Table~\ref{tab:fgsm} shows the comparison of the results with the original version of \amun{} that uses PGD-50. As the results show, using this weaker attack leads to worse results; however, they still outperform prior SOTA methods in unlearning, specially in the setting where there is no access to $\Dr$ and the size of $\Df$ is $50\%$ of $\D$.

For each image in CIFAR-10, Figure~\ref{fig:fgsm_ratio} shows $\delta_x$ (see Definition~\ref{def:attack}) for the adversarial examples that Algorithm~\ref{alg:advset} finds using PGD-50 ($x$-axis) and FFGSM ($y$-axis). The dashed line shows the $x=y$ line for the reference. As the figure shows $\delta_x$ is much smaller for PGD-50. This value is smaller for FFGSM for less than $4\%$ of the images, but still even for those images, the value of $\delta_x$ for PGD-50 is very small, compared to the range of values that are required for FFGSM in many cases. This, we believe, is the main reason behind worse performance when using FFGSM. However, still note that the adversarial examples that are found using FFGSM belong to the natural distribution of the trained model and therefore fine-tuning the model on these samples does not lead to noticable deterioration of the test accuracy, while achieving reasonable \texttt{FT AUC} score. Indeed this larger distance of the adversarial examples with the original samples in $\Df$, leads to better performance of \amun{} when it does not include $\Df$ when fine-tuning the model, because the difference in the predicted logits compared to the $\delta_x$ leads to under-estimation of the local Lipschitz constant and therefore, the model is able to fit perfectly to both the original samples and its corresponding adversarial sample without changing much. This consequently leads to a larger value of \texttt{FT AUC} score.

\begin{table*}[th!]
\begin{center}
\begin{small}
\begin{sc}
\resizebox{0.98\textwidth}{!}{
\begin{tabular}{@{} l  c c c c c c | c c c c c c @{}}
 \toprule


& \multicolumn{6}{@{}c}{\textbf{With access to $\Dr$}} & \multicolumn{6}{@{}c}{\textbf{No access to $\Dr$}} \\\addlinespace[0.3em]

& \multicolumn{3}{@{}c}{\textbf{Random Forget ($10\%$)}} & \multicolumn{3}{@{}c}{\textbf{Random Forget ($50\%$)}} 
& \multicolumn{3}{@{}c}{\textbf{Random Forget ($10\%$)}} & \multicolumn{3}{@{}c}{\textbf{Random Forget ($50\%$)}} \\\addlinespace[0.3em]

 \multicolumn{1}{c}{\scriptsize \textbf{}} & 
 \multicolumn{1}{c}{\scriptsize Test Acc} & 
 \multicolumn{1}{c}{\scriptsize FT AUC} & 
 \multicolumn{1}{c}{\scriptsize Avg. Gap} & 
 \multicolumn{1}{c}{\scriptsize Test AUC} &
 \multicolumn{1}{c}{\scriptsize FT AUC} & 
 \multicolumn{1}{c}{\scriptsize Avg. Gap} &

  \multicolumn{1}{c}{\scriptsize Test Acc} & 
 \multicolumn{1}{c}{\scriptsize FT AUC} & 
 \multicolumn{1}{c}{\scriptsize Avg. Gap} & 
 \multicolumn{1}{c}{\scriptsize Test AUC} &
 \multicolumn{1}{c}{\scriptsize FT AUC} & 
 \multicolumn{1}{c}{\scriptsize Avg. Gap} 
 \\\addlinespace[0.3em]

 \cmidrule(r){2-4}
 \cmidrule(r){5-7}
 \cmidrule(r){8-10}
 \cmidrule(r){11-13}



 \textbf{PGD-50} & $93.45$ {\tiny $\pm 0.22$} & $50.18 $ {\tiny $\pm 0.36$} & $0.62$ { \tiny $\pm 0.05$ } & 
 
 $92.39$ {\tiny $\pm 0.04$}  & $49.99$ {\tiny $\pm 0.18$} & $0.33$ { \tiny $\pm 0.03$ } &

 $91.67$ {\tiny $\pm 0.04$}  & $52.24$ {\tiny $\pm 0.23$} & $1.94$ { \tiny $\pm 0.13$ } &

 $89.43$ {\tiny $\pm 0.19$}  & $52.60$ {\tiny $\pm 0.22$} & $2.51$ { \tiny $\pm 0.09$ } 
 
 \\\addlinespace[0.3em]

 \textbf{FGSM} & $93.87$  { \tiny $\pm 0.16$ }  & $50.64$  { \tiny $\pm 0.51$ }  & $0.92$  { \tiny $\pm 0.25$ }  & 
 
 $89.41$  { \tiny $\pm 1.01$ }  & $50.93$  { \tiny $\pm 0.46$ }  & $1.81$  { \tiny $\pm 0.77$ } &

 $92.14$  { \tiny $\pm 0.28$ }  & $56.58$  { \tiny $\pm 1.05$ }  & $3.46$  { \tiny $\pm 0.36$ } &
 
 $90.12$  { \tiny $\pm 0.28$ }  & $54.54$  { \tiny $\pm 0.47$ }  & $3.29$  { \tiny $\pm 0.10$ }

 \\\addlinespace[0.3em]

 

 
 
 \bottomrule
\end{tabular}
}
\end{sc}
\end{small}
\end{center}
\caption{\footnotesize {\bf Using weaker attacks.} Comparing the effectiveness of unlearning when PGD-10 in Algorithm~\ref{alg:advset} is replaced with a variant of FGSM attack, which is considered to be significantly weaker and leads to finding adversarial examples at a much higher distance to the original samples. We evaluate unlearning $10\%$ and $50\%$ of the training samples in CIFAR-10 from a trained ResNet-18 model. As the results show, in both settings of unlearning (with access to $\Dr$ and no access to $\Dr$), using the weaker attack does not perform as well as the original method. However, it still outperforms prior SOTA unlearning methods.}
\label{tab:fgsm}

\end{table*}

\section{Comparison Using Prior Evaluation Methods}
\label{apx:svc_mia}

In this section we perform similar comparisons to what we presented in section~\ref{sec:amun_results}, but instead of \texttt{FT AUC}, we use the same MIA used by prior SOTA methods in unlearning for evaluations. As mentioned in section~\ref{sec:amun_metrics}, we refer to the score derived by this MIA as \texttt{MIS}.

\begin{table*}[th!]
\begin{center}
\begin{small}
\begin{sc}
\resizebox{\textwidth}{!}{
\begin{tabular}{@{} l  c c  c  c c | c c c c c @{}}
 \toprule


& \multicolumn{5}{@{}c}{\textbf{Random Forget ($10\%$)}} & \multicolumn{5}{@{}c}{\textbf{Random Forget ($50\%$)}} \\\addlinespace[0.3em]

 \multicolumn{1}{c}{\scriptsize \textbf{}} & 
 \multicolumn{1}{c}{\scriptsize Unlearn Acc} & 
 \multicolumn{1}{c}{\scriptsize Retain Acc} & 
 \multicolumn{1}{c}{\scriptsize Test Acc} & 
 \multicolumn{1}{c}{\scriptsize MIS} & 
 \multicolumn{1}{c}{\scriptsize Avg. Gap} & 
 \multicolumn{1}{c}{\scriptsize Unlearn Acc} &
 \multicolumn{1}{c}{\scriptsize Retain Acc} & 
 \multicolumn{1}{c}{\scriptsize Test Acc} & 
 \multicolumn{1}{c}{\scriptsize  MIS} & 
 \multicolumn{1}{c}{\scriptsize Avg. Gap} 
 \\\addlinespace[0.3em]

 \cmidrule(r){2-6}
 \cmidrule(r){7-11}



 Retrain & $94.49$ {\tiny $\pm 0.20$} & $100.0 $ {\tiny $\pm 0.00$} & $94.33$ {\tiny $\pm 0.18$} & $12.53$ {\tiny $\pm 0.32$} & $0.00$  
 & 
 
 $92.09$ {\tiny $\pm 0.37$}  & $100.0$ {\tiny $\pm 0.00$} & $91.85$ {\tiny $\pm 0.33$} & $16.78$ {\tiny $\pm 0.37$} & $0.00$
 \\\addlinespace[0.3em]
  \cmidrule(r){1-11}

 FT  & $95.16$ { \tiny $\pm 0.29$ }  & $96.64$ { \tiny $\pm 0.25$ }  & $92.21$ { \tiny $\pm 0.27$ }  & $11.33$ { \tiny $\pm 0.35$ }  & $1.84$ { \tiny $\pm 0.10$ } 
 & 
 
 $94.24$ { \tiny $\pm 0.30$ }  & $95.22$ { \tiny $\pm 0.31$ }  & $91.21$ { \tiny $\pm 0.33$ }  & $12.10$ { \tiny $\pm 0.72$ }  & $3.06$ { \tiny $\pm 0.24$ } 
 \\\addlinespace[0.3em]
 
 RL & $99.22$ { \tiny $\pm 0.19$ }  & $99.99$ { \tiny $\pm 0.01$ }  & $94.10$ { \tiny $\pm 0.11$ }  & $10.94$ { \tiny $\pm 0.45$ }  & $1.64$ { \tiny $\pm 0.19$ } 
 & 
 
 $92.98$ { \tiny $\pm 1.07$ }  & $94.83$ { \tiny $\pm 1.04$ }  & $89.19$ { \tiny $\pm 0.74$ }  & $12.48$ { \tiny $\pm 0.90$ }  & $3.29$ { \tiny $\pm 0.04$ } 
 \\\addlinespace[0.3em]

  GA & $98.94$ { \tiny $\pm 1.39$ }  & $99.22$ { \tiny $\pm 1.31$ }  & $93.39$ { \tiny $\pm 1.18$ }  & $4.21$ { \tiny $\pm 5.25$ }  & $3.62$ { \tiny $\pm 1.04$ } 
  & 
 
 $99.94$ { \tiny $\pm 0.09$ }  & $99.95$ { \tiny $\pm 0.08$ }  & $94.36$ { \tiny $\pm 0.31$ }  & $0.62$ { \tiny $\pm 0.30$ }  & $6.64$ { \tiny $\pm 0.15$ } 
 \\\addlinespace[0.3em]

 BS & $99.14$ { \tiny $\pm 0.31$ }  & $99.89$ { \tiny $\pm 0.06$ }  & $93.04$ { \tiny $\pm 0.14$ }  & $5.50$ { \tiny $\pm 0.39$ }  & $3.27$ { \tiny $\pm 0.13$ } 
 & 
  
  $100.00$ { \tiny $\pm 0.00$ }  & $100.00$ { \tiny $\pm 0.00$ }  & $94.62$ { \tiny $\pm 0.08$ }  & $0.40$ { \tiny $\pm 0.05$ }  & $6.77$ { \tiny $\pm 0.03$ } 
 \\\addlinespace[0.3em]

  $l_1$-Sparse & $94.29$ { \tiny $\pm 0.34$ }  & $95.63$ { \tiny $\pm 0.16$ }  & $91.55$ { \tiny $\pm 0.17$ }  & $12.03$ { \tiny $\pm 1.92$ }  & $2.26$ { \tiny $\pm 0.26$ } 
  & 
  
  $92.63$ { \tiny $\pm 0.13$ }  & $95.02$ { \tiny $\pm 0.10$ }  & $89.56$ { \tiny $\pm 0.08$ }  & $12.03$ { \tiny $\pm 0.39$ }  & $3.14$ { \tiny $\pm 0.17$ } 
 \\\addlinespace[0.3em]

  SalUn & $99.25$ { \tiny $\pm 0.12$ }  & $99.99$ { \tiny $\pm 0.01$ }  & $94.11$ { \tiny $\pm 0.13$ }  & $11.29$ { \tiny $\pm 0.56$ }  & $1.56$ { \tiny $\pm 0.20$ } 
  & 

  $95.69$ { \tiny $\pm 0.80$ }  & $97.26$ { \tiny $\pm 0.79$ }  & $91.55$ { \tiny $\pm 0.59$ }  & $11.27$ { \tiny $\pm 0.94$ }  & $3.06$ { \tiny $\pm 0.12$ } 
 \\\addlinespace[0.5em]

 \textbf{Amun} & $95.45$ { \tiny $\pm 0.19$ }  & $99.57$ { \tiny $\pm 0.00$ }  & $93.45$ { \tiny $\pm 0.22$ }  & $12.55$ { \tiny $\pm 0.08$ }  & $\bf 0.59$ { \tiny $\pm 0.09$ } 
 & 
 
  $93.50$ { \tiny $\pm 0.09$ }  & $99.71$ { \tiny $\pm 0.01$ }  & $92.39$ { \tiny $\pm 0.04$ }  & $13.53$ { \tiny $\pm 0.19$ }  & $\bf 1.37$ { \tiny $\pm 0.07$ } 
 \\\addlinespace[0.3em]

 \textbf{Amun}$_{+SalUn}$ & $94.73$ { \tiny $\pm 0.07$ }  & $99.92$ { \tiny $\pm 0.01$ }  & $93.95$ { \tiny $\pm 0.18$ }  & $14.23$ { \tiny $\pm 0.40$ }  & $\underline{0.60}$ { \tiny $\pm 0.10$ } 
 & 
 
$93.56$ { \tiny $\pm 0.07$ }  & $99.72$ { \tiny $\pm 0.02$ }  & $92.52$ { \tiny $\pm 0.20$ }  & $13.33$ { \tiny $\pm 0.10$ }  & $\underline{1.47}$ { \tiny $\pm 0.01$ } 
 \\\addlinespace[0.3em]

 

 
 
 \bottomrule
\end{tabular}
}
\end{sc}
\end{small}
\end{center}
\caption{\footnotesize {\bf Unlearning with access to $\Dr$.} Comparing different unlearning methods in unlearning $10\%$ and $50\%$ of $\D$. Avg. Gap (see \S~\ref{sec:amun_metrics}), with MIS as the MIA score, is used for evaluation (lower is better). The lowest value is shown in bold while the second best is specified with underscore. As the results show, \amun{} outperforms all other methods by achieving lowest Avg. Gap and \amun{}$_{SalUn}$ achieves comparable results.}
\label{tab:mia}

\end{table*}

\begin{table*}[th!]
\begin{center}
\begin{small}
\begin{sc}
\resizebox{\textwidth}{!}{
\begin{tabular}{@{} l  c c  c  c c | c c c c c @{}}
 \toprule


& \multicolumn{5}{@{}c}{\textbf{Random Forget ($10\%$)}} & \multicolumn{5}{@{}c}{\textbf{Random Forget ($50\%$)}} \\\addlinespace[0.3em]

 \multicolumn{1}{c}{\scriptsize \textbf{}} & 
 \multicolumn{1}{c}{\scriptsize Unlearn Acc} & 
 \multicolumn{1}{c}{\scriptsize Retain Acc} & 
 \multicolumn{1}{c}{\scriptsize Test Acc} & 
 \multicolumn{1}{c}{\scriptsize MIA} & 
 \multicolumn{1}{c}{\scriptsize Avg. Gap} & 
 \multicolumn{1}{c}{\scriptsize Unlearn Acc} &
 \multicolumn{1}{c}{\scriptsize Retain Acc} & 
 \multicolumn{1}{c}{\scriptsize Test Acc} & 
 \multicolumn{1}{c}{\scriptsize  MIA} & 
 \multicolumn{1}{c}{\scriptsize Avg. Gap} 
 \\\addlinespace[0.3em]

 \cmidrule(r){2-6}
 \cmidrule(r){7-11}



 Retrain & $94.49$ {\tiny $\pm 0.20$} & $100.0 $ {\tiny $\pm 0.00$} & $94.33$ {\tiny $\pm 0.18$} & $12.53$ {\tiny $\pm 0.32$} & $0.00$  & 
 
 $92.09$ {\tiny $\pm 0.37$}  & $100.0$ {\tiny $\pm 0.00$} & $91.85$ {\tiny $\pm 0.33$} & $16.78$ {\tiny $\pm 0.37$} & $0.00$
 
 \\\addlinespace[0.3em]
  \cmidrule(r){1-11}

 RL & $100.00$ { \tiny $\pm 0.00$ }  & $100.00$ { \tiny $\pm 0.00$ }  & $94.45$ { \tiny $\pm 0.09$ }  & $3.06$ { \tiny $\pm 0.63$ }  & $3.77$ { \tiny $\pm 0.13$ } 
 & 
 
 $100.00$ { \tiny $\pm 0.00$ }  & $100.00$ { \tiny $\pm 0.00$ }  & $94.54$ { \tiny $\pm 0.11$ }  & $0.40$ { \tiny $\pm 0.03$ }  & $6.75$ { \tiny $\pm 0.02$ } 
 \\\addlinespace[0.3em]

  GA & $4.77$ { \tiny $\pm 3.20$ }  & $5.07$ { \tiny $\pm 3.54$ }  & $5.09$ { \tiny $\pm 3.38$ }  & $32.63$ { \tiny $\pm 50.85$ }  & $76.58$ { \tiny $\pm 7.73$ } 
  & 
 
  $100.00$ { \tiny $\pm 0.00$ }  & $100.00$ { \tiny $\pm 0.00$ }  & $94.57$ { \tiny $\pm 0.06$ }  & $0.35$ { \tiny $\pm 0.10$ }  & $6.77$ { \tiny $\pm 0.04$ } 
 \\\addlinespace[0.3em]

 BS & $100.00$ { \tiny $\pm 0.00$ }  & $100.00$ { \tiny $\pm 0.00$ }  & $94.48$ { \tiny $\pm 0.04$ }  & $1.11$ { \tiny $\pm 0.30$ }  & $4.27$ { \tiny $\pm 0.07$ } 
 & 
  
 $100.00$ { \tiny $\pm 0.00$ }  & $100.00$ { \tiny $\pm 0.00$ }  & $94.59$ { \tiny $\pm 0.03$ }  & $0.38$ { \tiny $\pm 0.02$ }  & $6.76$ { \tiny $\pm 0.01$ }
 \\\addlinespace[0.3em]

  SalUn  & $100.00$ { \tiny $\pm 0.00$ }  & $100.00$ { \tiny $\pm 0.00$ }  & $94.47$ { \tiny $\pm 0.10$ }  & $2.39$ { \tiny $\pm 0.64$ }  & $3.95$ { \tiny $\pm 0.14$ }
  & 

   $100.00$ { \tiny $\pm 0.00$ }  & $100.00$ { \tiny $\pm 0.00$ }  & $94.57$ { \tiny $\pm 0.12$ }  & $0.33$ { \tiny $\pm 0.04$ }  & $6.77$ { \tiny $\pm 0.03$ } 
 \\\addlinespace[0.5em]

 \textbf{Amun} & $94.28$ { \tiny $\pm 0.37$ }  & $97.47$ { \tiny $\pm 0.10$ }  & $91.67$ { \tiny $\pm 0.04$ }  & $11.61$ { \tiny $\pm 0.60$ }  & $\underline{1.61}$ { \tiny $\pm 0.09$ } 
 & 
 
 $92.77$ { \tiny $\pm 0.52$ }  & $95.66$ { \tiny $\pm 0.25$ }  & $89.43$ { \tiny $\pm 0.19$ }  & $14.13$ { \tiny $\pm 0.67$ }  & $\underline{2.52}$ { \tiny $\pm 0.16$ } 
 \\\addlinespace[0.3em]

 \textbf{Amun}$_{+Salun}$ & $94.19$ { \tiny $\pm 0.38$ }  & $97.71$ { \tiny $\pm 0.06$ }  & $91.79$ { \tiny $\pm 0.12$ }  & $11.66$ { \tiny $\pm 0.16$ }  & $\bf 1.51$ { \tiny $\pm 0.02$ } 
 & 
 
 $91.90$ { \tiny $\pm 0.63$ }  & $96.59$ { \tiny $\pm 0.31$ }  & $89.98$ { \tiny $\pm 0.44$ }  & $13.07$ { \tiny $\pm 0.66$ }  & $\bf 2.35$ { \tiny $\pm 0.15$ } 
 \\\addlinespace[0.3em]

 \bottomrule
\end{tabular}
}
\end{sc}
\end{small}
\end{center}
\caption{\footnotesize {\bf Unlearning with access to only $\Df$.}  Comparing different unlearning methods in unlearning $10\%$ and $50\%$ of $\D$. Avg. Gap (see \S~\ref{sec:amun_metrics}) is used for evaluation (lower is better) when only $\Df$ is available during unlearning. As the results show, \amun{}$_{SalUn}$ significantly outperforms all other methods, and \amun{} achieves comparable results. }
\label{tab:mia_forgetonly}

\end{table*}

\section{Continuous Unlearning (cont.)}
\label{apx:adaptive}

In \S~\ref{sec:adaptive}, we showed \amun{}, whether with adaptive computation of $\Da$ or using the pre-computed ones, outperforms other unlearning methods when handling multiple unlearning requests. Another important observation on the presented results in Figure~\ref{fig:adaptive} is that {\em the effectiveness of unlearning decreases with the number of unlearning requests}. For the setting with access to $\Dr$, this decrease is due to the fact that the $\Df$ at each step has been a part of $\Dr$ at the previous steps; the model has been fine-tuned on this data in all the previous steps which has led to further improving confidence of the modes on predicting those samples. 
This result also matches the theoretical and experimental results in differential privacy literature as well~\cite{dwork2006differential,abadi2016deep}. 

This problem does not exist for the setting where there is no access to $\Dr$, but we still see a decrease in the unlearning effectiveness as we increase the number of unlearning requests. The reason behind this deterioration is that the model itself is becoming weaker. As Figure~\ref{fig:adaptive_FT_RF} shows, the accuracy on the model on both $\Dr$ and $\Dt$ gets worse as it proceeds with the unlearning request; this is because each unlearning step shows the model only $2\%$ ($1$K) of the samples and their corresponding adversarial examples for fine-tuning. So this deterioration is expected after a few unlearning requests. So when using \amun{} in this setting (no access to $\Dr$) in practice, it would be better to decrease the number of times that the unlearning request is performed, for example by performing a lazy-unlearning (waiting for a certain number of requests to accumulate) or at least using a sub-sample of $\Dr$ if that is an option.

\begin{figure*}[h!]
\centering
\includegraphics[width=0.28\textwidth]{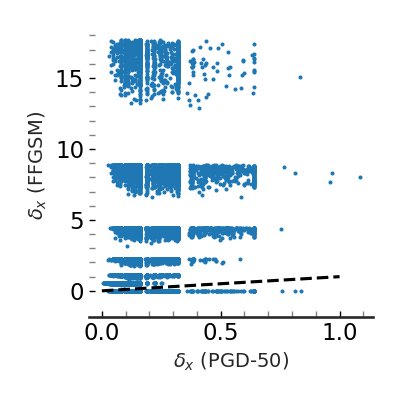}
  \caption{For each image in CIFAR-10 the $x$-axis shows the Euclidean distance of the corresponding adversarial example that is found by using PGD-50 in Algorithm~\ref{alg:advset}. $y$-axis shows this distance for the adversarial examples found by the variant of FFGSM in Algorithm~\ref{alg:advset}. The dashed line shows the $x=y$ line.}
  \label{fig:fgsm_ratio}
\end{figure*}


\chapter{Class Unlearning (Cont.)}

\section{Proofs}
\label{apx:proofs}

\begin{proof}[Proof of Proposition~\ref{lem:iproj-tilt}.]

We start with showing that it is enough to prove the proposition holds for $\tilde p$ (i.e., $q^*$ is the I-projection of $\tilde p$ onto the probability simplex of retrained classes with the given linear constraints). 

Let $q\in\mathcal A$; then $q(f\mid x)=0$ and $q$ has support only on retained labels. Note that:
\[
\mathrm{KL}\!\big(q\,\|\,p\big)
=\sum_{y\neq f} q(y)\log\frac{q(y)}{p(y)}
=\sum_{y\neq f} q(y)\log\frac{q(y)}{\tilde p(y)(1-p(f))}
=\sum_{y\neq f} q(y)\log\frac{q(y)}{\tilde p(y)}\;-\;\log\big(1-p(f)\big).
\]
Now, the final term is independent of $q$, hence
\[
\arg\min_{q\in\mathcal A}\ \mathrm{KL}\!\big(q\,\|\,p\big)
=\arg\min_{q\in\mathcal A}\ \mathrm{KL}\!\big(q\,\|\,\tilde p\big).
\]

Therefore, the information projection from $p$ onto the set $\mathcal{A}$ is equivalent to the projection from $\tilde p$ onto this set. Now, the objective $\mathrm{KL}(q\|\tilde p)=\sum_y q(y)\log\!\big(q(y)/\tilde p(y)\big)$ is strictly convex on the simplex, so any feasible minimizer is unique. Now we introduce Lagrange multipliers $\alpha,\beta$ for the normalization and expectation constraints.
Stationarity with respect to $q(y)$ gives
\[
\log\frac{q(y)}{\tilde p(y)}+1+\alpha+\beta s_y=0,
\]
which rearranges to
\[
q(y)=\tilde p(y)\,\exp\{-1-\alpha-\beta s_y\}.
\]
Normalizing yields the exponential family
\[
q_\beta(y)=\frac{\tilde p(y)\,e^{\beta s_y}}{\sum_j \tilde p(j)\,e^{\beta s_j}}.
\]

It remains to pick $\beta$ so that $\sum_y q_\beta(y)\,s_y=c$.
First, we define $m(\beta)=\sum_y q_\beta(y)s_y$. So we need to find value of $\beta$ that gives us $m(\beta) = c$. We suppress the fixed $x$ to lighten notation and define:
\[
Z(\beta)\ :=\ \sum_{j\in S}\tilde p(j)\,e^{\beta s_j},\qquad
q_\beta(y)\ :=\ \frac{\tilde p(y)\,e^{\beta s_y}}{Z(\beta)},\qquad
\]
Now note that:
\[
Z'(\beta)=\sum_{j\in S}\tilde p(j)\,s_j\,e^{\beta s_j},\qquad
Z''(\beta)=\sum_{j\in S}\tilde p(j)\,s_j^2\,e^{\beta s_j}.
\]
So, we can write:
\[
m(\beta)=\sum_{y\in S} \frac{\tilde p(y)\,e^{\beta s_y}}{Z(\beta)}\,s_y
=\frac{Z'(\beta)}{Z(\beta)}.
\]
Therefore, since $Z(\beta)$ is $C^\infty$, $m(\beta)$ is $C^\infty$ as well, and by differentiating once more we get:
\begin{align*}
m'(\beta)
= \frac{Z''(\beta)}{Z(\beta)}-\Big(\frac{Z'(\beta)}{Z(\beta)}\Big)^2
= \sum_{y\in S} q_\beta(y)\,s_y^2 - \Big(\sum_{y\in S} q_\beta(y)\,s_y\Big)^2
= \mathrm{Var}_{q_\beta}(s)\ \ge\ 0.
\end{align*}
Hence $m$ is nondecreasing on $\mathbb{R}$. Moreover, if the scores are not all equal on $S$, then
$q_\beta(y)>0$ for all $y\in S$, and the variance $\mathrm{Var}_{q_\beta}(s)$ is strictly positive, so
\[
\forall\beta\in\mathbb{R}:\qquad m'(\beta)=\mathrm{Var}_{q_\beta}(s)>0,
\]
i.e., $m$ is \emph{strictly} increasing. Moreover, $\lim_{\beta\to-\infty}m(\beta)=\min s_y$ and $\lim_{\beta\to+\infty}m(\beta)=\max s_y$.
Thus, by the intermediate value theorem, for any feasible $c$ there exists a unique $\beta^\star$ with $m(\beta^\star)=c$.
The corresponding $q^\star=q_{\beta^\star}$ is the unique minimizer.
\end{proof}

\section{Methods (cont.)}

In~\Cref{apx:tilt} we will see further empirical observation about the necessity of using the tilted distribution to better approximate the distirubution of the Retrain model. In~\Cref{apx:score} we will elaborate on the scoring function used in our experiments.

\subsection{Effect of tilting}
\label{apx:tilt}

In this section, we further evaluate the insufficiency of the basic rescaling of the target model's probability distribution to approximating the distribution of the Retrain model. For this purpose, we plot the conditional distributions of each remaining classes, when unlearning the \texttt{automobile} (Figure~\ref{fig:score_c1}) and \texttt{frog} (Figure~\ref{fig:score_c6}) classes from a ResNet18 model trained on CIFAR-10. As shown in both figures, the rescaled conditional distributions are very different from those of the Retrain model. As mentioned earlier, in the Retrain model, the distributions are more skewed toward a few classes. This bias toward more similar classes are better captured using the tilted reweighting distribution, which utilizes the inter-class similarities to adjust the rescaled distribution by introducing the bias toward more similar classes that is expected from the Retrain model.

\begin{figure}[t]
\centering
\includegraphics[width=0.95\linewidth]{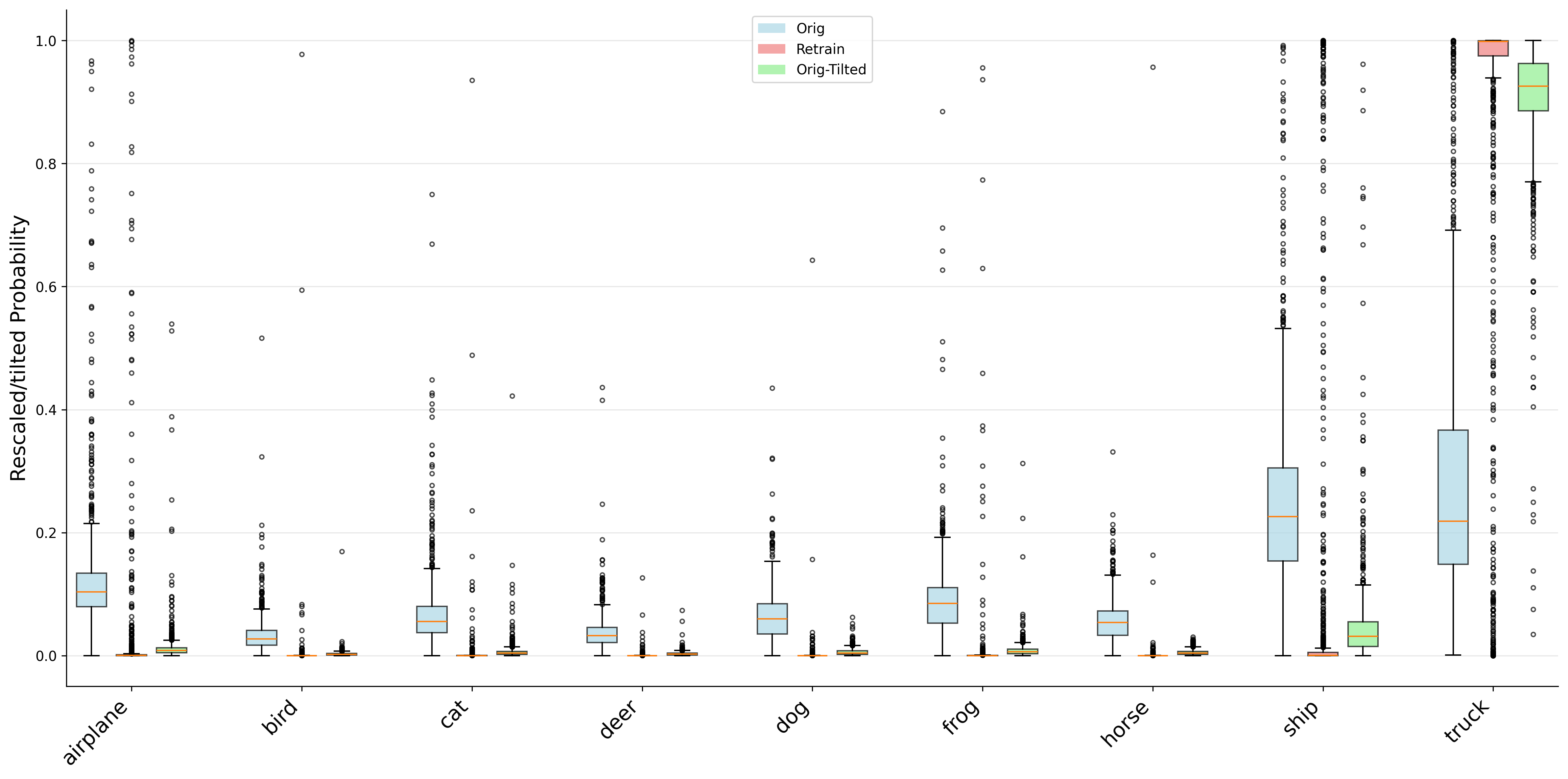}
\caption{Forgetting the \texttt{automobile} class in a ResNet-18 model trained on CIFAR-10. The reweighted probabilities of the original model according to~\eqref{equ:reweighted} (\texttt{Orig}) compared to the probabilities assigned to \texttt{automobile} in the Retrain model. As the figure shows, the Retrain model has a much higher bias toward \texttt{Truck} class, which is better captured when using the tilted reweighting accoriding to~\eqref{equ:tilted} (\texttt{Orig-Tilted}).}
\label{fig:score_c1}
\end{figure}

\subsection{Score function}
\label{apx:score}

Our goal is to assign, for each retained class $y\neq y_f$, a scalar similarity score $s_y$ that captures how likely a classifier trained \emph{without} the forget class $y_f$ would bias predictions of $x\in D_f$ toward class $y$. We use a geometry-based score derived from the classifier’s logit weights.
Let $w_y\in\mathbb{R}^d$ denote the column of the final linear layer (logit weights) for class $y\in\mathcal{Y}$ in the original model. To capture the dominant inter-class structure while reducing noise and redundancy, we perform PCA on the matrix
$W\!\!=\![\,w_1\ \cdots\ w_K\,]\in\mathbb{R}^{d\times K}$.
Let $U_{d'}\in\mathbb{R}^{d\times d'}$ be the top $d'$ principal directions ($d'\ll d$) and define the projected class embeddings
\[
\phi(w_y)\ \triangleq\ U_{d'}^\top w_y\ \in\ \mathbb{R}^{d'}.
\]
We then define cosine similarities between the forget class and each retained class:
\begin{equation}
\label{eq:cos-sim}
\tilde{s}_y\ \triangleq\ \cos\!\big(\phi(w_y),\ \phi(w_{y_f})\big)
\ =\
\frac{\langle \phi(w_y),\ \phi(w_{y_f})\rangle}{\|\phi(w_y)\|_2\,\|\phi(w_{y_f})\|_2},
\quad y\neq y_f.
\end{equation}

Finally, we use softmax to derive $s_y$ values in the form of probabilities from the values $\tilde s_y$. We use a small value for the temperature of softmax in our experiments (e.g., $0.01$) to make the similarity values distinct for more similar classes.

\begin{figure}[t]
\centering
\includegraphics[width=0.95\linewidth]{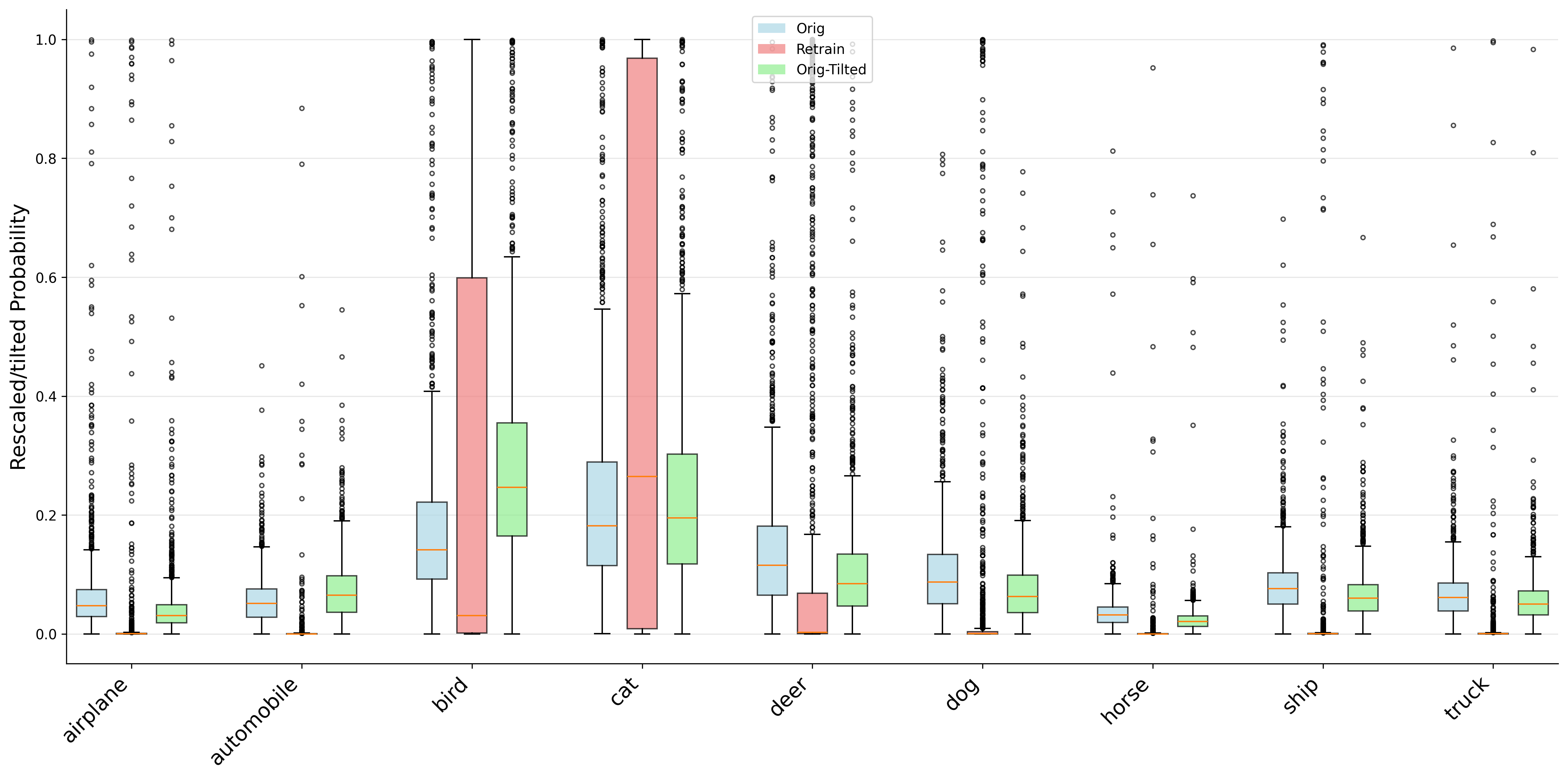}
\caption{Forgetting the \texttt{frog} class in a ResNet-18 model trained on CIFAR-10. The reweighted probabilities of the original model according to~\eqref{equ:reweighted} (\texttt{Orig}) compared to the probabilities assigned to \texttt{frog} in the Retrain model. As the figure shows, the Retrain model has a much higher bias toward a few classes, which are better captured when using the tilted reweighting accoriding to~\eqref{equ:tilted} (\texttt{Orig-Tilted}).}
\label{fig:score_c6}
\end{figure}

\section{Experiments (cont.)}

\subsection{Baseline Methods}
\label{sec:baselines}
In this section, we detail several baseline methods for machine unlearning.

\noindent\textbf{Retraining} refers to training a new model from scratch using only the retained dataset $\mathcal{D}_r$. This baseline serves as the ideal unlearning outcome and is commonly regarded as the "gold standard."
    
\noindent\textbf{Fine-tuning (FT)}~\cite{warnecke2021machine} fine-tunes the source model on the remaining dataset $\mathcal{D}_r$. In contrast to this standard fine-tuning, our approach performs targeted intervention by adjusting the model's output distribution to explicitly suppress the forgotten class and reallocate its probability mass proportionally to non-forgotten classes.

\noindent\textbf{Random Labeling (RL)}~\cite{golatkar2020eternal} 
fine-tunes the model with randomly relabeled forgetting data $\mathcal{D}_f$, which helps prevent the forgotten data from influencing the model’s predictions. 

\noindent\textbf{Gradient Ascent (GA)}~\cite{thudi2022unrolling} inverts SGD updates to erase the influence of specific training examples.

\noindent\textbf{l1-sparse (l1)}~\cite{jia2023model} proposes a sparsity-aware unlearning framework that prunes model weights to simplify the parameter space.

\noindent\textbf{Boundary Unlearning (BU)}~\cite{chen2023boundary} shifts the decision boundary of the original model to mimic the model’s decisions after retraining from scratch.

\noindent\textbf{SALUN~\cite{fan2023salun}} introduce the concept of weight saliency and perform unlearning by modifying the model weights rather than the entire model, improving effectiveness and efficiency.

\noindent\textbf{SVD Unlearning (SVD)}~\cite{kodge2024deep} performs class unlearning by identifying and suppressing class-discriminative features through singular value decomposition (SVD) of layer-wise activations.

\noindent\textbf{SCRUB}~\cite{kurmanji2023towards} proposes a novel teacher-student formulation, where the student model selectively inherit from a knowing-all teacher only when the knowledge does not pertain to the data to be deleted. 

\noindent\textbf{L2UL}~\cite{cha2024learning}  introduces an instance-wise unlearning framework that removes information using only the pre-trained model and the data points flagged for deletion.

\noindent\textbf{SCAR}~\cite{bonato2024retain} proposes a distillation-trick mechanism that transfers the original model’s knowledge to the unlearned model using out-of-distribution images, preserving test performance without any retain set.

\subsection{Experiment Settings}
\label{sec:exp setting}
We evaluate our method using two standard architectures: \textbf{VGG-19}~\cite{simonyan2014very} and \textbf{ResNet-18}~\cite{he2016deep}. All models are trained from scratch for \textbf{101 - 201 epochs} with a batch size of \textbf{128}. Optimization is performed using \textbf{stochastic gradient descent (SGD)} with a learning rate of \textbf{0.1}, \textbf{momentum of 0.9}, and \textbf{weight decay of $5 \times 10^{-4}$}. The learning rate follows a \texttt{torch.optim.lr\_scheduler.StepLR} schedule with \texttt{step\_size}$\,=40$ and \texttt{gamma}$\,=0.1$.  

For our unlearning procedure, we run updates for \textbf{10 epochs} with a learning rate of \textbf{0.001}. Baseline re‑training budgets mirror prior work: GA~\cite{thudi2022unrolling}, FT~\cite{warnecke2021machine},  l1~\cite{jia2023model}, and L2UL~\cite{cha2024learning} baselines are run for \textbf{20 epochs}, SVD~\cite{kodge2024deep}, and SCRUB~\cite{kurmanji2023towards} are run for \textbf{10 epochs}, while SalUn~\cite{fan2023salun} is run for \textbf{15 epochs} and SCAR~\cite{bonato2024retain} for \textbf{25 epochs}. All experiments are implemented in \textbf{PyTorch 3.11} and executed on four \textbf{NVIDIA A40} GPUs, and repeated with three different random seeds; we report the average results across those runs.
\section{Results on MNIST and CIFAR-100}
\label{sec:mnist_cifar}
\begin{table}[htbp]
\small
\centering
\resizebox{\textwidth}{!}
{\begin{tabular}{@{}clc c cc c c@{}}
\toprule
\multirow{2}{*}{\textbf{Data}} & \multirow{2}{*}{\textbf{Method}} & \multicolumn{3}{c|}{\textbf{VGG19}~\cite{simonyan2014very}} & \multicolumn{3}{c}{\textbf{ResNet18}~\cite{he2016deep}} \\ 
\addlinespace
& & $ACC_r$ (\textuparrow) & $ACC_f$ (\textdownarrow) & $MIA$ (\textuparrow) & $ACC_r$ (\textuparrow) & $ACC_f$ (\textdownarrow) & $MIA$ (\textuparrow) \\
\midrule
\multirow{8}{*}{\rotatebox{90}{\textsc{MNIST}}} 
    & Original & 99.52 $\pm$ 0.01 & 99.57 $\pm$ 1.33  &  0 &99.65 $\pm$ 0.04   & 99.91 $\pm$ 1.05  & 0.23 $\pm$ 0.23  \\
    & Retraining &99.54 $\pm$ 0.03 &0  & 100 $\pm$ 0&99.64 $\pm$ 0.05 & 0  & 100 $\pm$ 0\\
    & FT~\cite{warnecke2021machine} &99.43 $\pm$ 0.72 &0 & 99.25 $\pm$ 0.11 &98.16 $\pm$ 0.65 &0 &95.78 $\pm$ 1.28\\
    & RL~\cite{golatkar2020eternal} &99.04 $\pm$ 0.19 &0 & 99.66 $\pm$ 0.21 &99.33 $\pm$ 0.25 &0 & 99.64 $\pm$ 0.15\\
    & GA~\cite{thudi2022unrolling} &97.69 $\pm$ 2.42   & 0 &96.86 $\pm$ 2.87 &98.26 $\pm$ 0.12  &14.94 $\pm$ 0.03 &85.19 $\pm$ 0.17  \\
    & l1~\cite{jia2023model} &94.07 $\pm$ 4.48 &0.01 $\pm$ 0.02 &92.55 $\pm$ 1.53 &93.47 $\pm$ 1.98 &0.04 $\pm$ 0.02 &97.51 $\pm$ 0.42\\
    & BU~\cite{chen2023boundary} &93.14 $\pm$ 8.19 &0 &95.40 $\pm$ 0.06 &94.12 $\pm$ 6.51 &0 &98.62 $\pm$ 0.15 \\
    & SalUn~\cite{foster2024fast} &99.23 $\pm$ 0.18 &0 &100 $\pm$ 0&99.43 $\pm$ 0.77 &0 & 100 $\pm$ 0\\
    & SVD~\cite{kodge2024deep} &99.16 $\pm$ 0.20 &0 &100 $\pm$ 0&99.37 $\pm$ 0.32 &0 &99.87 $\pm$ 0.13\\
    & SCRUB~\cite{kurmanji2023towards} &99.34 $\pm$ 0.03 &0 &98.73 $\pm$ 0.15 &99.45 $\pm$ 0.06 & 0 & 91.88 $\pm$ 0.34 \\
    & SCAR~\cite{bonato2024retain} &98.93 $\pm$ 0.10 & 0 &100 $\pm$ 0 &99.20 $\pm$ 0.24 & 0  & 97.82 $\pm$ 0.68\\
    & l2ul~\cite{cha2024learning} &99.15 $\pm$ 0.12 & 0 &100 $\pm$ 0 &95.75 $\pm$ 0.16 & 0 &93.83 $\pm$ 0.19 \\
    & \textbf{TRW} &\textbf{99.52 $\pm$ 0.07}  &\textbf{0}  & \textbf{100 $\pm$ 0} &99.49 $\pm$ 0.04&0  &100 $\pm$ 0 \\
    & \textbf{TRW-2R} & 99.48 $\pm$ 0.13&0& 99.78 $\pm$ 0.22&99.45 $\pm$ 0.08 &0.04 $\pm$ 0.06&99.98 $\pm$ 0.15\\
\addlinespace
\multirow{8}{*}{\rotatebox{90}{\textsc{CIFAR}-100}} 
    & Original &69.87 $\pm$ 0.80 &70.72 $\pm$ 6.41 &0.45 $\pm$ 0.85 & 78.52 $\pm$ 0.58 & 78.93 $\pm$ 5.77 & 0.3 $\pm$ 0.5\\
    & Retraining &69.54 $\pm$ 0.92&0 &100 $\pm$ 0& 78.30 $\pm$ 0.84 & 0 & 100 $\pm$ 0\\
    & FT~\cite{warnecke2021machine} &65.26 $\pm$ 1.58 & 0& 87.41 $\pm$ 3.32 &73.37 $\pm$ 1.42 &0 &86.09 $\pm$ 0.46\\
    & RL~\cite{golatkar2020eternal} &58.83 $\pm$ 3.04 & 0&87.20 $\pm$ 3.11&70.97 $\pm$ 1.68 & 0 &88.65 $\pm$ 0.33\\
    & GA~\cite{thudi2022unrolling} &62.34 $\pm$ 0.77 &0 &83.58 $\pm$ 0.41 &72.05 $\pm$ 0.19 & 0 &85.11 $\pm$ 1.45\\
    & l1~\cite{jia2023model} &57.28 $\pm$ 1.47 &0 &83.89 $\pm$ 1.78 &72.30 $\pm$ 1.84 &0 &84.75 $\pm$ 2.06\\
    & BU~\cite{chen2023boundary} &59.27 $\pm$ 3.23 &0 &84.32 $\pm$ 2.45 &67.52 $\pm$ 1.86 &0 &81.62 $\pm$ 0.79 \\
    & SalUn~\cite{foster2024fast} &63.58$\pm$ 3.06 &0 & 99.97 $\pm$ 0.03&72.11 $\pm$ 1.37 &0 &98.65 $\pm$ 1.35\\
    & SVD~\cite{kodge2024deep} &68.53 $\pm$ 1.78 &0 &100 $\pm$ 0&75.86 $\pm$ 1.79 &0 &99.30 $\pm$ 0.08 \\
    & SCRUB~\cite{kurmanji2023towards} &58.91 $\pm$ 1.32 &0 &89.01 $\pm$ 1.34 &76.92 $\pm$ 1.06 &0 & 87.65 $\pm$ 2.18 \\
    & SCAR~\cite{bonato2024retain} &65.97 $\pm$ 1.88 &1.38 $\pm$ 1.04 & 100 $\pm$ 0 &75.42 $\pm$ 1.17 & 0.5 $\pm$ 0.43& 97.61 $\pm$ 0.81\\
    & l2ul~\cite{cha2024learning} &66.77 $\pm$ 0.84 & 0&99.52 $\pm$ 0.53 &71.56 $\pm$ 1.38 & 0 & 97.05 $\pm$ 1.97 \\
    & \textbf{TRW} &69.36 $\pm$ 0.61 &0&99.15 $\pm$ 0.75 &\textbf{78.05 $\pm$ 0.43}&\textbf{0}&\textbf{99.56 $\pm$ 1.75} \\
    & \textbf{TRW-2R} & 69.07$\pm$ 1.69&0& 99.00 $\pm$ 1.36 & 77.97 $\pm$ 0.78 &0.01 $\pm$ 0.07 & 98.87 $\pm$ 1.15\\
\bottomrule
\end{tabular}}
\caption{\textbf{Single-class forgetting on MNIST and CIFAR-100} 
We bold the method with the highest retained accuracy (ACC\textsubscript{r}), membership attack robustness (MIA), and lowest forgotten class accuracy (ACC\textsubscript{f}). 
Our method (TRW, TRW-2R) consistently achieves perfect forgetting ($\mathrm{ACC}_f = 0$), while preserving high retained accuracy and strong MIA robustness across datasets and architectures.}
\label{tab:mnist_cifar_results}
\end{table}
To study performance on a higher-variance image domain, we further evaluate on \textbf{MNIST} and \textbf{CIFAR-100} with VGG and ResNet-18 backbones. As shown in~\cref{tab:mnist_cifar_results}, our methods deliver (near-)perfect forgetting with competitive retained accuracy and consistently stronger resistance to MIA attacks than the baselines.

\subsection{Results on Tiny-ImageNet-200}
\label{sec:imagenet}

\begin{table*}[htbp]
\centering
\begin{tabular}{l|ccc}
\hline
\multirow{2}{*}{\textbf{Method}} & \multicolumn{3}{c}{\textbf{ResNet-18}} \\ 
\cline{2-4}
& $ACC_r$ (\textuparrow) & $ACC_f$ (\textdownarrow) & $MIA$ (\textuparrow) \\
\hline
Original & 63.49  & 64.52  & 0 \\
Retrain & 63.32  & 0 & 100  \\
FT~\cite{warnecke2021machine} & 58.92 & 0 & 83.62 \\
RL~\cite{golatkar2020eternal} & 41.58 & 0 & 83.55 \\
SCRUB~\cite{kurmanji2023towards} &62.31 &0 &87.73 \\
{TRW} & {62.78} & {0} & {100} \\
\textbf{TRW-2R} & \textbf{62.84} & \textbf{0} & \textbf{100} \\
\hline
\end{tabular}
\caption{\textbf{Single-class forgetting on Tiny-ImageNet-200 with ResNet-18.} We report retained accuracy ($ACC_r$), forget accuracy ($ACC_f$), and membership inference attack accuracy (MIA).}
\label{tab:imagenet_forgetting}
\end{table*}

To evaluate unlearning on a more challenging benchmark, we also apply our method to the \textbf{Tiny-ImageNet-200} dataset using a ResNet-18 backbone. We average the results over 10 semantically diverse classes: \textit{goldfish}, \textit{European fire salamander}, \textit{bullfrog}, \textit{tailed frog}, \textit{American alligator}, \textit{boa constrictor}, \textit{trilobite}, \textit{scorpion}, \textit{black widow spider}, and \textit{tarantula}. Due to computational constraints, we compare against a selected subset of representative baselines. From~\cref{tab:imagenet_forgetting}, we observe that our methods achieve perfect forgetting with strong retained accuracy and MIA robustness.

\subsection{Running Time Analysis}
\label{sec:time analysis}
\begin{figure}[htbp]
    \centering
    \includegraphics[width=0.75\linewidth]{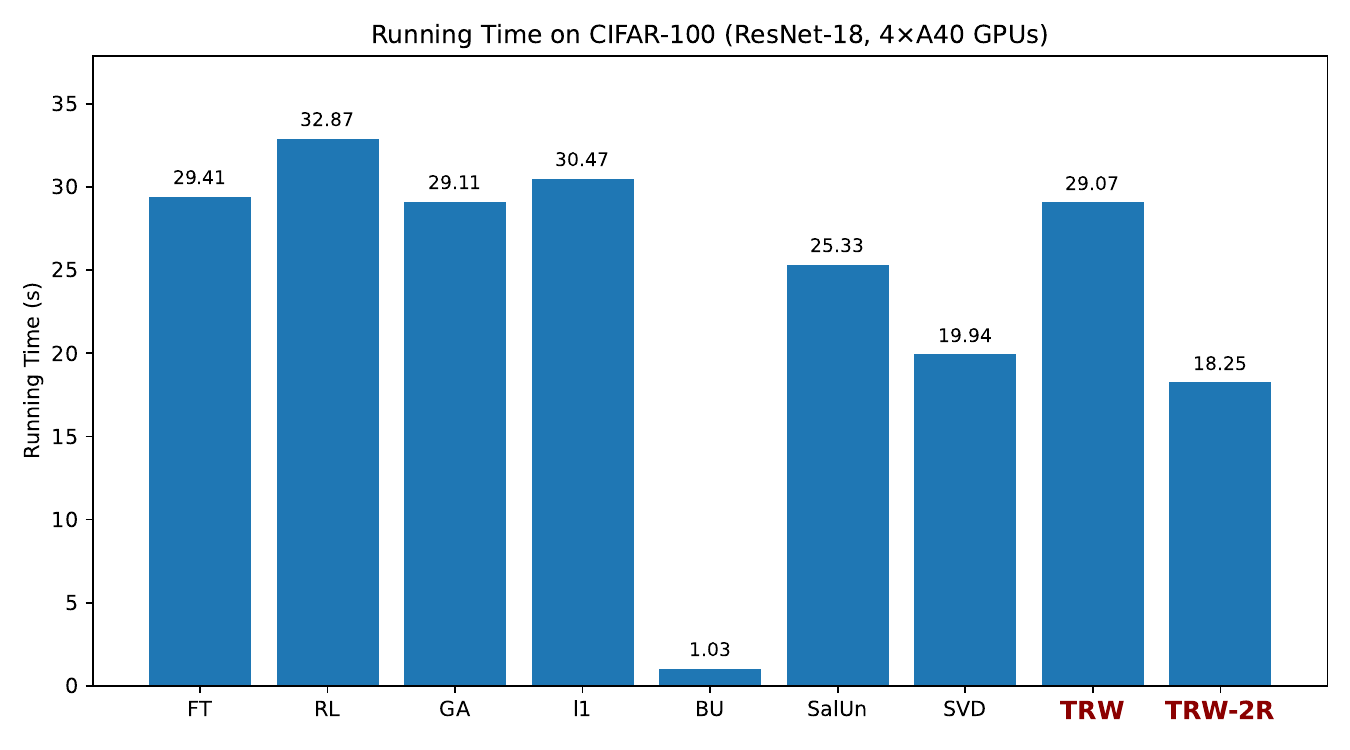}
    \caption{Running time comparison (in seconds per epoch) on CIFAR-100 with ResNet-18 using 4$\times$A40 GPUs. Our TRW-2R and TRW are among the fastest methods.}
    \label{fig:runtime}
\end{figure}

We compare the wall-clock training time of each unlearning method on the CIFAR-100 dataset using a ResNet-18 backbone, measured on four NVIDIA A40 GPUs. As shown in~\cref{fig:runtime}, our proposed \textbf{TRW-2R} method is among the fastest, completing in 18.25s, which is significantly faster than standard retraining-based baselines like FT and RL. 

Importantly, both of our methods — TRW, and TRW-2R — achieve optimal unlearning performance within just \textbf{10 epochs}. This makes them highly efficient and scalable in practice, offering strong privacy guarantees at a fraction of the computational cost of full retraining.

\subsection{Confusion Matrix analysis}
\begin{figure}
    \centering
    \includegraphics[width=\linewidth]{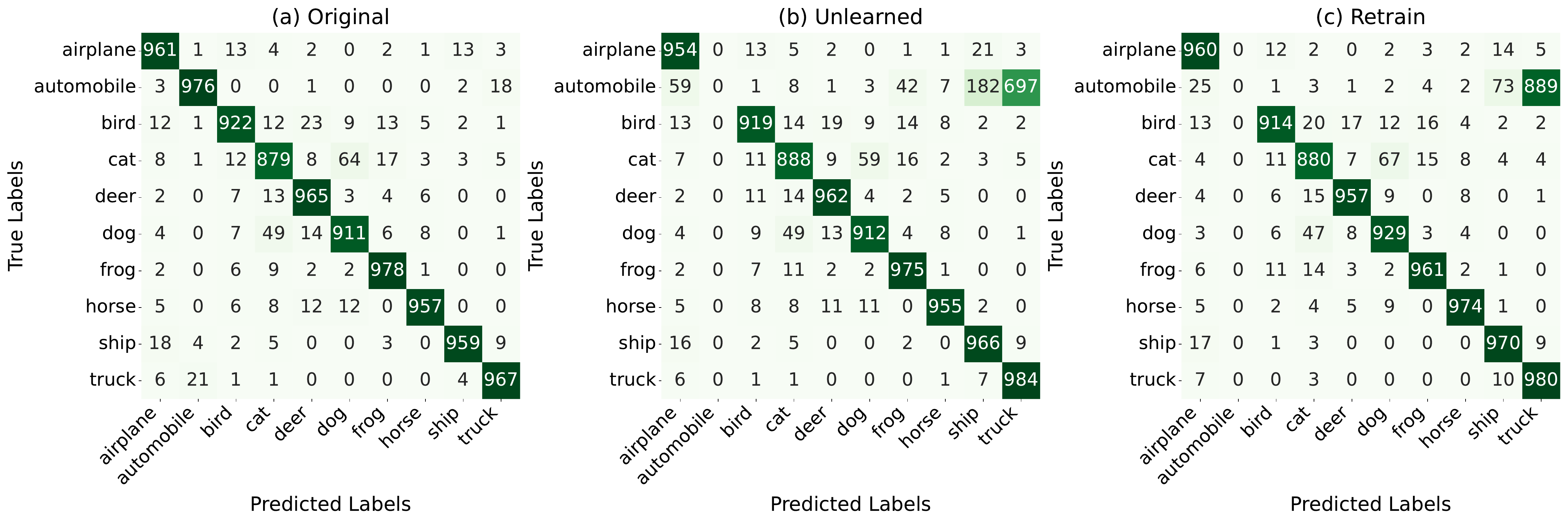}
    \caption{Confusion Matrix for original ResNet18 model, our unlearning model and retraining model forgetting automobile class, showing redistribution of automobile samples to other classes in proportion to the confusion in original model.}
    \label{fig:confusion_mat}
\end{figure}
Following~\cite{kodge2024deep}, we plot the confusion matrix showing the distribution of true labels and predicted labels for the original ResNet18 model, our unlearning model (on Automobile) and retraining model for CIFAR10 in~\cref{fig:confusion_mat}. Interestingly, we observe that in our unlearning model, a great number of the automobile samples are assigned to 'truck' class, with a small portion being assigned to the 'ship' class, indicating that they share some similar features. This result aligns with that in the retraining model. Furthermore, we see no performance drop in the remaining classes, indicating the robustness of our method.

\noindent \textbf{Forgetting Class Dependency:}
In addition to evaluating forgetting efficacy through classification accuracy, we further examine where the predictions of forgotten samples are redirected. ~\Cref{fig:willow_tree_reassign} summarizes the top classes that \textit{willow tree} samples were reassigned to after unlearning. Notably, a significant proportion of the reassigned predictions fall into semantically or visually similar categories, such as \textit{palm tree} (25\%), \textit{forest} (17\%), and \textit{oak tree} (14\%). This suggests that although the model effectively forgets the target class, it redistributes the predictions to classes with similar visual features. 
\begin{figure}[htbp]
    \centering
    \includegraphics[width=0.65\linewidth]{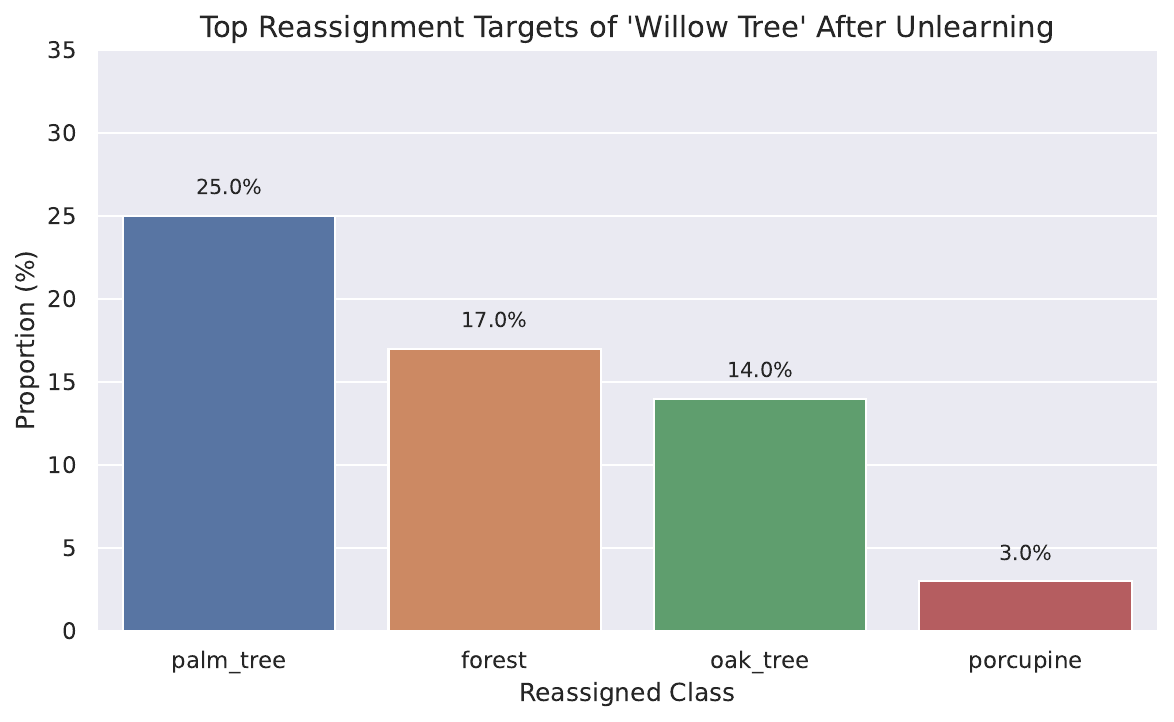}
    \caption{\textbf{Top reassignment targets of the forgotten class \textit{willow tree}.}
    The model tends to reassign samples to semantically similar classes such as \textit{palm tree}, \textit{forest}, and \textit{oak tree}.}
    \label{fig:willow_tree_reassign}
\end{figure}
\begin{table}[htbp]
\centering
\label{tab:top_acc_drop}
\begin{tabular}{l|c c c}
\toprule
\textbf{Class Name} & \textbf{Original Accuracy (\%)} & \textbf{Unlearned Accuracy (\%)} & \textbf{Accuracy Drop (\%)} \\
\midrule
maple tree  & 73.00 & 47.00 & 26.00 \\
lion         & 89.00 & 82.00 & 7.00 \\
otter        & 60.00 & 53.00 & 7.00 \\
boy          & 58.00 & 52.00 & 6.00 \\
dinosaur     & 76.00 & 71.00 & 5.00 \\
\bottomrule
\end{tabular}
\caption{\textbf{Top classes with the largest accuracy drop after unlearning the willow tree class in CIFAR100.}}
\end{table}

In addition,~\cref{tab:top_acc_drop} highlights the top five classes most affected by the unlearning of the \textit{willow tree} class in CIFAR-100. Notably, \textit{maple tree} suffers the largest performance degradation, with a 26\% drop in accuracy. This is followed by moderate drops in \textit{lion}, \textit{otter}, \textit{boy}, and \textit{dinosaur}, which show decreases of 5--7\%. These affected classes are likely to share visual similarities with \textit{willow tree}, indicating that the unlearning process may have altered representations that were partially shared across conceptually related categories.

\subsection{Ablation Study}
\label{sec:ablation}

\begin{table}[htbp]
\centering
\begin{tabular}{l|llll}
\hline
$\beta$ & $ACC_r$ (\textuparrow) & $ACC_f$ (\textdownarrow) & $MIA$ (\textuparrow) & $MIA-NN$(\textuparrow) \\ \hline
0   & 94.35 & 0 & 99.94 & 47.2   \\
5   & 94.37 & 0 & 98.94 & 76.1   \\
10  & 94.28 & 0 & 97.65 & 82.1   \\
20  & 77.67 & 0 & 98.72 & 98.6   \\ \hline
\end{tabular}
\caption{Effect of $\beta$ on retained accuracy ($ACC_r$), forgetting accuracy ($ACC_f$), and membership inference resistance ($MIA$ and $MIA\text{--}NN$). Increasing $\beta$ improves unlearning robustness but may harm retained performance. We select $\beta=10$ as it provides the best balance between maintaining accuracy on retained data and achieving strong resistance to membership inference attacks.}
\label{tab:beta}
\end{table}

We explore different choices of $\beta$ as it directly affects the trade-off between retained accuracy and unlearning robustness and the results are shown in Table~\ref{tab:beta}. We find that $\beta = 10$ provides the most balanced result across metrics. At this setting, the retained accuracy ($ACC_r$) remains very high (94.28), almost identical to $\beta = 0$ and $\beta = 5$, indicating that the model preserves strong predictive performance on the retained data. Meanwhile, the unlearning effectiveness improves significantly: the $MIA$ score decreases to 97.65, and the $MIA\text{--}NN$ score rises to 82.1. This demonstrates that the model is more resistant to membership inference attacks compared to smaller $\beta$ values, where adversaries can more easily detect forgotten samples. Although $\beta = 20$ achieves even stronger robustness ($MIA\text{--}NN = 98.6$), it severely degrades retained accuracy (dropping to 77.67). Therefore, we select $\beta = 10$ as it achieves the optimal balance between maintaining accuracy on retained data and providing strong attack robustness.

\subsection{Multiple Class Forgetting} 
\label{sec:multi-class}

\begin{table}[htbp]
\centering
\begin{tabular}{l|c|c|c}
\toprule
\textbf{Forget Classes} & \(\mathrm{ACC}_r \uparrow\) & \(\mathrm{ACC}_f \downarrow\) & \(\mathrm{MIA} \uparrow\) \\
\midrule
1 class  & 77.08 (77.24) & 0 (77.03) & 100 (0.40) \\
5 classes & 76.95 (77.10) & 0 (75.8)  & 100 (0.96) \\
10 classes & 74.84 (76.66) & 0 (77.04) & 99.42 (1.23) \\
\bottomrule
\end{tabular}
\caption{
\textbf{Multi-class forgetting performance on CIFAR-100 using ResNet18.} 
For each setting, we report the accuracy on the retained classes (ACC\textsubscript{r}~$\uparrow$), the accuracy on the forgotten classes (ACC\textsubscript{f}~$\downarrow$), and the MIA attack success rate (MIA~$\uparrow$). 
Values in parentheses denote the corresponding metrics from the original (unforgotten) model. 
\textbf{Our method maintains high retained accuracy while fully forgetting target classes and preserving robustness under strong MIAs, demonstrating stable and scalable forgetting performance across varying numbers of classes.}
}
\label{tab:cifar100_multiclass}
\end{table}

We evaluated our approach on multi-class forgetting tasks on CIFAR-100 using ResNet18, gradually increasing the number of randomly chosen forgotten classes from 1 to 10. Our method achieves stable and effective forgetting across different numbers of target classes. As shown in Table~\ref{tab:cifar100_multiclass}, the retained accuracy (ACC\textsubscript{r}) remains consistent with the original model, while the forgotten class accuracy (ACC\textsubscript{f}) drops to zero in all cases. Additionally, the MIA attack success rate remains near-perfect, demonstrating that the unlearned model is indistinguishable from retrained baselines even under strong adversarial probing. These results indicate the robustness of our forgetting strategy in multi-class scenarios.

\end{document}